%% file: main.tex
\documentclass{article}

\usepackage{silence}
\PassOptionsToPackage{round}{natbib}
\usepackage[final]{neurips_2020}

\usepackage[utf8]{inputenc} 
\usepackage[T1]{fontenc}    
\usepackage{url}            
\usepackage{booktabs}       
\usepackage{amsfonts}       
\usepackage{nicefrac}       
\usepackage{microtype}      

\usepackage{xspace}        
\usepackage{placeins}      

\usepackage{bm}
\usepackage{amsmath}
\usepackage{amssymb}
\usepackage{mathtools}

\usepackage{minitoc}

\usepackage{caption}
\usepackage{subcaption}

\usepackage{amsthm}
\usepackage{thmtools}
\usepackage{thm-restate}

\usepackage{enumerate}
\usepackage[inline]{enumitem}

\usepackage{algorithm}
\usepackage{algorithmicx}
\usepackage[noend]{algpseudocode}
\algnewcommand\algorithmicinput{\textbf{Inputs:}}
\algnewcommand\INPUT{\item[\algorithmicinput]}
\algnewcommand\algorithmicoutput{\textbf{Outputs:}}
\algnewcommand\OUTPUT{\item[\algorithmicoutput]}

\usepackage[hidelinks]{hyperref}       
\usepackage[capitalize]{cleveref}

\theoremstyle{plain}
\newtheorem{theorem}{Theorem}
\newtheorem{lemma}{Lemma}

\newtheorem{corollary}{Corollary}

\theoremstyle{definition}
\newtheorem{assumption}{Assumption}
\newtheorem{definition}{Definition}

\theoremstyle{remark}

\usepackage{tikz}
\usetikzlibrary{arrows.meta,positioning,matrix,calc}
\usetikzlibrary{shapes,arrows}
\usetikzlibrary{shapes.misc}
\usetikzlibrary{intersections}
\tikzset{>=latex}
\tikzstyle{block} = [rectangle, fill=blue!30,
    minimum height=3em, minimum width=6em]

\crefname{equation}{}{} 
\crefname{enumi}{}{}
\crefname{assumption}{Assumption}{Assumptions}

\usepackage{titletoc}

\input{notation.tex}

\usepackage{color}

\renewcommand{\paragraph}[1]{\textbf{#1}\hspace{1em}}

\title{Efficient Model-Based Reinforcement Learning through Optimistic Policy Search and Planning}

\author{%
      Sebastian Curi \thanks{Equal contribution} \\
      Department of Computer Science \\
      ETH Zurich \\
      \texttt{scuri@inf.ethz.ch} \\
  \And
      Felix Berkenkamp $^*$ \\
      Bosch Center for Artificial Intelligence \\
      \texttt{felix.berkenkamp@de.bosch.com} \\
  \And
      Andreas Krause \\
      Department of Computer Science \\
      ETH Zurich \\
      \texttt{krausea@ethz.ch} \\
}

\begin{document}

\doparttoc 
\faketableofcontents 

\maketitle

\input{sections/0-abstract}
\input{sections/1-introduction}
\input{sections/2-problem_statement_background}
\input{sections/3-optimistic_exploration}
\input{sections/4-experiments}
\input{sections/5-conclusion}
\section*{Broader Impact}
Improving sample efficiency is one of the key bottlenecks in applying reinforcement learning to real-world problems with potential major societal benefit such as personal robotics, renewable energy systems, medical decisions making, etc.
Thus, algorithmic and theoretical contributions as presented in this paper can help decrease the cost associated with optimizing RL policies. Of course, the overall RL framework is so general that potential misuse cannot be ruled out.

\begin{ack}
This project has received funding from the European Research Council (ERC) under the European Unions Horizon 2020 research and innovation program grant agreement No 815943. It was also supported by a fellowship from the Open Philanthropy Project.
\end{ack}

\bibliographystyle{plainnat}
\bibliography{thesis.bib}

\clearpage
\appendix

\part{Appendix} 

The following table provides an overview of the appendix.
\noptcrule 
\parttoc 


\clearpage
\input{appendix/expected_performance}

\input{appendix/extra_experiments}

\input{appendix/dyna_mpc}
\clearpage 

\input{appendix/exploration_proofs}

\input{appendix/gp_background}
\input{appendix/gp_variance_lipschitz}
\input{appendix/gp_exploration_proofs}

\input{appendix/extension_unbounded_domain}

\end{document}

%% file: notation.tex
\newcommand{\iid}{\textit{i.i.d.}\xspace}

\newcommand{\E}[2][]{\mathbb{E}_{#1}\mathopen{}\left[#2\right]\mathclose{}}

\newcommand{\Or}[1]{\mathcal{O}\mathopen{}\left(#1\right)\mathclose{}}

\newcommand{\Om}[1]{\Omega\mathopen{}\left(#1\right)\mathclose{}}

\newcommand{\R}{\mathbb{R}}

\DeclareMathOperator*{\argmax}{argmax}

\newcommand{\mb}[1]{\mathbf{#1}}        

\newcommand{\eqdef}{\vcentcolon=}

\renewcommand{\mid}{\,|\,}

\newcommand{\T}{\mathrm{T}}             

\newcommand{\Mi}[2]{\,\mathrm{I}\mathopen{}\left(#1; #2\right)\mathclose{}}

\newcommand{\st}{\text{s.t.~}}

\newcommand{\ball}{\ensuremath{\mathbb{B}}}

\newcommand{\btheta}{{\bm{\theta}}}

\newcommand{\alg}{H-UCRL\xspace}

\newcommand{\x}{\mb{s}}
\renewcommand{\u}{\mb{a}}
\newcommand{\X}{\mathcal{S}}
\newcommand{\U}{\mathcal{A}}

\newcommand{\ti}{n}                  
\newcommand{\Ti}{N}
\renewcommand{\ni}{t}                
\newcommand{\Ni}{T}

\newcommand{\noise}{\bm{\omega}}
\newcommand{\snoise}{\omega}

\newcommand{\xo}{\tilde{\x}}
\newcommand{\uo}{\tilde{\u}}

\newcommand{\nstate}{\ensuremath{p}}
\newcommand{\ninp}{\ensuremath{q}}

\newcommand{\modelclass}{\ensuremath{\mathcal{M}}}
\newcommand{\dataset}{\ensuremath{\mathcal{D}}}
\newcommand{\modelposterior}{\ensuremath{p(\tilde{f} \mid \dataset_{1:\ni})}}

\newcommand{\bsigma}{\bm{\sigma}}
\newcommand{\bSigma}{\bm{\Sigma}}
\newcommand{\bmu}{\bm{\mu}}

\newcommand{\ab}{\ensuremath{\mb{x}}}
\newcommand{\adomain}{\mathcal{X}}

%% file: sections/0-abstract.tex
\begin{abstract}
    \looseness -1
    Model-based reinforcement learning algorithms with probabilistic dynamical models are amongst the most data-efficient learning methods. 
    This is often attributed to their ability to distinguish between epistemic and aleatoric uncertainty.
    However, while most algorithms distinguish these two uncertainties for {\em learning} the model, they ignore it when {\em optimizing} the policy, which leads to greedy and insufficient exploration.
    At the same time, there are no practical solvers for optimistic exploration algorithms.
    In this paper, we propose a {\em practical} optimistic exploration algorithm (\alg).
    \alg reparameterizes the set of plausible models and {\em hallucinates} control directly on the {\em epistemic} uncertainty.
    By augmenting the input space with the {\em hallucinated} inputs, \alg can be solved using standard greedy planners.
    Furthermore, we analyze \alg and construct a general regret bound for well-calibrated models, which is provably sublinear in the case of Gaussian Process models. 
    Based on this theoretical foundation, we show how optimistic exploration can be easily combined with state-of-the-art reinforcement learning algorithms and different probabilistic models. 
    Our experiments demonstrate that optimistic exploration significantly speeds-up learning when there are penalties on actions, a setting that is notoriously difficult for existing model-based reinforcement learning algorithms.
\end{abstract}

%% file: sections/1-introduction.tex
\section{Introduction}
\looseness -1 
Model-Based Reinforcement Learning (MBRL) with probabilistic dynamical models can solve many challenging high-dimensional tasks with impressive sample efficiency \citep{Chua2018Deep}.
These algorithms alternate between two phases: 
first, they collect data with a policy and fit a model to the data; then, they simulate transitions with the model and optimize the policy accordingly.
%
A key feature of the recent success of MBRL algorithms is the use of models that explicitly distinguish between {\em epistemic} and {\em aleatoric} uncertainty when learning a model \citep{Gal2016Uncertainty}.
Aleatoric uncertainty is inherent to the system (noise), whereas epistemic uncertainty arises from data scarcity \citep{DerKiureghian2009Aleatory}. 
However, to optimize the policy, practical algorithms marginalize over both the aleatoric \emph{and} epistemic uncertainty to optimize the expected performance under the current model, as in PILCO \citep{Deisenroth2011PILCO}. 
This \emph{greedy exploitation} can cause the optimization to get stuck in local minima even in simple environments like the swing-up of an inverted pendulum:
In \cref{fig:inverted_final}, all methods can solve this problem without action penalties (left plot). 
However, with action penalties, the expected reward (under the epistemic uncertainty) of swinging up the pendulum is low relative to the cost of the maneuver. 
Consequently, the greedy policy does not actuate the system at all and fails to complete the task. 
While optimistic exploration is a well-known remedy, there is currently a lack of efficient, principled means of incorporating optimism in deep MBRL.

\paragraph{Contributions} 
\looseness -1 
Our main contribution is a novel optimistic MBRL algorithm, {\em Hallucinated-UCRL} (\alg), which 
can be applied together with state-of-the-art RL algorithms (\cref{sec:algorithm}).
Our key idea is to {\em reduce optimistic exploration to greedy exploitation} by 
reparameterizing the model-space using a mean/epistemic variance decomposition.
In particular, we augment the control space of the agent with {\em hallucinated} control actions that directly control the agent's {\em epistemic} uncertainty about the 1-step ahead transition dynamics (\cref{sec:practical}).
We provide a general theoretical analysis for \alg and prove sublinear regret bounds for the special case of Gaussian Process (GP) dynamics models (\cref{ssec:TheoreticalAnalysis}). 
Finally, we evaluate \alg in high-dimensional continuous control tasks that shed light on when optimistic exploration outperforms greedy exploitation and Thompson sampling (\cref{sec:experiments}).
To the best of our knowledge, this is the first approach that successfully implements \emph{optimistic} exploration with deep-MBRL.

\begin{figure}[t]
  \centering\includegraphics[width=\linewidth]{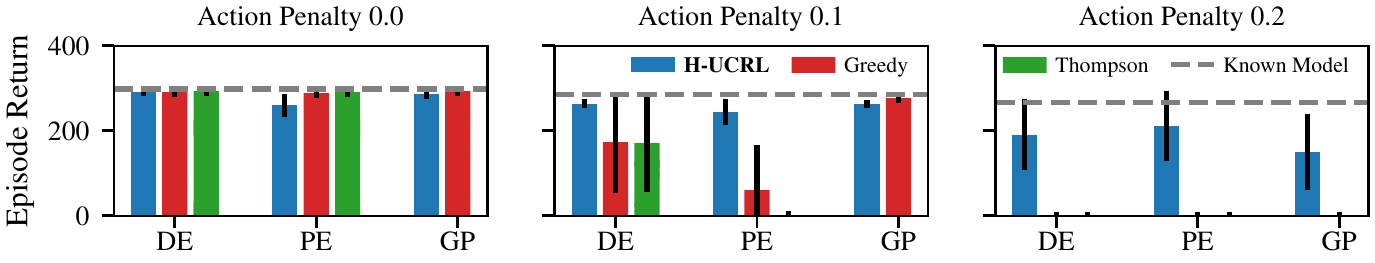}
\caption{
  Final returns in an inverted pendulum swing-up task with sparse rewards.
  As the action penalty increases, exploration through noise is penalized and algorithms get stuck in a local minimum, where the pendulum is kept at the bottom position. 
  Instead, \alg is able to solve the swing-up task reliably.
  This holds for for all considered dynamical models: Deterministic- (DE) and Probabilistic Ensembles (PE) of neural networks as well as Gaussian Processes (GP) models. 
  }
  \vspace{-1.5em}
  \label{fig:inverted_final}
\end{figure}

\paragraph{Related Work}
MBRL is a promising avenue towards applying RL methods to complex real-life decision problems due to its sample efficiency \citep{Deisenroth2013Survey}.
For instance, \citet{Kaiser2019MBAtari} use MBRL to solve the Atari suite, whereas \citet{Kamthe2018DataEfficient} solve low-dimensional continuous-control problems using GP models and \citet{Chua2018Deep} solve high-dimensional continuous-control problems using ensembles of probabilistic Neural Networks (NN).
All these approaches perform {\em greedy exploitation} under the current model using a variant of PILCO \citep{Deisenroth2011PILCO}.
Unfortunately, greedy exploitation  is {\em provably} optimal only in very limited cases such as linear quadratic regulators (LQR) \citep{Mania2019CertaintyEquivalence}.

\looseness -1 
Variants of {\em Thompson (posterior) sampling} are a common approach for {\em provable} exploration in reinforcement learning \citep{dearden1999model}. In particular, \citet{Osband2013More} propose Thompson sampling for {\em tabular} MDPs.
\citet{Chowdhury2019Online} prove a $\tilde{\mathcal{O}}(\sqrt{\Ni})$ regret bound for continuous states and actions for this theoretical algorithm, where $\Ni$ is the number of episodes. 
However, Thompson sampling can be applied only when it is tractable to sample from the posterior distribution over dynamical models. For example, this is intractable for GP models with continuous domains.
Moreover, \citet{Wang2018BatchedBO} suggest that approximate inference methods may suffer from variance starvation and limited exploration.

\looseness -1 
The \emph{Optimism-in-the-Face-of-Uncertainty} (OFU) principle is a classical approach towards {\em provable} exploration in the theory of RL.
Notably, \citet{Brafman2003Rmax} present the R-Max algorithm for {\em tabular} MDPs, where a learner is optimistic about the reward function and uses the {\em expected} dynamics to find a policy. R-Max has a sample complexity of $\mathcal{O}(1/\epsilon^3)$, which translates to a sub-optimal regret of $\tilde{\mathcal{O}}(\Ni^{2/3})$.
\citet{Jaksch2010UCRL} propose the UCRL algorithm that is optimistic on the transition dynamics and achieves an optimal $\tilde{\mathcal{O}}(\sqrt{\Ni})$ regret rate for tabular MDPs. Recently, \citet{Zanette2019Euler}, \citet{Efroni2019Tight}, and \citet{Domingues2020KernelVIUCB} provide refined UCRL algorithms for tabular MDPs.
When the number of states and actions increase, these {\em tabular algorithms are inefficient} and practical algorithms must exploit structure of the problem. The use of optimism in continuous state/action MDPs however is much less explored. \citet{Jin2019LSVI-UCB} present an optimistic algorithm for {\em linear} MDPs and \citet{Abbasi-Yadkori2011Regret} for linear quadratic regulators (LQR), both achieving $\tilde{\mathcal{O}}(\sqrt{\Ni})$ regret. 
Finally, \citet{Luo2018SLBO} propose a trust-region UCRL meta-algorithm that asymptotically finds an optimal policy but it is intractable to implement.  

Perhaps most closely related to our work, \citet{Chowdhury2019Online} present GP-UCRL for continuous state and action spaces. They use optimistic exploration for the policy optimization step with dynamical models that lie in a Reproducing Kernel Hilbert Space (RKHS). 
However, as mentioned by \citet{Chowdhury2019Online}, their algorithm is intractable to implement and cannot be used in practice.
Instead, we build on an implementable but expensive strategy that was heuristically suggested by \citet{Moldovan2015Optimismdriven} for planning on \emph{deterministic} systems and develop a principled and highly efficient optimistic exploration approach for deep MBRL.
Partial results from this paper appear in \citet[Chapter 5]{Berkenkamp2019Safe}.

\paragraph{Concurrent Work} 
\citet{kakade2020information} build tight confidence intervals for our problem setting based on information theoretical quantities.
However, they assume an optimization oracle and do not provide a practical implementation (their experiments use Thompson sampling). 
\citet{abeille2020efficient} propose an equivalent algorithm to \alg in the context of LQR and proved that the planning problem can be solved efficiently. 
In the same spirit as \alg, \citet{neu2020unifying} reduce intractable optimistic exploration to greedy planning using well-selected reward bonuses. In particular, they prove an equivalence between optimistic reinforcement learning and exploration bonus \citep{azar2017minimax} for tabular and linear MDPs.
How to generalize these exploration bonuses to our setting is left for future work.

%% file: sections/2-problem_statement_background.tex
\section{Problem Statement and Background} \label{sec:problem_statement}
We consider a stochastic environment with states $\x \in \X \subseteq \R^\nstate$, actions $\u \in \U \subset \R^\ninp$ within a compact set $\U$, and \iid, additive transition noise $\noise_\ti \in \R^\nstate$. The resulting transition dynamics are
\begin{equation}
	\x_{\ti+1} = f(\x_\ti, \u_\ti) + \noise_\ti
	\label{eq:stochastic_dynamic_system_additive}
\end{equation}
with $f \colon \X \times \U \to \X$. 
For tractability we assume continuity of $f$, which is common for any method that aims to approximate $f$ with a continuous model (such as neural networks). In addition, we also assume sub-Gaussian noise $\noise$, which includes any zero-mean distribution with bounded support and Gaussians. This assumption allows the noise to depend on states and actions.
\begin{assumption}[System properties]
  The true dynamics $f$ in \cref{eq:stochastic_dynamic_system_additive} are $L_f$-Lipschitz continuous and, for all $\ti \geq 0$, the elements of the noise vector $\noise_\ti$ are \iid $\sigma$-sub-Gaussian.
  \label{as:dynamics_f_lipschitz}
  \label{as:transition_noise_sub_gaussian}
\end{assumption}

\subsection{Model-based Reinforcement Learning}
\paragraph{Objective}
Our goal is to control the stochastic system \cref{eq:stochastic_dynamic_system_additive} optimally in an \emph{episodic} setting over a finite time horizon $\Ti$.
To control the system, we use any deterministic policy $\pi_\ti \colon \X \to \U$ from a set $\Pi$ that selects actions $\u_\ti = \pi_\ti(\x_\ti)$ given the current state. 
For ease of notation, we assume that the system is reset to a known state $\x_0$ at the end of each episode, that there is a known reward function $r \colon \X \times \U \to \mathbb{R}$, and we omit the dependence of the policy on the time index. 
Our results, easily extend to known initial state distributions and unknown reward functions using standard techniques (see \citet{Chowdhury2019Online}).
For any dynamical model $\tilde{f} \colon \X \times \U \to \X$ (e.g., $f$ in \cref{eq:stochastic_dynamic_system_additive}), the performance of a policy $\pi$ is the total reward collected during an episode in expectation over the transition noise $\noise$,
\begin{equation}
    \label{eq:performance_estimator}
    J(\tilde{f}, \pi) = \E[\tilde{\noise}_{0:\Ti-1}]{\sum \nolimits_{\ti =0}^\Ti r(\xo_\ti, \pi(\xo_\ti)) \, \bigg| \, \x_0},
    \quad \st \xo_{\ti + 1} = \tilde{f}(\xo_\ti, \pi(\xo_\ti)) + \tilde{\noise}_\ti .
\end{equation}
Thus, we aim to find the optimal policy $\pi^*$ for the true dynamics $f$ in  \cref{eq:stochastic_dynamic_system_additive},
\begin{equation}
  \label{eq:exploration:regret:optimal_parameters}
  \pi^* = \argmax_{\pi \in \Pi}\,J(f, \pi).
\end{equation}
If the dynamics $f$ were known, \cref{eq:exploration:regret:optimal_parameters} would be a standard stochastic optimal control problem. However, in model-based reinforcement learning we do \emph{not} know the dynamics $f$ and have to learn them online.

\paragraph{Model-learning}
We consider algorithms that iteratively select policies $\pi_\ni$ at each iteration/episode $\ni$ and conduct a single rollout on the real system \cref{eq:stochastic_dynamic_system_additive}. That is, starting with $\dataset_1 = \emptyset$, at each iteration $\ni$ we apply the selected policy $\pi_\ni$ to \cref{eq:stochastic_dynamic_system_additive} and collect transition data $\dataset_{\ni+1} = \{(\x_{\ti-1,\ni}, \u_{\ti-1, \ni} ), \x_{\ti,\ni}\}_{\ti=1}^\Ti$. 

\looseness -1 
We use a statistical model to estimate which dynamical models $\tilde{f}$ are compatible with the data in $\dataset_{1:\ni} = \cup_{0 < i \leq \ni} \dataset_i $. This can either come from a frequentist model with mean and confidence estimate  $\bmu_\ni(\x, \u)$ and $\bSigma_\ni(\x, \u)$, or from a Bayesian perspective that estimates a posterior distribution $\modelposterior$ over dynamical models $\tilde{f}$ and defines $\bmu_{\ni}(\cdot) = \mathbb{E}_{\tilde{f} \sim \modelposterior} [\tilde{f}(\cdot) ]$ and $\bSigma_\ni^2(\cdot)= \mathrm{Var}[ \tilde{f}(\cdot)]$, respectively. Either way, we require the model to be well-calibrated:

\begin{assumption}[Calibrated model]
    The statistical model is {\em calibrated} w.r.t. $f$ in \cref{eq:stochastic_dynamic_system_additive}, so that with $\bsigma_\ni(\cdot) = \mathrm{diag}(\bSigma_\ni(\cdot))$ there exists a sequence $\beta_\ni \in \R_{>0}$ such that, with probability at least $(1 - \delta)$, it holds jointly for all $\ni \geq 0$ and $\x, \u \in \X \times \U$ that $| f(\x, \u) - \bmu_\ni(\x, \u) | \leq \beta_\ni \bsigma_\ni(\x, \u)$, elementwise.
    \label{as:well_calibrated_model}
\end{assumption}

Popular choices for statistical dynamics models include {\em Gaussian Processes (GP)} \citep{Rasmussen2006Gaussian} and {\em Neural Networks (NN)} \citep{Anthony2009Neural}. GP models naturally differentiate between aleatoric noise and epistemic uncertainty and are effective in the low-data regime. 
They provably satisfy \cref{as:well_calibrated_model} when the true function $f$ has finite norm in the RKHS induced by the covariance function.
In contrast to GP models, NNs potentially scale to larger dimensions and data sets.
From a practical perspective, NN models that differentiate aleatoric from epistemic uncertainty can be efficiently implemented using Probabilistic Ensembles (PE) \citep{Lakshminarayanan2017Simple}.
Deterministic Ensembles (DE) are also commonly used but they do not represent aleatoric uncertainty correctly \citep{Chua2018Deep}. NN models are not calibrated in general, but can be re-calibrated to satisfy \cref{as:well_calibrated_model} \citep{Kuleshov2018Calibrated}. State-of-the-art methods typically learn models so that the one-step predictions in \cref{as:well_calibrated_model} combine to yield good predictions for trajectories \citep{Archer2015HMM,Doerr2018Probabilistic,Curi2020Structured}.


\subsection{Exploration Strategies}
\label{sec:exploration:optimism:expected_performance}

\begin{algorithm}[t]
  \caption{Model-based Reinforcement Learning}
  \label{alg:mbrl}
  \begin{algorithmic}[1]
    \INPUT{Calibrated dynamical model,
          reward function $r(\x, \u)$,
          horizon $\Ti$, 
          initial state $\x_0$}
    \For{$\ni = 1, 2, \dots$}
      \State Select $\pi_\ni$ based on \cref{eq:expected_performance_exploration}, \cref{eq:thompson_sampling_exploration}, or  \cref{eq:optimistic_exploration}
      \State Reset the system to $\x_{0, \ni} = \x_0$
      \For{$\ti = 1, \dots, \Ti$}
        \State $\x_{\ti, \ni} = f(\x_{\ti-1, \ni}, \pi_\ni(\x_{\ti-1, \ni})) + \noise_{\ti-1, \ni}$
        \label{alg:mbrl:rollout}
      \EndFor
      \State Update statistical dynamical model with the $\Ti$ observed state transitions in $\dataset_\ni$.
    \EndFor
  \end{algorithmic}
\end{algorithm}

Ultimately the performance of our algorithm depends on the choice of $\pi_\ni$. We now provide a unified overview of existing exploration schemes and summarize the MBRL procedure in \cref{alg:mbrl}.

\paragraph{Greedy Exploitation} In practice, one of the most commonly used algorithms is to select the policy $\pi_\ni$ that greedily maximizes the expected performance over the aleatoric uncertainty {\em and} epistemic uncertainty induced by the dynamical model. Other exploration strategies, such as dithering (e.g., epsilon-greedy, Boltzmann exploration) \citep{Sutton1998Reinforcement} or certainty equivalent control \citep[Chapter 6.1]{Bertsekas1995Dynamic}, can be grouped into this class. The greedy policy is
\begin{align}
    \label{eq:expected_performance_exploration}
    \pi_\ni^\mathrm{Greedy} = \argmax_{\pi \in \Pi} \E[\tilde{f} \sim \modelposterior]{ J(\tilde{f}, \pi)} .
\end{align}
\looseness -1
For example, PILCO \citep{Deisenroth2011PILCO} and GP-MPC \citep{Kamthe2018DataEfficient} use moment matching to approximate $\modelposterior$ and use {\em greedy} exploitation to optimize the policy.
Likewise, PETS-1 and PETS-$\infty$ from \citet{Chua2018Deep} also lie in this category, in which $\modelposterior$ is represented via ensembles. 
The main difference between PETS-$\infty$ and other algorithms is that PETS-$\infty$ ensures consistency by sampling a function per rollout, whereas PETS-1, PILCO, and GP-MPC sample a new function at each time step for computational reasons. 
We show in \cref{sec:expected_performance_bandit} that, in the bandit setting, the exploration is only driven by noise and optimization artifacts.
In the tabular RL setting, dithering takes an exponential number of episodes to find an optimal policy \citep{Osband2014Generalization}.
As such, it is \emph{not} an efficient exploration scheme for reinforcement learning.
Nevertheless, for some specific reward and dynamics structure, such as linear-quadratic control, greedy exploitation indeed achieves no-regret \citep{Mania2019CertaintyEquivalence}.
However, it is the most common exploration strategy and many practical algorithms to efficiently solve the optimization problem \cref{eq:expected_performance_exploration} exist (cf.~\cref{sec:practical}).

\looseness -1
\paragraph{Thompson Sampling}
A theoretically grounded exploration strategy is Thompson sampling, which optimizes the policy w.r.t.~a single model that is sampled from $\modelposterior$ at every episode. Formally,
\begin{align}
    \label{eq:thompson_sampling_exploration}
    \tilde{f}_\ni &\sim \modelposterior, \quad  \pi_\ni^\mathrm{TS} = \argmax_{\pi \in \Pi} J(\tilde{f}_\ni, \pi) .
\end{align}
This is different to PETS-$\infty$, as the former algorithm optimizes w.r.t.~the average of the (consistent) model trajectories instead of a single model.
In general, it is intractable to sample from $\modelposterior$. 
 Nevertheless, after the sampling step, the optimization problem is equivalent to greedy exploitation of the sampled model. Thus, the same optimization algorithms can be used to solve \cref{eq:expected_performance_exploration} and \cref{eq:thompson_sampling_exploration}. 

\paragraph{Upper-Confidence Reinforcement Learning (UCRL)}
The final exploration strategy we address is UCRL exploration \citep{Jaksch2010UCRL}, which optimizes jointly over policies and models inside the set $\modelclass_\ni = \{ \tilde{f} \mid | \tilde{f}(\x, \u) - \bmu_\ni(\x, \u) | \leq \beta_\ni \bsigma_\ni(\x, \u) \, \forall \x,\u \in \X \times \U \}$ that contains all statistically-plausible models compatible with \cref{as:well_calibrated_model}. The UCRL algorithm is
\begin{align}
   \label{eq:optimistic_exploration_background}
    \pi_\ni^\mathrm{UCRL} = \argmax_{\pi \in \Pi} \max_{ \tilde{f} \in \modelclass_\ni} J(\tilde{f}, \pi) .
\end{align}
\looseness -1 Instead of greedy exploitation, these algorithms optimize an optimistic policy that maximizes performance over all plausible models.
Unfortunately, this joint optimization is in general {\em intractable} and algorithms designed for greedy exploitation \cref{eq:expected_performance_exploration} do {\em not} generally solve the UCRL objective \cref{eq:optimistic_exploration_background}.



%% file: sections/3-optimistic_exploration.tex
\section{Hallucinated Upper Confidence Reinforcement Learning (\alg)} \label{sec:algorithm}
\begin{figure}[t]
  \centering \input{figures/optimistic_trajectory}
  \caption{\looseness -1 Illustration of the optimistic trajectory $\xo_\ti$ from \alg. 
  The policy $\pi$ is used to choose the next-state distribution, and the variables $\eta$ to choose the next state optimistically inside the one-step confidence interval (dark grey bars). 
  The true dynamics is contained inside the light grey confidence intervals, but, after the first step, not necessarily inside the dark grey bars.
  Even when the expected reward w.r.t.~the epistemic uncertainty is small (red cross compared to light grey bar), 
  \alg efficiently finds the high-reward region (red cross).
  Instead, greedy exploitation strategies fail.
  }
  \label{fig:optimistic_trajectory_new}
\end{figure}
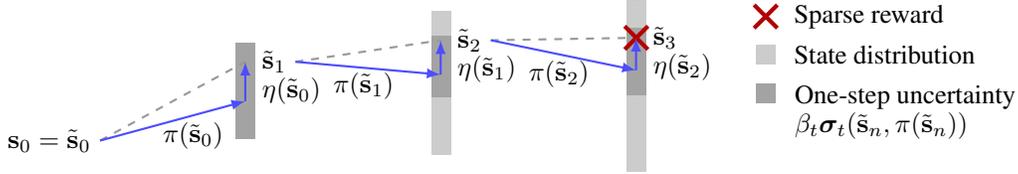

We propose a practical variant of the UCRL-exploration \eqref{eq:optimistic_exploration_background} algorithm. 
Namely, we reparameterize the functions $\tilde{f} \in \modelclass_\ni$ as  $\tilde{f} = \bmu_{\ni-1}(\x, \u) + \beta_{\ni-1} \bSigma_{\ni-1}(\x, \u) \eta(\x, \u)$, for some function $\eta \colon \R^\nstate \times \R^\ninp \to [-1, 1]^\nstate$. 
This transformation is similar in spirit to the re-parameterization trick from \citet{Kingma2013AutoEncoding}, except that $\eta(\x, \u)$ are functions.
The key insight is that instead of optimizing over dynamics in $\tilde{f} \in \modelclass_\ni$ as in UCRL, it suffices to optimize over the functions $\eta(\cdot)$. 
We call this algorithm \alg, formally:
\begin{align}
   \label{eq:optimistic_exploration}
    \pi_\ni^\mathrm{\alg} = \argmax_{\pi \in \Pi}\max_{ \eta(\cdot) \in [-1, 1]^\nstate}  J(\tilde{f}, \pi),  \st \tilde{f}(\x, \u) = \bmu_{\ni-1}(\x, \u) + \beta_{\ni-1} \bSigma_{\ni-1}(\x, \u) \eta(\x, \u) .
\end{align}
At a high level, the policy $\pi$ acts on the \emph{inputs} (actions) of the dynamics and chooses the next-state distribution. 
In turn, the optimization variables $\eta$ act in the \emph{outputs} of the dynamics to select the most-optimistic outcome from within the confidence intervals. 
We call the optimization variables the {\em hallucinated} controls as the agent hallucinates control authority to find the most-optimistic model.

The \alg algorithm \emph{does not explicitly propagate uncertainty} over the horizon.
Instead, it does so \emph{implicitly} by using the pointwise uncertainty estimates from the model to recursively plan an optimistic trajectory, as illustrated in \cref{fig:optimistic_trajectory_new}. This has the practical advantage that the model only has to be well-calibrated for 1-step predictions and not $\Ti$-step predictions.
In practice, the parameter $\beta_\ni$ trades off between exploration and exploitation.




\ActivateWarningFilters[pdftoc]
\subsection{Solving the Optimization Problem} 
\DeactivateWarningFilters[pdftoc]
\label{sec:practical}
\looseness -1 Problem \eqref{eq:optimistic_exploration} is still intractable as it requires to optimize over general functions.
The {\em crucial} insight is that we can make the \alg algorithm \cref{eq:optimistic_exploration} practical by optimizing over a smaller class of functions $\eta$. 
In \cref{ap:exploration:practical_implementation}, we prove that it suffices to optimize over Lipschitz-continuous bounded functions instead of general bounded functions. 
Therefore, we can optimize jointly over policies and Lipschitz-continuous, bounded \emph{functions} $\eta(\cdot)$.
Furthermore, we can re-write $\eta(\xo_\ti, \uo_\ti) = \eta(\xo_\ti, \pi(\xo_{\ti, \ni})) = \eta(\xo_{\ti, \ni})$. 
This allows to reduce the intractable optimistic problem \eqref{eq:optimistic_exploration} to {\em greedy exploitation} \eqref{eq:expected_performance_exploration}: We simply treat $\eta(\cdot) \in [-1, 1]^\nstate$ as an additional {\em hallucinated} control input that has no associated control penalties and can exert as much control as the current {\em epistemic} uncertainty that the model affords. 
With this observation in mind, \alg greedily exploits a {\em hallucinated} system with the extended dynamics $\tilde{f}$ in \cref{eq:optimistic_exploration} and a corresponding augmented control policy $(\pi, \eta)$. 
This means that we can now use the {\em same} efficient MBRL approaches for optimistic exploration that were previously restricted to greedy exploitation and Thompson sampling (albeit on a slightly larger action space, since the dimension of the action space increases from $\ninp$ to $\ninp + \nstate$).

In practice, if we have access to a greedy oracle $\pi = \texttt{GreedyOracle}(f)\xspace$, we simply access it using $\pi, \eta = \texttt{GreedyOracle}(\bmu_{\ni-1}  + \beta_{\ni-1} \bSigma_{\ni-1} \eta)\xspace$. 
Broadly speaking, greedy oracles are implemented using offline-policy search or online planning algorithms. 
Next, we discuss how to use these strategies independently to solve the \alg planning problem~\eqref{eq:optimistic_exploration}. 
For a detailed discussion on how to augment common algorithms with hallucination, see \Cref{ap:dyna-mpc}. 

\textbf{Offline Policy Search} is any algorithm that optimizes a parametric policy to maximize performance of the current dynamical model. 
As inputs, it takes the dynamical model and a parametric family for the policy and the critic (the value function).
It outputs the optimized policy and the corresponding critic of the optimized policy. 
These algorithms have fast inference time and scale to large dimensions but can suffer from model bias and inductive bias from the parametric policies and critics \citep{VanHasselt2019MBRLbias}.

\textbf{Online Planning} or Model Predictive Control \citep{Morari1999} is a local planning algorithm that outputs the best action for the current state. 
This method solves the \alg planning problem~\eqref{eq:optimistic_exploration} in a receding-horizon fashion. 
The planning horizon is usually shorter than $\Ti$ and the reward-to-go is bootstrapped using a terminal reward.
In most cases, however, this terminal reward is unknown and must be learned \citep{Lowrey2019POLO}.
As the planner observes the {\em true} transitions during deployment, it suffers less from model errors.
However, its running time is too slow for real-time implementation.

\paragraph{Combining Offline Policy Search with Online Planning}
In \Cref{alg:hucrl}, we propose to combine the best of both worlds to solve the \alg planning problem~\eqref{eq:optimistic_exploration}.
In particular, \Cref{alg:hucrl} takes as inputs a policy search algorithm and a planning algorithm. 
After each episode, it optimizes parametric (e.g. neural networks) control and hallucination policies $(\pi_{\theta}, \eta_{\theta})$ using the policy search algorithm. 
As a by-product of the policy search algorithm we have the {\em learned} critic $Q_{\vartheta}$.
At deployment, the planning algorithm returns the true and hallucinated actions $(a, a')$, and we only execute the true action $a$ to the true system.
We initialize the planning algorithm using the learned policies $(\pi_{\theta}, \eta_{\theta})$ and use the {\em learned} critic to bootstrap at the end of the prediction horizon. 
In this way, we achieve the best of both worlds. 
The policy search algorithm accelerates the planning algorithm by shortening the planning horizon with the learned critic and by using the learned policies to warm-start the optimization. 
The planning algorithm reduces the model-bias that a pure policy search algorithm has. 

\begin{algorithm}[t]
  \caption{\alg combining Optimistic Policy Search and Planning}
  \label{alg:hucrl}
  \begin{algorithmic}[1]
    \INPUT{Mean $\bmu(\cdot, \cdot)$ and variance $\bSigma^2(\cdot,\cdot)$, parametric policies $\pi_{\theta}(\cdot)$, $\eta_\theta(\cdot)$, parametric critic $Q_{\vartheta}(\cdot)$, horizon $\Ti$, policy search algorithm \texttt{PolicySearch}, online planning algorithm \texttt{Plan}}, 
    \For{$\ni = 1, 2, \dots$}
        \State {$(\pi_{\theta, \ni}, \eta_{\theta,\ni}), Q_{\vartheta,\ni} \gets  \texttt{PolicySearch}(\bmu_{\ni-1}; \bSigma_{\ni-1}^2; (\pi_{\theta,\ni-1}, \eta_{\theta,\ni-1}$))}
        \For{$\ti = 1, \dots, \Ti$}
        \State $(\u_{\ti-1, \ni}, \u'_{\ti-1, \ni}) = \texttt{Plan}(\x_{\ti-1, \ni}; \bmu_{\ni-1}; \bSigma_{\ni-1}^2; (\pi_{\theta,\ni}, \eta_{\theta,\ni}), Q_{\vartheta})$
        \State $\x_{\ti, \ni} = f(\x_{\ti-1, \ni}, \u_{\ti-1, \ni}) + \noise_{\ti-1, \ni}$
         \EndFor
    \State Update statistical dynamical model with the $\Ti$ observed state transitions in $\dataset_\ni$.
    \EndFor
  \end{algorithmic}
\end{algorithm}

\subsection{Theoretical Analysis} \label{ssec:TheoreticalAnalysis}
\looseness -1 
In this section, we analyze the \alg algorithm \cref{eq:optimistic_exploration}. A natural quality criterion to evaluate exploration schemes is the {\em cumulative regret} $R_\Ni = \sum_{\ni = 1}^\Ni | J(f, \pi^*) - J(f, \pi_\ni) |$, which is the difference in performance between the optimal policy $\pi^*$ and $\pi_\ni$ on the true system $f$ over the run of the algorithm \citep{Chowdhury2019Online}.
If we can show that $R_\Ni$ is sublinear in $\Ni$, then we know that the performance $J(f, \pi_\ni)$ of our chosen policies $\pi_\ni$ converges to the performance of the optimal policy $\pi^*$.
We first introduce the final assumption for the results in this section to hold.

\begin{assumption}[Continuity]
    \looseness -1 
    The functions $\bmu_\ni$ and $\bsigma_\ni$ are $L_\mu$ and $L_\sigma$ Lipschitz continuous, any policy $\pi \in \Pi$ is $L_\pi$-Lipschitz continuous 
    and the reward $r(\cdot, \cdot)$ is $L_r$-Lipschitz continuous.
    \label{as:model_predictions_lipschitz}
    \label{as:pi_lipschitz}
    \label{as:action_set_compact}
    \label{as:reward_lipschitz}
\end{assumption}

\cref{as:pi_lipschitz,as:reward_lipschitz} is not restrictive.
NN with Lipschitz-continuous non-linearities or GP with Lipschitz-continuous kernels output Lipschitz-continuous predictions (see \cref{ap:gp_predictions_lipschitz}).  Furthermore, we are free to choose the policy class $\Pi$, and most reward functions are either quadratic or tolerance functions \citep{Tassa2018DeepMind}. 
Discontinuous reward functions are generally very difficult to optimize.

\paragraph{Model complexity}
In general, we expect that $R_\Ni$ depends on the complexity of the statistical model in \cref{as:well_calibrated_model}. If we can quickly estimate the true model using a few data-points, then the regret would be lower than if the model is slower to learn. To account for these differences, we construct the following complexity measure over a given set $\X$ and $\U$,
\begin{equation}
    \label{eq:model_prediction_variance_bound}
    I_\Ni(\X, \U) = \max_{\dataset_1, \dots, \dataset_\Ni \subset \X \times \X \times \U, \, |\dataset_\ni| = \Ti} \sum\nolimits_{\ni=1}^\Ni \sum_{\x, \u \in \dataset_\ni} \| \bsigma_{\ni-1}(\x, \u) \|_2^2 .
\end{equation}
While in general impossible to compute, this complexity measure considers the ``worst-case'' datasets $\dataset_{1}$ to $\dataset_\Ni$, with $|\dataset_\ni| = \Ti$ elements each, that we could collect at each iteration of \cref{alg:mbrl} in order to maximize the predictive uncertainty of our statistical model. Intuitively, if $\bsigma(\x, \u)$ shrinks sufficiently quickly after observing a transition $(\cdot, \x, \u)$ and if the model generalizes well over $\X \times \U$, then \cref{eq:model_prediction_variance_bound} will be small. In contrast, if our model does not learn or generalize at all, then $I_\Ni$ will be $\Or{\Ni\Ti\nstate}$ and we cannot hope to succeed in finding the optimal policy. For the special case of Gaussian process (GP) models, we show that $I_\Ni$ is indeed sublinear in the following.

\paragraph{General regret bound}
\looseness -1
The true sequence of states $\x_{\ti,\ni}$ at which we obtain data during our rollout in \cref{alg:mbrl:rollout} of \cref{alg:mbrl} lies somewhere withing the light-gray shaded state distribution with epistemic uncertainty in \cref{fig:optimistic_trajectory_new}. While this is generally difficult to compute, we can bound it in terms of the predictive variance $\bsigma_{\ni-1}(\x_{\ti,\ni}, \pi_\ni(\x_{\ti,\ni}))$, which is directly related to $I_\Ni$. However, the optimistically planned trajectory instead depends on $\bsigma_{\ni-1}(\xo_{\ti,\ni}, \pi(\xo_{\ti,\ni}))$ in \cref{eq:optimistic_exploration}, which enables policy optimization without explicitly constructing the state distribution. How the predictive uncertainties of these two trajectories relate depends on the generalization properties of our statistical model; specifically on $L_\sigma$ in \cref{as:model_predictions_lipschitz}. We can use this observation to obtain the following bound on $R_\Ni$:




\begin{restatable}{theorem}{generalregretbound}
  \label{thm:exploration:regret:general_regret_bound}
   Under \cref{as:dynamics_f_lipschitz,as:pi_lipschitz,as:reward_lipschitz,as:well_calibrated_model,as:model_predictions_lipschitz} let $\x_{\ti,\ni} \in \X$ and $\u_{\ti,\ni} \in \U$ for all $\ti,\ni>0$. Then, for all $\Ni \geq 1$, with probability at least $(1-\delta)$, the regret of \alg in \cref{eq:optimistic_exploration} is at most
  $
    R_\Ni \leq \Or{L_\sigma^\Ti \beta_{\Ni-1}^\Ti \sqrt{ \Ni \Ti^3  \, I_\Ni(\X, \U) }  }
  $.
\end{restatable}

We provide a proof of \cref{thm:exploration:regret:general_regret_bound} in \cref{ap:exploration_proof}. The theorem ensures that, if we evaluate optimistic policies according to \cref{eq:optimistic_exploration}, we eventually achieve performance $J(f, \pi_\ni)$ arbitrarily close to the optimal performance of $J(f, \pi^*)$ if $I_\Ni(\X, \U)$ grows at a rate smaller than $\Ni$. As one would expect, the regret bound in \cref{thm:exploration:regret:general_regret_bound} depends on constant factors like the prediction horizon $\Ti$, the relevant Lipschitz constants of the dynamics, policy, reward, and the predictive uncertainty. 
The dependence on the dimensionality of the state space $\nstate$ is hidden inside $I_\Ni$, while $\beta_\ni$ is a function of $\delta$.

\paragraph{Gaussian Process Models}
\looseness -1 
For the bound in \cref{thm:exploration:regret:general_regret_bound} to be useful, we must show that $I_\Ni$ is sublinear. Proving this is impossible for general models, but can be proven for GP models. In particular, we show in \cref{ap:sec:gp_regret_bounds} that $I_\Ni$ is bounded by the worst-case mutual information (information capacity) of the GP model. \citet{Srinivas2012Gaussian,Krause2011Contextual} derive upper-bounds for the information capacity for commonly-used kernels. For example, when we use their results for independent GP models with squared exponential kernels for each component $[f(\x, \u)]_i$, we obtain a regret bound $\mathcal{O}(\, (1+B_f)^\Ti  L_\sigma^\Ti \Ti^2 \sqrt{\Ni} (\nstate^2 (\nstate + \ninp) \log(\nstate \Ni \Ti))^{(\Ti+1) / 2} )$, where $B_f$ is a bound on the functional complexity of the function $f$. Specifically, $B_f$ is the norm of $f$ in the RKHS that corresponds to the kernel.

A similar optimistic exploration scheme was analyzed by \citet{Chowdhury2019Online}, but for an algorithm that is not implementable as we discussed at the beginning of \cref{sec:algorithm}. 
Their exploration scheme \emph{depends} on the (generally unknown) Lipschitz constant of the value function, which corresponds to knowing $L_f$ \emph{a priori} in our setting. While this is a restrictive  and impractical requirement, we show in \cref{ap:sec:comparison_to_chowdhury2019} that under this assumption we can improve the dependence on $L_\sigma^\Ti \beta_\Ni^\Ti$ in the regret bound in \cref{thm:exploration:regret:general_regret_bound} to $(L_f \beta_\Ni)^{1/2}$. This matches the bounds derived by \citet{Chowdhury2019Online} up to constant factors. Thus we can consider the regret term $L_\sigma^\Ti \beta_\Ni^\Ti$ to be the additional cost that we have to pay for a practical algorithm. 

\paragraph{Unbounded domains}
\looseness -1
We assume that the domain $\X$ is compact in order to bound $I_\Ni$ for GP models, which enables a convenient analysis and is also used by \citet{Chowdhury2019Online}. 
However, it is incompatible with \cref{as:transition_noise_sub_gaussian}, which allows for potentially unbounded noise $\noise$. While this is a technical detail, we formally prove in \cref{ap:unbounded_domains} that we can bound the domain with high probability within a norm-ball of radius $b_\ni = \mathcal{O}(L_f^\Ti \Ti \nstate \log(\Ti \ni^2))$. For GP models with a squared exponential kernel, we analyze $I_\Ni$ in this setting and show that the regret bound only increases by a polylog factor.


%% file: figures/optimistic_trajectory.tex
\tikzset{cross/.style={cross out, draw=black, minimum size=2*(#1-\pgflinewidth), inner sep=0pt, outer sep=0pt},
cross/.default={1pt}}

\begin{tikzpicture}[x=1in, y=1in]

\node at (0, 0) 	[name=start] 	                {$\x_0 = \xo_0$};
\node [name=o1, above right=0.2 and .8 of start]    {$\xo_1$};
\node [name=o2, above right=-0.1 and .8 of o1] 		{$\xo_2$};
\node [name=o3, above right=-0.2 and .8 of o2] 		{$\xo_3$};

\node 	[name=mu1, below left=0.2 and 0.01 of o1, anchor=west]    		{};


\fill[black, opacity=0.2] ($(o1) + (-0.2,-0.4)$) rectangle ($(o1) + (-0.1,0.1)$);
\fill[black, opacity=0.2] ($(o2) + (-0.2,-0.3)$) rectangle ($(o2) + (-0.1,0.02)$);
\fill[black, opacity=0.2] ($(o3) + (-0.2,-0.3)$) rectangle ($(o3) + (-0.1,+0.05)$);

\fill[black, opacity=0.2] ($(o1) + (-0.2,-0.4)$) rectangle ($(o1) + (-0.1,0.1)$);
\fill[black, opacity=0.2] ($(o2) + (-0.2,-0.6)$) rectangle ($(o2) + (-0.1,0.15)$);
\fill[black, opacity=0.2] ($(o3) + (-0.2,-0.7)$) rectangle ($(o3) + (-0.1,+0.2)$);

\draw[black!40, thick, dashed] (start.east) -- ($(o1.west) + (-0.04, 0) $);

\draw[black!40, thick, dashed] (o1.east) -- ($(o2.west) + (-0.04, 0) $);

\draw[black!40, thick, dashed] (o2.east) -- ($(o3.west) + (-0.04, 0) $);

\draw ($(o3) - (0.15, 0)$) node[cross=6pt,black!30!red, ultra thick]{};

\node 	[name=pi0, above right=-0.2 and 0.28 of start]  {$\pi(\xo_0)$};
\node 	[name=eta0, below=0.15 of o1.west,anchor=west]   {$\eta(\xo_0)$};

\node 	[name=pi1, below right=-0.1 and 0.15 of o1]    	{$\pi(\xo_1)$};
\node 	[name=eta1, below=0.15 of o2.west,anchor=west]	{$\eta(\xo_1)$};


\node 	[name=pi2, below right=-0.05 and 0.15 of o2]    	{$\pi(\xo_2)$};
\node 	[name=eta2, below=0.15 of o3.west,anchor=west]	{$\eta(\xo_2)$};

\draw[blue!70, thick, ->] (start.east) -- ($(mu1.west) + (-0.03, 0.1) $);
\draw[blue!70, thick, ->] ($(mu1.west) + (-0.03, 0.1) $) --($(o1.west) + (-0.04, 0) $);

\draw[blue!70, thick, ->] (o1.east) -- ($(o2.west) + (-0.04, -0.18)$);
\draw[blue!70, thick, ->] ($(o2.west) + (-0.04, -0.18)$) --($(o2.west) + (-0.04, 0) $);

\draw[blue!70, thick, ->] (o2.east) -- ($(o3.west) + (-0.04, -0.17)$);
\draw[blue!70, thick, ->] ($(o3.west) + (-0.04, -0.17)$) --($(o3.west) + (-0.04, 0) $);

\draw (3.75,.65) node[cross=6pt,black!30!red, ultra thick]{};
\fill[opacity=0.2,black] (3.7,.4) rectangle (3.8,.5);
\fill[opacity=0.2,black] (3.7,.2) rectangle (3.8,.3);
\fill[opacity=0.2,black] (3.7,.2) rectangle (3.8,.3);
\node at (3.85, .65) [anchor=west] {Sparse reward};
\node at (3.85, .45) [anchor=west] {State distribution};
\node at (3.77, .15) [anchor=west] {\begin{tabular}{l}
One-step uncertainty \\
$\beta_\ni \bsigma_\ni(\xo_\ti, \pi(\xo_\ti))$ \\
\end{tabular}};

\end{tikzpicture}

%% file: sections/4-experiments.tex
\section{Experiments} \label{sec:experiments}
\looseness -1 Throughout the experiments, we consider reward functions of the form $r(\x, \u) = r_{\text{state}}(\x) - \rho c_{\text{action}}(\u)$, where $r_{\text{state}}(\x)$ is the reward for being in a ``good'' state, and $\rho \in [0, \infty)$ is a parameter that scales the action costs $c_{\text{action}}(\u)$.
We evaluate how \alg, greedy exploitation, and Thompson sampling perform for different values of $\rho$ in different Mujoco environments \citep{Todorov2012Mujoco}.
We expect greedy exploitation to struggle for larger $\rho$, whereas \alg and Thompson sampling should perform well.
As modeling choice, we use 5-head probabilistic ensembles as in \citet{Chua2018Deep}. 
For greedy exploitation, we sample the next-state from the ensemble mean and covariance (PE-DS algorithm in \citet{Chua2018Deep}). 
We use ensemble sampling \citep{Lu2017Ensemble} to approximate Thompson sampling.
For \alg, we follow \citet{Lakshminarayanan2017Simple} and use the ensemble mean and covariance as the next-state predictive distribution.
For more experimental details and learning curves, see \Cref{ap:extended_experiments}.
We provide an open-source implementation of our method, which is available at \url{http://github.com/sebascuri/hucrl}.

\paragraph{Sparse Inverted Pendulum}
We first investigate a swing-up pendulum with sparse rewards.
In this task, the policy must perform a complex maneuver to swing the pendulum to the upwards position.
A policy that does not act obtains zero state rewards but suffers zero action costs. 
Slightly moving the pendulum still has zero state reward but the actions are penalized. 
Hence, a zero-action policy is locally optimal, but it fails to complete the task. 
We show the results in \cref{fig:inverted_final}:
With no action penalty, all exploration methods perform equally well -- the randomness is enough to explore and find a quasi-optimal sequence.
For $\rho=0.1$, greedy exploitation struggles: sometimes it finds the swing-up sequence, which explains the large error bars. 
Finally, for $\rho=0.2$ only \alg is able to successfully swing up the pendulum.

\begin{figure}[t]
  \centering\includegraphics[width=\linewidth]{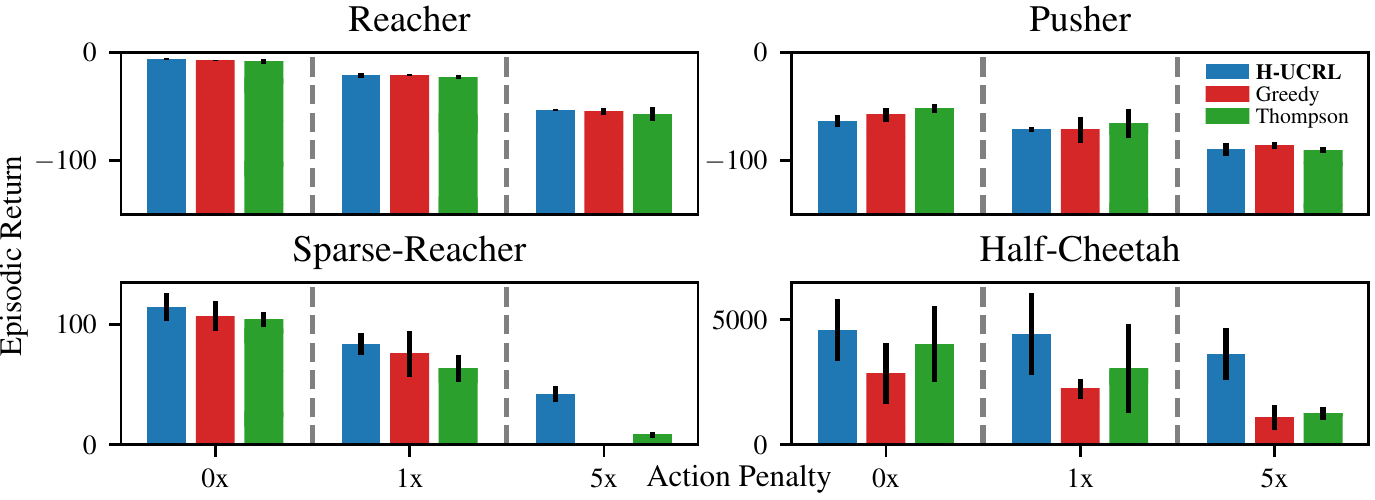}
  \caption{Mean final episodic returns on Mujoco tasks averaged over five different random seeds. For Reacher and Pusher (50 episodes), all exploration strategies perform equally. For Sparse-Reacher (50 episodes) and Half-Cheetah (250 episodes), \alg outperforms other exploration algorithms.}
  \label{fig:mujoco_final}
\end{figure}

\begin{figure}[htpb]
\includegraphics[width=\linewidth]{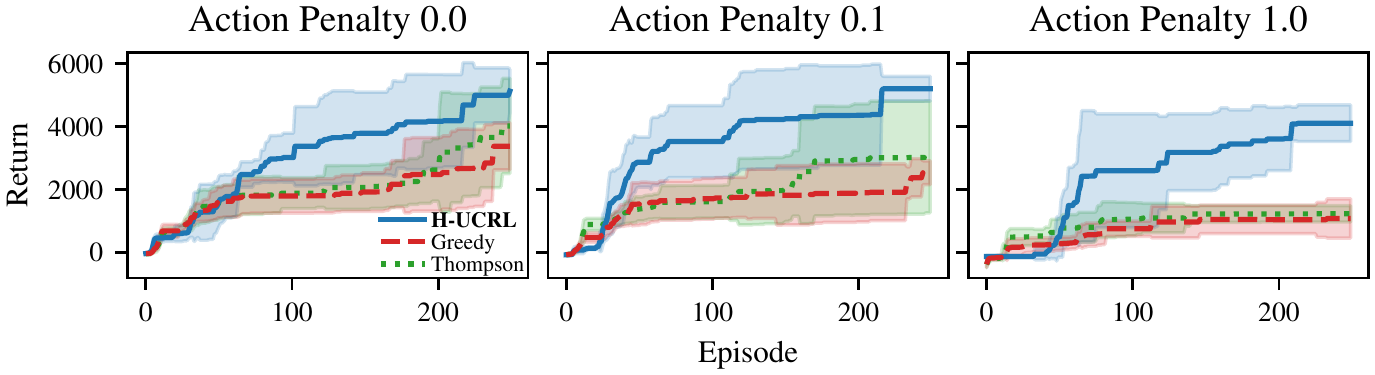}
\caption{Learning curves in Half-Cheetah environment. For all action penalties, \alg learns faster than greedy and Thompson sampling strategies. For larger action penalties, greedy and Thompson lead to insufficient exploration and get stuck in local optima with poor performance.} \label{fig:cheetah_learning_main}
\end{figure}

\paragraph{7-DOF PR2 Robot}
Next, we evaluate how \alg performs in higher-dimensional problems.
We start by comparing the Reacher and Pusher environments proposed by \citet{Chua2018Deep}. 
We plot the results in the upper left and right subplots in \cref{fig:mujoco_final}. 
The Reacher has to move the end-effector towards a goal that is randomly sampled at the beginning of each episode. 
The Pusher has to push an object towards a goal. 
The rewards and costs in these environments are quadratic. 
All exploration strategies achieve state-of-the-art performance, which seems to indicate that greedy exploitation is indeed sufficient for these tasks. 
Presumably, this is due to the over-actuated dynamics and the reward structure. 
This is in line with the theoretical results for linear-quadratic control by \citet{Mania2019CertaintyEquivalence}.

To test this hypothesis, we repeat the Reacher experiment with a sparse reward function.
We plot the results in the lower left plot of \cref{fig:mujoco_final}.
The state reward has a positive signal when the end-effector is close to the goal and the action has a non-negative signal when it is close to zero.
Here we observe that \alg outperforms alternative methods, particularly for larger action penalties.

\paragraph{Half-Cheetah}
\looseness -1 Our final experiment demonstrates \alg on a common deep-RL benchmark, the Half-Cheetah.
The goal is to make the cheetah run forward as fast as possible.
The actuators have to interact in a complex manner to achieve running.
In \cref{fig:cheetah_learning_main}, we can see a clear advantage of using \alg at different action penalties, even at zero. 
This indicates that \alg not only addresses action penalties, but also explores through complex dynamics. 
For the sake of completeness, we also show the final returns in the lower right plot of \cref{fig:mujoco_final}.

\looseness -1 
\paragraph{\alg vs. Thompson Sampling} In \Cref{ap:ThompsonSampling}, we carry out extensive experiments to empirically evaluate why Thompson sampling fails in our setting.
\citet{phan2019thompson} in the Bandit Setting and \citet{kakade2020information} in the RL setting also report that approximate Thompson sampling fails unless strong modelling priors are used.  
We believe that the poor performance of Thompson sampling relative to \alg suggests that the models that we use are sufficient to construct well-calibrated 1-step ahead confidence intervals, but do not comprise a rich enough posterior distribution for Thompson sampling.
As an example, in \alg we use the five members of the ensemble to construct the 1-step ahead confidence interval at every time-step. 
On the other hand, in Thompson sampling we sample a {\em single} model from the {\em approximate} posterior for the full horizon. 
It is possible that in some regions of the state-space one member is more optimistic than others, and in a different region the situation reverses. 
This is not only a property of ensembles, but also other approximate models such as random-feature GP models (c.f.~\Cref{ap:TS_approxGP}) exhibit the same behaviour. 
This discussion highlights the advantage of \alg over Thompson sampling using deep neural networks: \alg only requires calibrated 1-step ahead confidence intervals, and we know how to construct them (c.f.~\citet{Malik2019Calibrated}). 
Instead, Thompson sampling requires posterior models that are calibrated throughout the full trajectory.
Due to the multi-step nature of the problem, constructing {\em scalable} approximate posteriors that have enough variance to sufficiently explore is still an open problem. 

%% file: sections/5-conclusion.tex
\section{Conclusions}
In this work, we introduced \alg: a practical optimistic-exploration algorithm for deep MBRL. 
The key idea is a reduction from (generally intractable) optimistic exploration to greedy exploitation in an augmented policy space.
Crucially, this insight enables the use of highly effective standard MBRL algorithms that previously were restricted to greedy exploitation and Thompson sampling.
Furthermore, we provided a theoretical analysis of \alg and show that it attains sublinear regret for some models. 
In our experiments, \alg performs as well or better than other exploration algorithms, achieving state-of-the-art performance on the evaluated tasks.

%% file: appendix/expected_performance.tex
\section{Expected Performance for Exploration in the Bandit Setting}
\label{sec:expected_performance_bandit}
\label{ap:expected_performance_bandit}

In practice, one of the most commonly used exploration strategies is to select $\btheta_\ni$ in order to maximize the expected performance over the aleatoric uncertainty \emph{and epistemic uncertainty} induced by the Gaussian process model. 

We consider the simplest possible case that still allows for nonlinear dynamics. That is, we consider a system with zero-mean noise, i.e., $\E{\noise_\ti = \mb{0}}$ for all time steps $\ti \geq 0$. In addition, we consider a one-dimensional system, $\nstate = 1$, with a linear (convex/concave) reward function $r(\x, \u) = \x$, a constant feedback policy $\pi(\x) = \btheta$ that is parameterized by some parameters $\btheta$, and a time horizon of one step, $\Ti=1$. With these simplifying assumptions, the performance estimate $J(f, \pi)$ in \cref{eq:performance_estimator} reduces to
\begin{equation}
\begin{aligned}
    \label{eq:performance_estimator_bandit_setting}
    J(\tilde{f}, \pi) &= \E[\noise_{0:\Ti-1}]{\sum_{\ti =0}^\Ti r(\xo_\ti, \pi(\xo_\ti) ) \, \bigg| \, \x_0},
    \quad \st \xo_{\ti + 1} = \tilde{f}(\xo_\ti, \pi(\xo_\ti)) + \noise_\ti ,\\
    &= \E[\noise_{0:\Ti-1}]{
        \sum_{\ti =0}^\Ti r(\xo_\ti, \pi(\xo_\ti))
        \, \bigg| \, \x_0
    },
    \quad \st \xo_{\ti + 1} = \tilde{f}(\xo_\ti, \btheta)  + \noise_\ti ,
    &&\text{($\pi(\x)=\btheta$)} \\
    &= \E[\noise_{0:\Ti-1}]{
        \sum_{\ti =0}^\Ti \xo_\ti 
        \, \bigg| \, \x_0
    },
    \quad \st \xo_{\ti + 1} = \tilde{f}(\xo_\ti, \btheta)  + \noise_\ti , &&\text{($\nstate=1$, $r(\x, \u) = \x$)} \\
    &= \E[\noise_{0}]{
        \x_0 + \xo_1 
        \, \bigg| \, \x_0
    },
    \quad \st \xo_{1} = \tilde{f}(\x_0, \btheta) + \noise_0 , &&\text{($\Ti=1$)} \\
    &= \x_0 + \tilde{f}(\x_0, \btheta) + \E[\noise_0]{\noise_0} ,\\
    &= \x_0 + \tilde{f}(\x_0, \btheta), &&\text{($\E{\noise} = 0$)}
\end{aligned}
\end{equation}
so that the overall goal of model-based reinforcement learning in \cref{eq:exploration:regret:optimal_parameters} becomes 
\begin{align}
    \label{eq:bandit_optimal_parameters}
    \btheta^* &= \argmax_{\pi_\btheta } J(f, \pi_\btheta) ,\\ 
    &= \argmax_{\btheta} \x_0 + f(\x_0, \btheta) ,\\
    &= \argmax_{\btheta} f(\x_0, \btheta) .
\end{align}
This is the simplest possible scenario and reduces the optimal control problem in \cref{eq:expected_performance_exploration} to the bandit problem, where want to maximize an unknown function $f$ that depends on parameters $\btheta$ together with a fixed context $\x_0$ that does not impact the solution of the problem.

Algorithms that model the unknown function $f$ in \cref{eq:bandit_optimal_parameters} with a probabilistic model $\modelposterior$ based on noisy observations in $\dataset_\ni$ are called Bayesian optimization algorithms \citep{Brochu2010Tutorial}. In this special case of model-based reinforcement learning, the expected performance objective \cref{eq:expected_performance_exploration} reduces to 
\begin{align}
  \label{eq:expected_performance_exploration_bandit}
    \btheta_\ni &= \argmax_{\btheta} \E[\tilde{f} \sim \modelposterior]{ J(\tilde{f}, \pi_\btheta)} ,\\
    &= \argmax_{\btheta} \E[\tilde{f} \sim \modelposterior]{ \x_0 + \tilde{f}(\x_0, \btheta) } ,\\
    &= \argmax_{\btheta} \E[\tilde{f} \sim \modelposterior]{ \tilde{f}(\x_0, \btheta) } ,\\
    &= \argmax_{\btheta} \bmu_{\ni-1}(\x_0, \btheta) .
\end{align}
Thus the expected performance objective selects parameters $\btheta_\ni$ that maximize the posterior mean estimate of $f$ according to \modelposterior. This may seem natural, since the linear reward function encourages states that are as large as possible. However, in the Bayesian optimization literature \cref{eq:expected_performance_exploration_bandit} is equivalent to the \textsc{UCB} strategy with $\beta_\ni = 0$. This is a greedy algorithm that is well-known to get stuck in local optima \citep{Srinivas2012Gaussian}.

This is illustrated in \cref{fig:expected_performance_comparison}: We use a Gaussian process model for $f$ and use \cref{eq:expected_performance_exploration_bandit}, which means we set $\beta=0$ in the \textsc{GP-UCB} algorithm. As a result, we obtain optimization behaviors as in \cref{fig:expected_performance_comparison_beta_0}. The first evaluation that achieves performance higher than the expected prior performance (in our case, zero), is evaluated repeatedly (orange crosses). However, this can correspond to a local optimum of the true, unknown objective function (black dashed). In contrast, if we use an optimistic algorithm and set $\beta=2$, \textsc{GP-UCB} evaluates parameters with close-to-optimal performance.

As a consequence of this counter-example, it is clear that we cannot expect the expected performance exploration criterion in \cref{eq:expected_performance_exploration} to yield regret guarantees for exploration \emph{in the general case}. However, under the additional assumption of linear dynamics, \citet{Mania2019CertaintyEquivalence} show that the algorithm is no-regret. More empirically, \citet[Section 6.1]{Deisenroth2014Gaussian} discuss how to choose specific reward functions that tend to encourage high-variance transitions and thus exploration. However, it is unclear how such an approach can be analyzed theoretically and we would prefer to avoid reward-shaping to encourage exploration.

\begin{figure}
  \begin{subfigure}[t]{0.49\textwidth}
    \includegraphics[width=\linewidth]{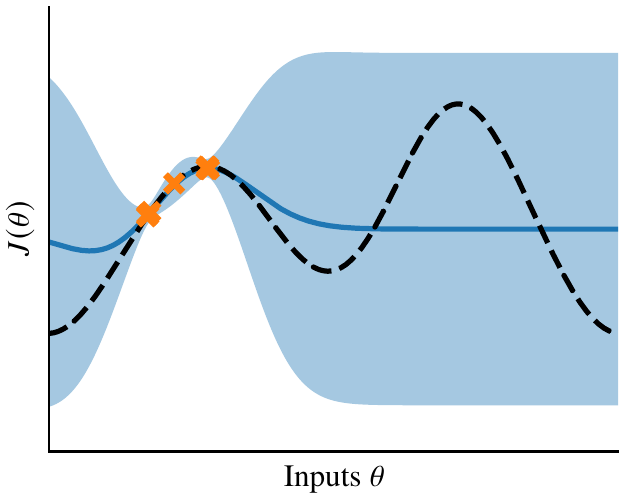}
    \caption{$\beta_\ni = 0.$}
    \label{fig:expected_performance_comparison_beta_0}
  \end{subfigure}%
  \begin{subfigure}[t]{0.49\textwidth}
    \includegraphics[width=\linewidth]{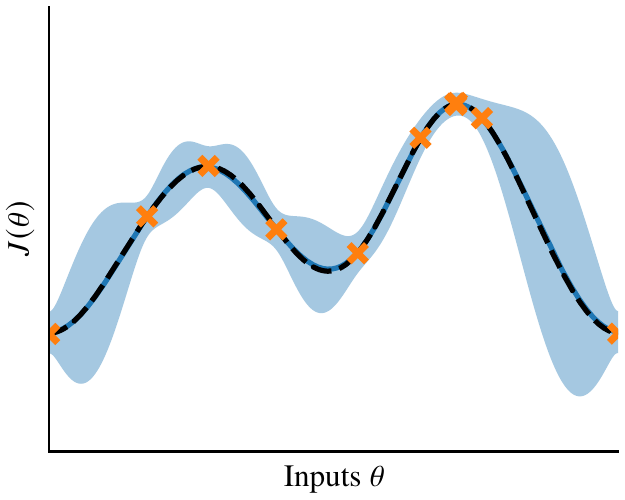}
    \caption{$\beta_\ni = 2.$}
    \label{fig:expected_performance_comparison_beta_2}
  \end{subfigure}
  \caption{Comparison of the \textsc{GP-UCB} algorithm with two different constants for $\beta_\ni$. The expected performance objective in \cref{eq:expected_performance_exploration_bandit} is equivalent setting to $\beta = 0$ in \cref{fig:expected_performance_comparison_beta_0}. The algorithm gets stuck and repeatedly evaluates inputs (orange crosses) at a local optimum of the true objective function (black dashed). This is due to the mean function (blue line) achieving higher values than the prior expected performance of zero. In contrast, an optimistic algorithm with $\beta=2$ in \cref{fig:expected_performance_comparison_beta_2} determines close-to-optimal parameters after few evaluations.}
  \label{fig:expected_performance_comparison}
\end{figure}

%% file: appendix/extra_experiments.tex

\section{Extended Experiments}
\label{ap:extended_experiments}

\subsection{Experimental Setup}

\paragraph{Models}
We consider ensembles of Probabilistic Neural Networks (PE) as in \citet{Chua2018Deep} and Gaussian Process (GP) Models for the inverted pendulum as in \citet{Kamthe2018DataEfficient}.
For GPs, we use the predictive variance estimate as $\bSigma_{\ni-1}(\x, \u)$
For Ensembles, we approximate the output of the ensemble with a Gaussian as suggested by \citet{Lakshminarayanan2017Simple} and use its predictive mean and variance as $\bmu_{\ni-1}(\x, \u)$ and $\bSigma_{\ni-1}(\x, \u)$.

\paragraph{Model Selection (Training)}
For GPs we do not optimize the Hyper-parameters as this is prone to getting stuck to local minima \citep{Bull2011Convergence}. Advanced methods to avoid this problem, such as those proposed by \citet{Berkenkamp2019NoRegret}, are left for future work. 
For Ensembles, we train each ensemble separately using Adam \citep{Kingma2014Adam}. 
We assign a transition to each ensemble member sampling from a Poisson distribution $\text{Poi}(1)$ \citep{Osband2016Deep}. 
This is an asymptotic approximation to the Bootstrap. 

\paragraph{Approximate Thompson Sampling}
We do not consider a Thompson sampling variant of Exact GPs due to the computational complexity. 
For PE, we sample at the beginning of each episode a head and use \emph{only} this head for optimizing the policy as in \citet{Lu2017Ensemble}. 

\paragraph{Trajectory Sampling}
For greedy exploitation, we propagate particles and the next-state distribution is given by the ensemble (or GP) output at the current particle location. 
This is the PE-DS algorithm from \citet{Chua2018Deep}, which has comparable performance to PE-TS1 and PE-TS$\infty$. 
We use this algorithm because it has the same predictive uncertainty used by \alg. 

\paragraph{Policy Search and Planning Algorithm}
For experiments, we use a modification of MPO \citep{Abdolmaleki2018MPPO} with Hallucinated Data Augmentation to simulate data and Hallucinated Value Expansion to compute targets as the $\texttt{PolicySearch}$ algorithm. 
As the resulting algorithm is on-policy, we only learn a value function as critic. 
The planning algorithm is implemented using Dyna-MPC from \Cref{alg:dyna-mpc}.
We update the sampling distribution using the Cross-Entropy Method from \citet{Botev2013CEM}. 
We provide an open-source implementation of our method, which is available at \url{http://github.com/sebascuri/hucrl} that builds upon the RL-LIB library from \citet{Curi2020RLLib}, based on pytorch \citep{paszke2017automatic}. 

\subsection{Environment Description and Learning Curves}

\FloatBarrier
\subsubsection{Swing-Up Inverted Pendulum}
The pendulum has $\nstate=2$ and $\ninp=1$, with actions bounded in $[-1, 1]$ and each episode lasts 400 time steps.. We transform the angles to a quaternion representation via $[\sin(\theta), \cos(\theta)]$. The pendulum starts at $\theta_0 = \pi$, $\omega_0 = 0$ and the objective is to swing it up to $\theta_0 = 0$, $\omega_0 = 0$. The reward function is $r(\theta, \omega, \u) = r_\theta \cdot r_\omega + \rho r_\u$, where $r_\theta = \textsc{tolerance}(\cos(\theta), \text{bounds}=(0.95, 1.), \text{margin}=0.1)$, $r_\omega = \textsc{tolerance}(\omega, \text{bounds}=(-0.5, 0.5), \text{margin}=0.5)$, and $r_\u = \textsc{tolerance}(\u, \text{bounds}=(-0.1, 0.1), \text{margin}=0.1) -1$. The $\textsc{tolerance}$ is defined in \citet{Tassa2018DeepMind}. In \cref{fig:inverted_evolution} we show the learning curve of the PE model for five different random seeds. \alg finds quickly a swing-up maneuvere even with high action penalties.  
        
\begin{figure}[ht]
  \includegraphics[width=\linewidth]{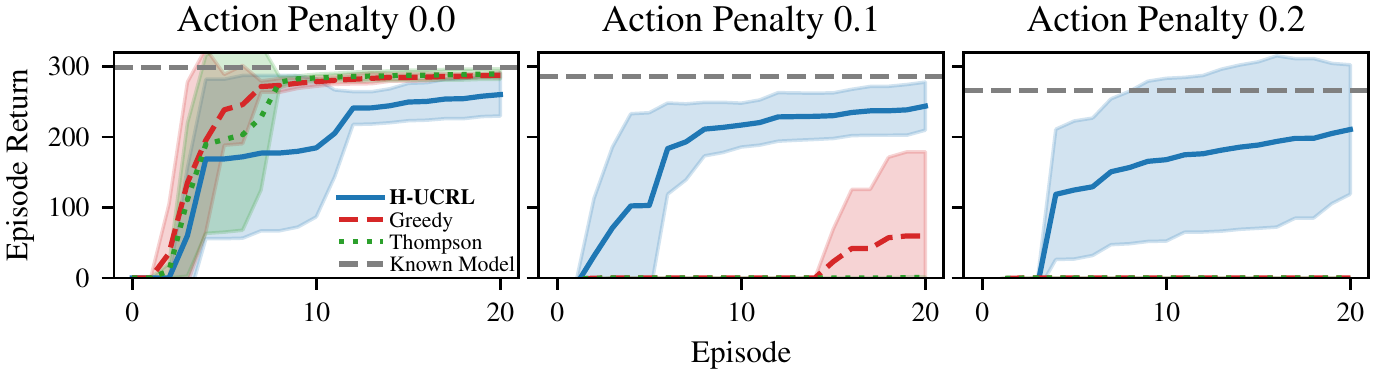}
  \caption{Learning curves of the inverted pendulum. \alg outperforms other algorithms during learning.}
  \label{fig:inverted_evolution}
\end{figure}

\FloatBarrier
\subsubsection{Mujoco Cart Pole}
We repeat the experiment in a easy environment, the Mujoco Cart Pole. 
The cart-pole has $\nstate=4$ and $\ninp=1$, with actions bounded in $[-3, 3]$ and each episode lasts 200 time steps.
We transform the angles to a quaternion representation via $[\sin(\theta), \cos(\theta)]$.
The cart-pole starts from $(0,0,0,0) + \omega$, where $\omega$ is a zero-mean normal noise with $0.1$ standard deviation. The goal is to upswing and stabilize the end-effector at position $x=0$. The reward is given by $r = e^{-\sum_{i=x,y}\text{ee}^2_i / 0.6^2} - \rho \u^2$, where $\text{ee}$ is vector of coordinates of the end-effector. 
Here we see again that, as the action penalty increases, expected and Thompson sampling do not find a swing-up maneuver. We plot the final results together with the learning curves in \cref{fig:mujoco_cartpole}. 

\begin{figure}[ht]
    \centering
    \begin{subfigure}{\textwidth}
    \centering
        \includegraphics[width=\linewidth]{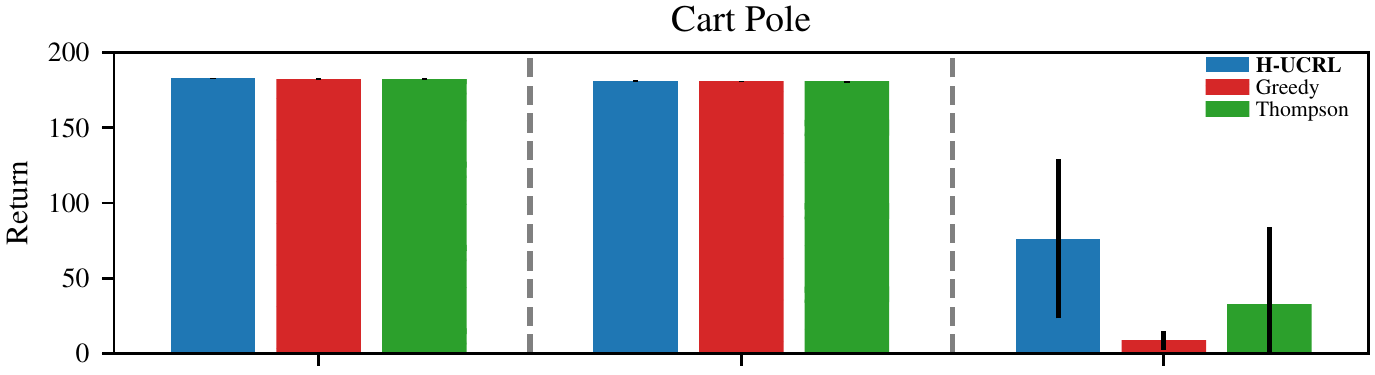}
    \end{subfigure} \\ 
    \begin{subfigure}{\textwidth}
    \centering
        \includegraphics[width=\linewidth]{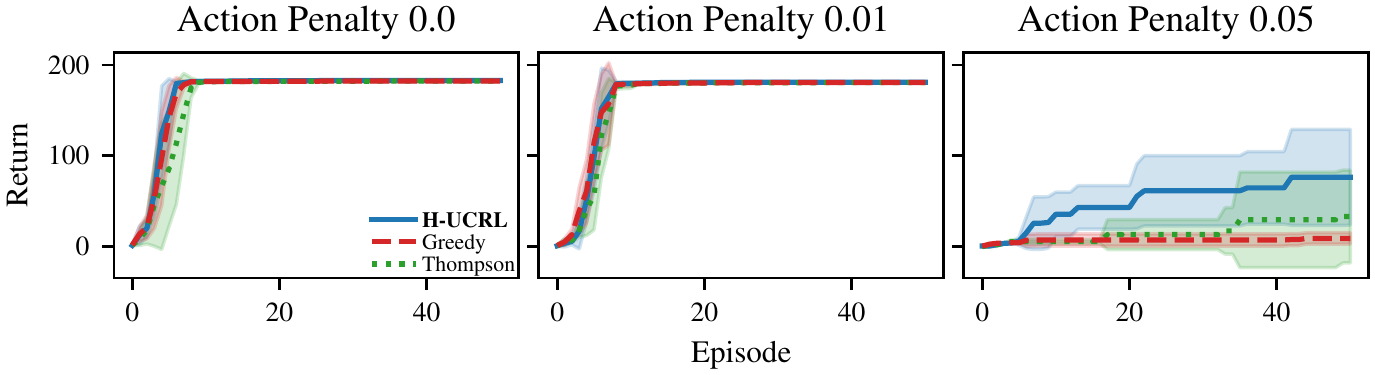}
    \end{subfigure}
  \caption{Top: Final episodic return in Cart-Pole environment. 
  Bottom: Learning curves in Cart-Pole environment. 
  For action penalty = 0.05, \alg outperforms other algorithms.
  For action penalty=0.2 already after the fifth episode it finds a swing-up maneuver. Thompson sampling finds it in only one run after the thirtyfifth episode.
  }
  \label{fig:mujoco_cartpole}
\end{figure}

\FloatBarrier
\subsubsection{Reacher}
The Reacher is a 7DOF robot with $\nstate=14$ and $\ninp=7$, with actions bounded in $[-20, 20]^\ninp$ and each episode lasts 150 time steps. The goal is sampled at location $(x, y, z) = (0, 0.25, 0) + \omega$, where $\omega$ is a zero-mean normal noise with $0.1$ standard deviation.
We transform the angles to a quaternion representation via $[\sin(\theta), \cos(\theta)]$. 
The goal is to move the end-effector towards the goal and the reward signal is given by $r = - \sum_{i =x,y,z}(\text{ee} - \text{goal})_i^2 - \rho \sum_{i=1}^7 \u_i^2$, where $\text{ee} - \text{goal}$ is the vector that measures the distance between the end-effector and the goal.
We show the results in \cref{fig:mujoco_reacher}. All algorithms perform equally for different action penalties.

\begin{figure}[ht]
    \centering
    \begin{subfigure}{\textwidth}
    \centering
        \includegraphics[width=\linewidth]{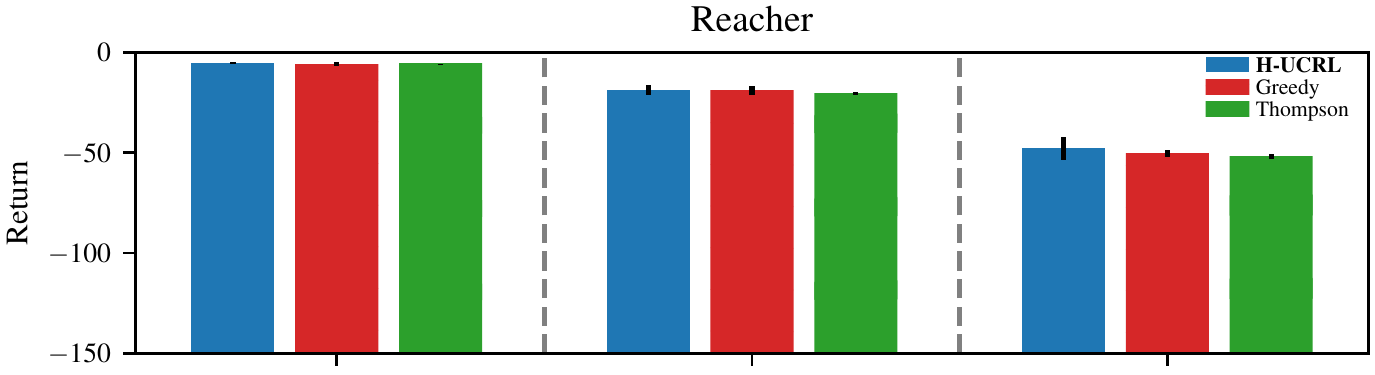}
    \end{subfigure} \\ 
    \begin{subfigure}{\textwidth}
    \centering
        \includegraphics[width=\linewidth]{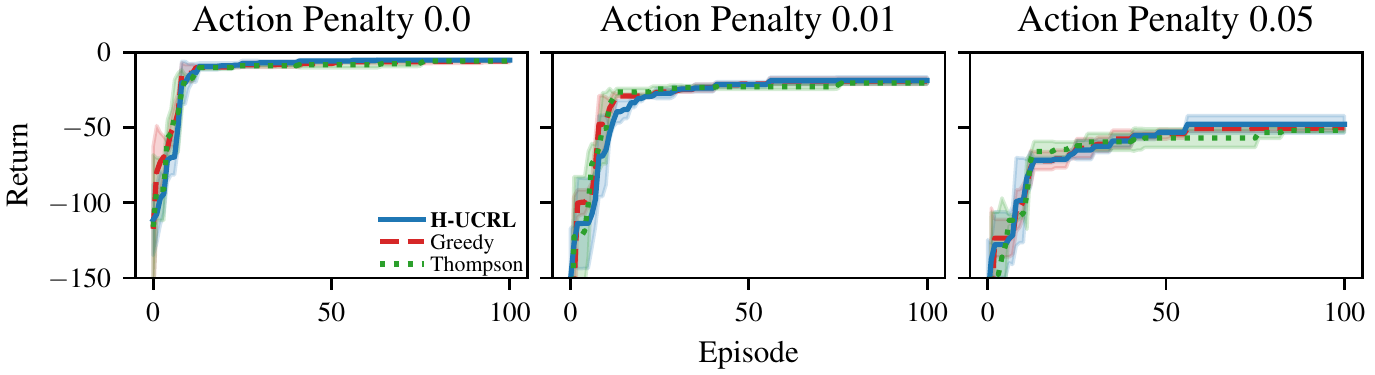}
    \end{subfigure}
  \caption{Top: Final episodic return in Reacher environment. 
  Bottom: Learning curves in Reacher environment. 
  Greedy, Thompson sampling, and \alg perform equally well.}
  \label{fig:mujoco_reacher}
\end{figure}

\FloatBarrier
\subsubsection{Pusher}
The Pusher is also a 7DOF robot with $\nstate=14$ and $\ninp=7$, with action bounds in $[-2, 2]^\ninp$ and each episode lasts 150 time steps. The object is free to move, introducing 3 more states to the environment. 
The robot starts with zero angles, an angular velocity sampled uniformly at random from $[-0.005, 0.005]$, the object is sampled from $(x,y)=(-0.25, 0.15) + \omega$, where $\omega$ is a zero-mean normal noise with $0.025$ standard deviation. The objective is to push the object towards the goal at $(x, y) = (0, 0)$. 
The reward signal is given by $r = - 0.5\sum_{i =x,y,z}(\text{ee} - \text{obj})_i^2 - 1.25 \sum_{i =x,y,z}(\text{obj} - \text{goal})_i^2 - \rho \sum_{i=1}^7  \u_i^2$, where $\text{ee} - \text{obj}$ is the distance between the end-effector and the object and $\text{obj} - \text{goal}$ is the distance between the object and the goal. We show the results in \cref{fig:mujoco_pusher}. All algorithms perform equally for different action penalties.

\begin{figure}[ht]
    \centering
    \begin{subfigure}{\textwidth}
    \centering
        \includegraphics[width=\linewidth]{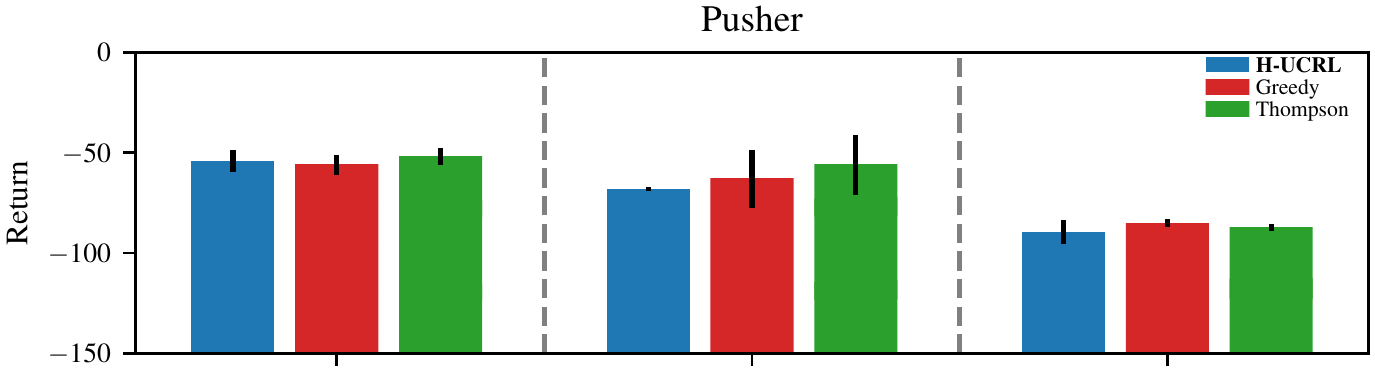}
    \end{subfigure} \\ 
    \begin{subfigure}{\textwidth}
    \centering
        \includegraphics[width=\linewidth]{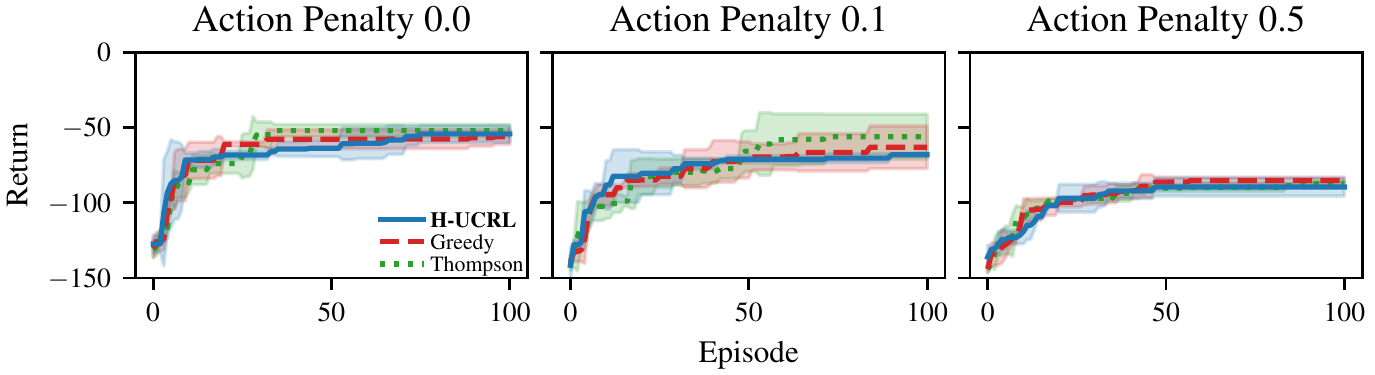}
    \end{subfigure}
  \caption{Top: Final episodic return in Pusher environment. 
  Bottom: Learning curves in Pusher environment. 
  Greedy, Thompson sampling, and \alg perform equally well.
  }
  \label{fig:mujoco_pusher}
\end{figure}

\FloatBarrier
\subsubsection{Sparse Reacher}
The sparse Reacher is the same 7DOF robot as the Reacher with $\nstate=14$ and $\ninp=7$, with actions bounded in $[-20, 20]^\ninp$ and each episode lasts 150 time steps. The sole difference arises in the reward function, which is given by $r = e^{-\sum_{i=x,y,z}  (\text{ee} - \text{goal})_i^2/0.45^2 } + \rho (e^{-\sum_{i=1}^7  \u_i^2} - 1)$. 
We show the results in \cref{fig:mujoco_reacher_sparse}. \alg performs better than Greedy and Thompson, particularly for larger action penalties.

\begin{figure}[ht]
    \centering
    \begin{subfigure}{\textwidth}
    \centering
        \includegraphics[width=\linewidth]{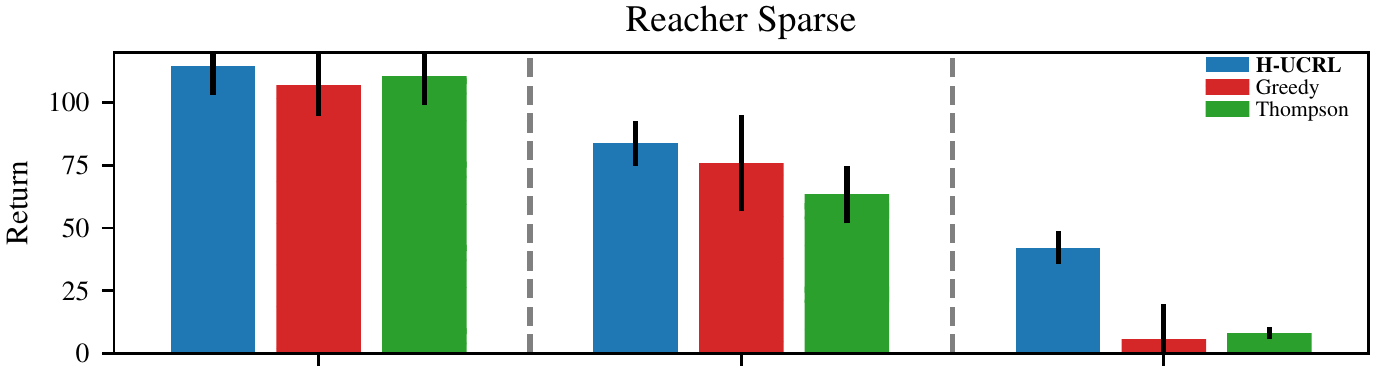}
    \end{subfigure} \\ 
    \begin{subfigure}{\textwidth}
    \centering
        \includegraphics[width=\linewidth]{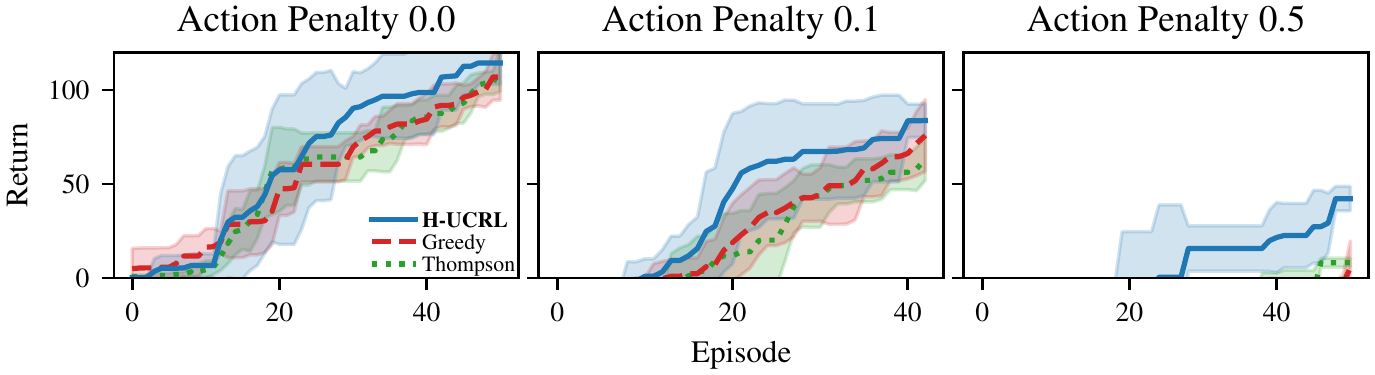}
    \end{subfigure}
  \caption{Top: Final episodic return in sparse Reacher environment. 
  Bottom: Learning curves in sparse Reacher environment. 
  \alg outperforms greedy and Thompson sampling, particularly when the action penalty increases.}
  \label{fig:mujoco_reacher_sparse}
\end{figure}

\FloatBarrier
\subsubsection{Half-Cheetah}
The Half-Cheetah is a mobile robot with $\nstate=17$ and $\ninp=6$, with actions bounded in $[-2, 2]^\ninp$ and each episode lasts 1000 time steps. The objective is to make the cheetah run as fast as possible forwards up to a maximum of $10 \text{m/s}$. The reward function is given by $r = \max(v, 10)$. 
We show the results in \cref{fig:mujoco_half_cheetah}. \alg performs finds quicker policies with higher returns and, when the action penalty is 1, it outperforms greedy and Thompson sampling considerably.

\begin{figure}[ht]
    \centering
    \begin{subfigure}{\textwidth}
    \centering
        \includegraphics[width=\linewidth]{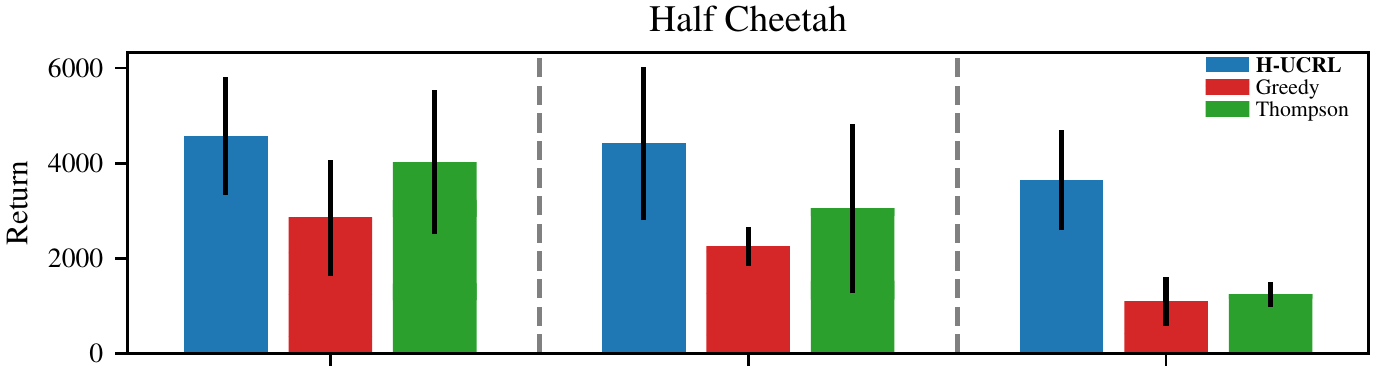}
    \end{subfigure} \\ 
    \begin{subfigure}{\textwidth}
    \centering
        \includegraphics[width=\linewidth]{figures/half_cheetah_evolution.pdf}
    \end{subfigure}
  \caption{Top: Final episodic return in Half-Cheetah environment. 
  Bottom: Learning curves in Half-Cheetah environment. 
  \alg outperforms greedy and Thompson sampling, particularly when the actoin penalty increases.}
  \label{fig:mujoco_half_cheetah}
\end{figure}

\FloatBarrier
\subsection{Visualization of Real and Simulated Trajectories for Inverted Pendulum}
In this section, we visualize the optimistic trajectory for the inverted pendulum problem. We plot the real and simulated trajectories using \alg in \cref{fig:optimistic_0,fig:optimistic_0.1,fig:optimistic_0.2} with increasing action penalties. 

\subsubsection{\alg Trajectories}
Already in the first episode, the \alg finds an optimistic trajectory to reach the goal (0, 0) position. 
With more episodes, it learns the dynamics and simulated and real trajectories match. 
As the action penalty increases, the action magnitude decreases and it takes longer for the algorithm to find a swing-up trajectory. 

\begin{figure}[H]
\centering
\begin{subfigure}{0.5\textwidth}
\centering
    \includegraphics[width=\linewidth]{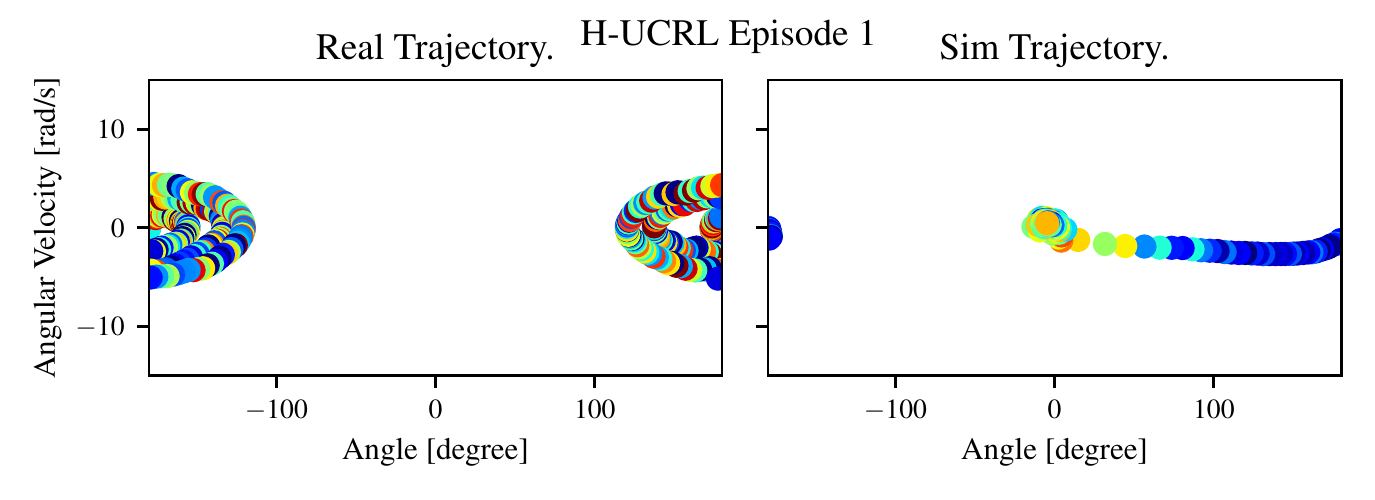}
\end{subfigure}%
\begin{subfigure}{0.5\textwidth}
\centering
    \includegraphics[width=\linewidth]{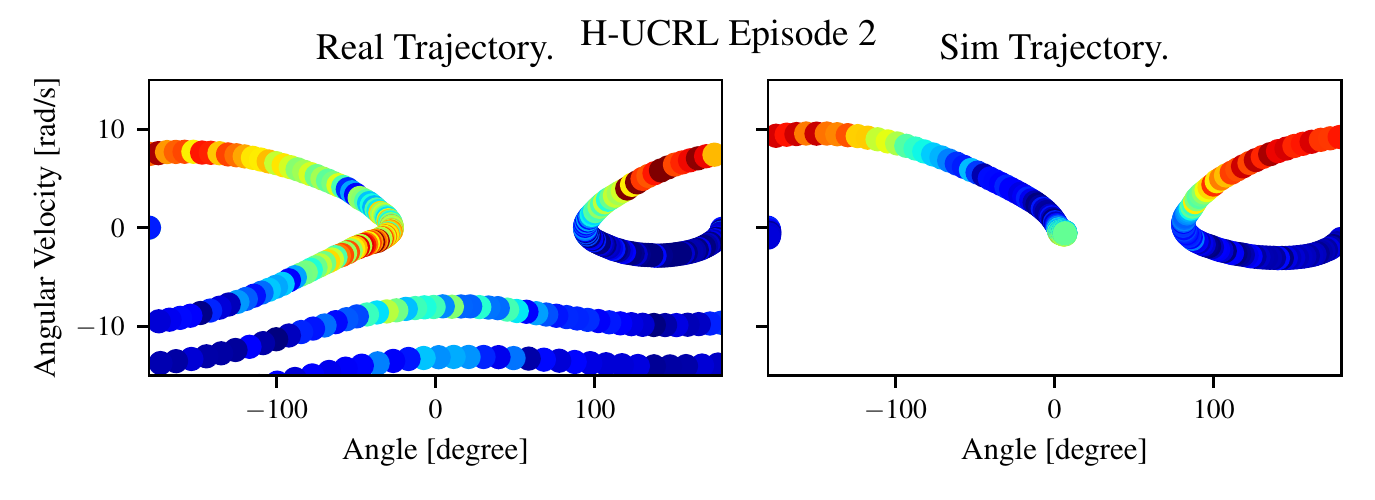}
\end{subfigure} \\ 
\begin{subfigure}{0.5\textwidth}
\centering
    \includegraphics[width=\linewidth]{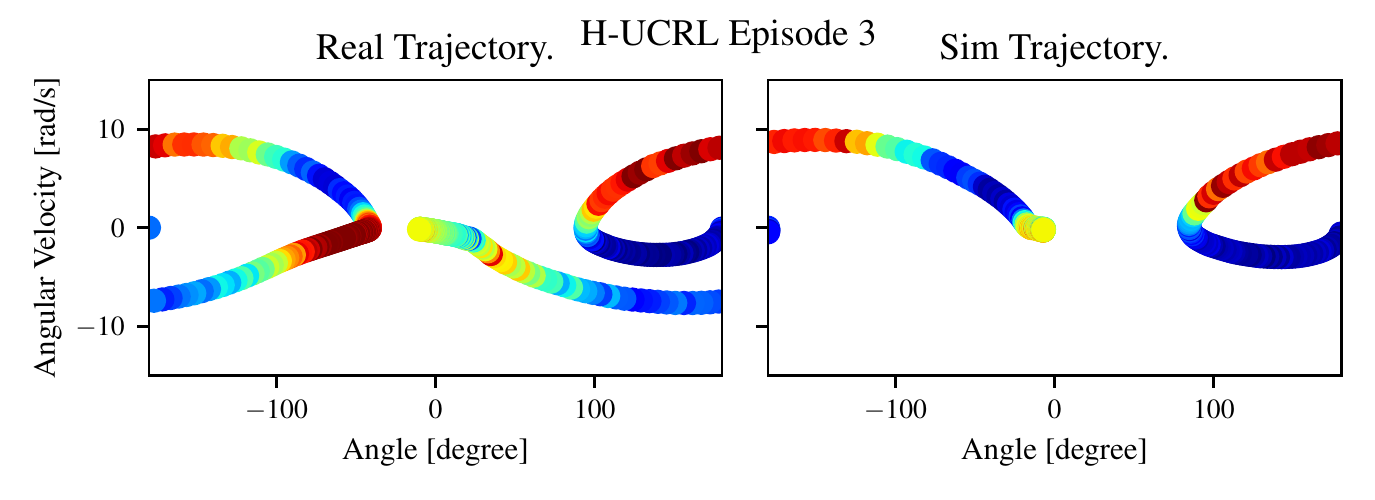}
\end{subfigure}%
\begin{subfigure}{0.5\textwidth}
\centering
    \includegraphics[width=\linewidth]{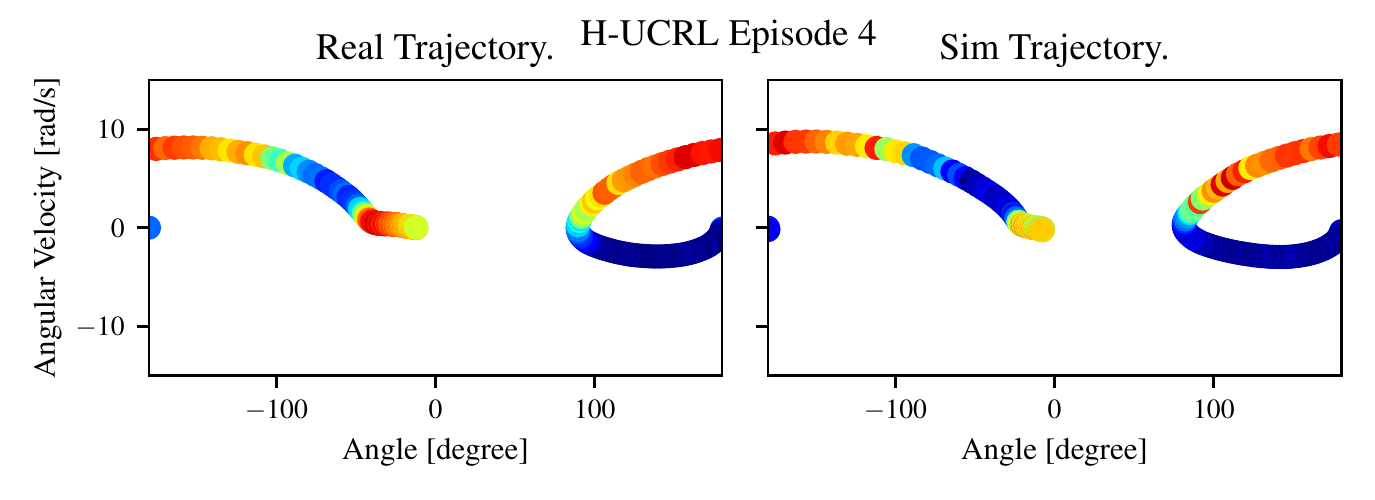}
\end{subfigure} \\
\begin{subfigure}{0.5\textwidth}
\centering
    \includegraphics[width=\linewidth]{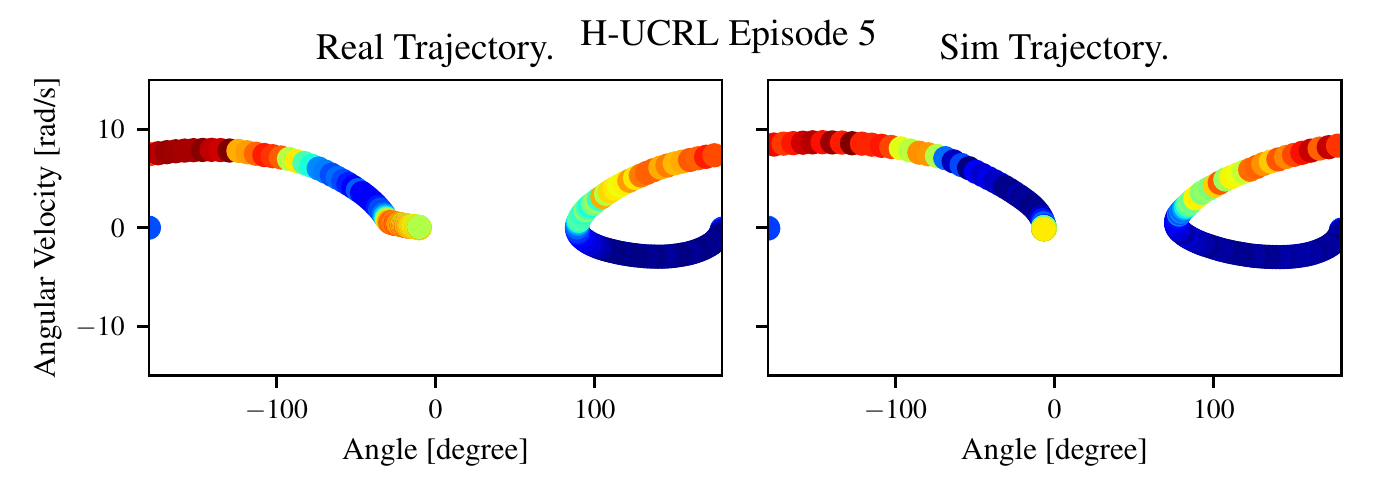}
\end{subfigure}%
\begin{subfigure}{0.5\textwidth}
\centering
    \includegraphics[width=\linewidth]{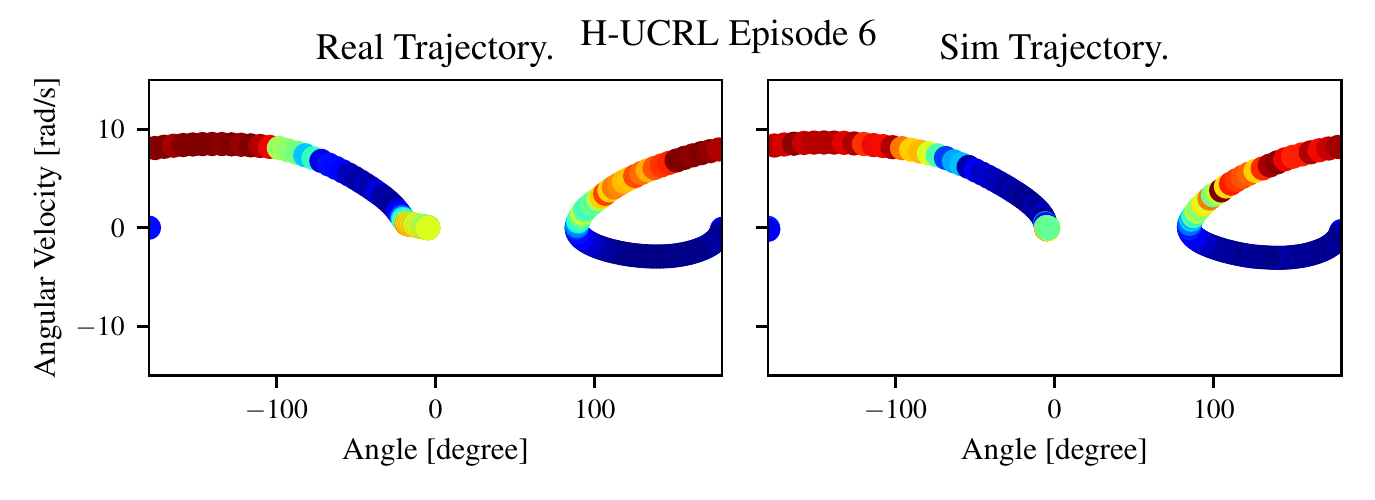}
\end{subfigure}
\caption{Real and simulated trajectories for first 6 episodes with \alg (0 action penalty).
We plot the trajectory in phase space, and use color coding to denote the action magnitude.}
\label{fig:optimistic_0}
\end{figure}

\begin{figure}[H]
\centering
\begin{subfigure}{0.5\textwidth}
\centering
    \includegraphics[width=\linewidth]{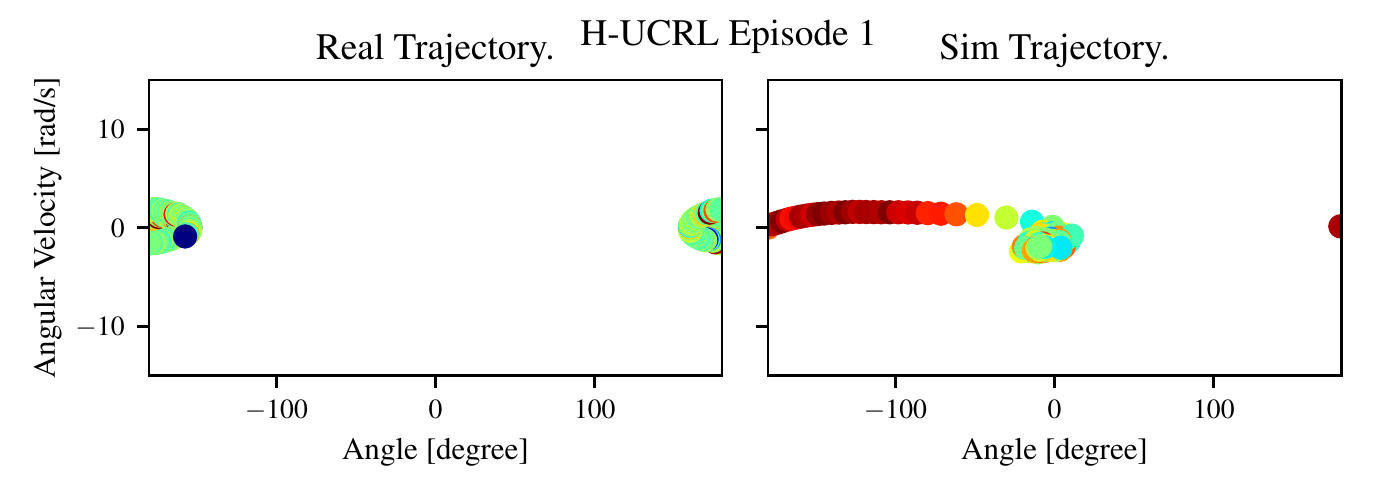}
\end{subfigure}%
\begin{subfigure}{0.5\textwidth}
\centering
    \includegraphics[width=\linewidth]{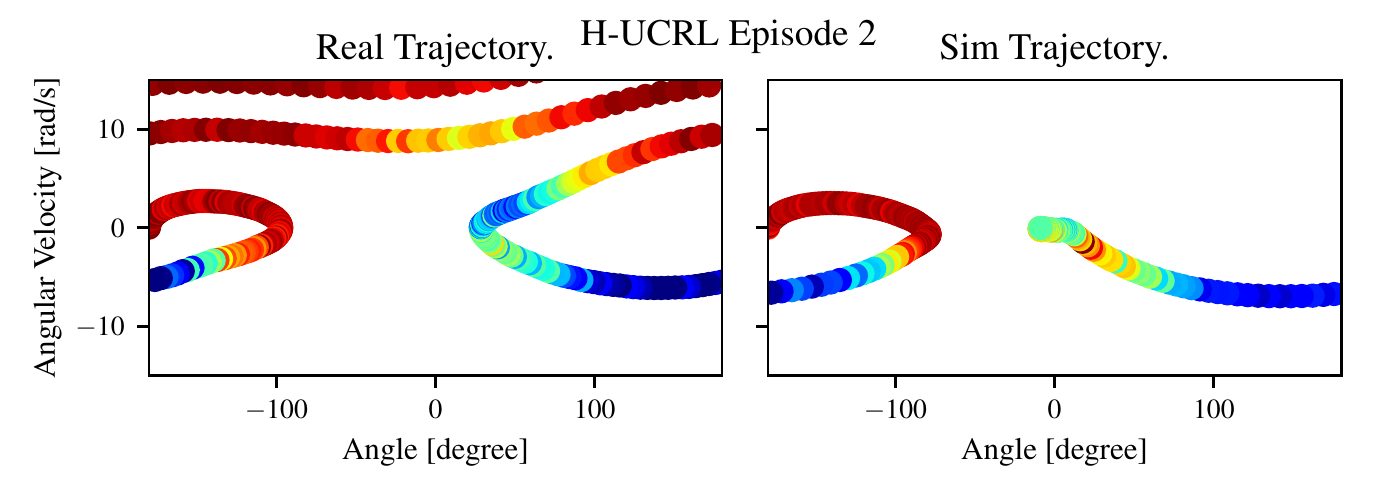}
\end{subfigure} \\ 
\begin{subfigure}{0.5\textwidth}
\centering
    \includegraphics[width=\linewidth]{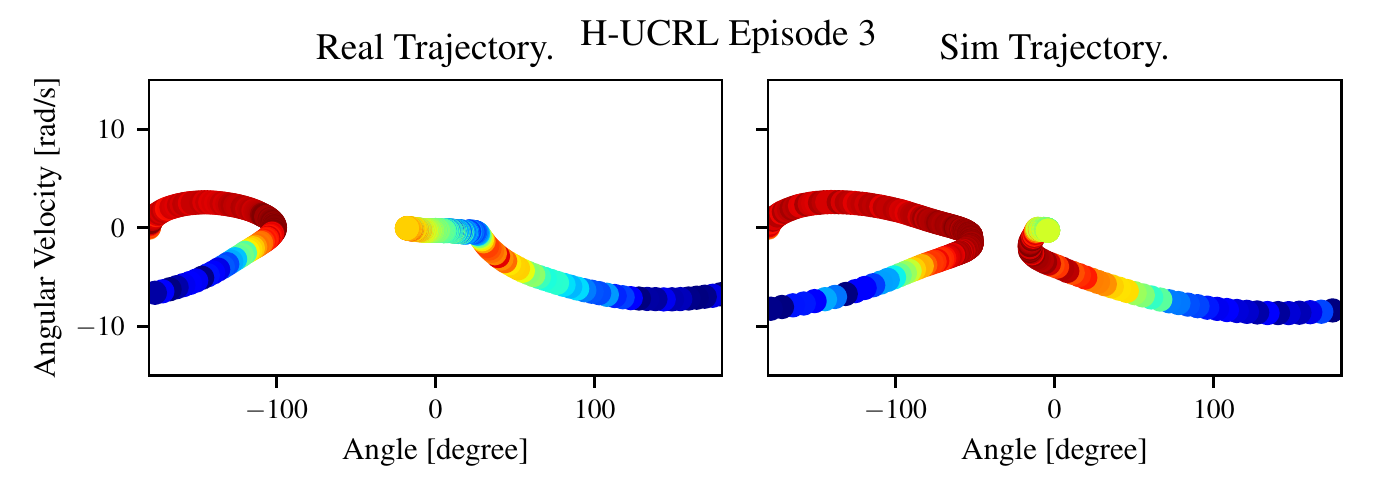}
\end{subfigure}%
\begin{subfigure}{0.5\textwidth}
\centering
    \includegraphics[width=\linewidth]{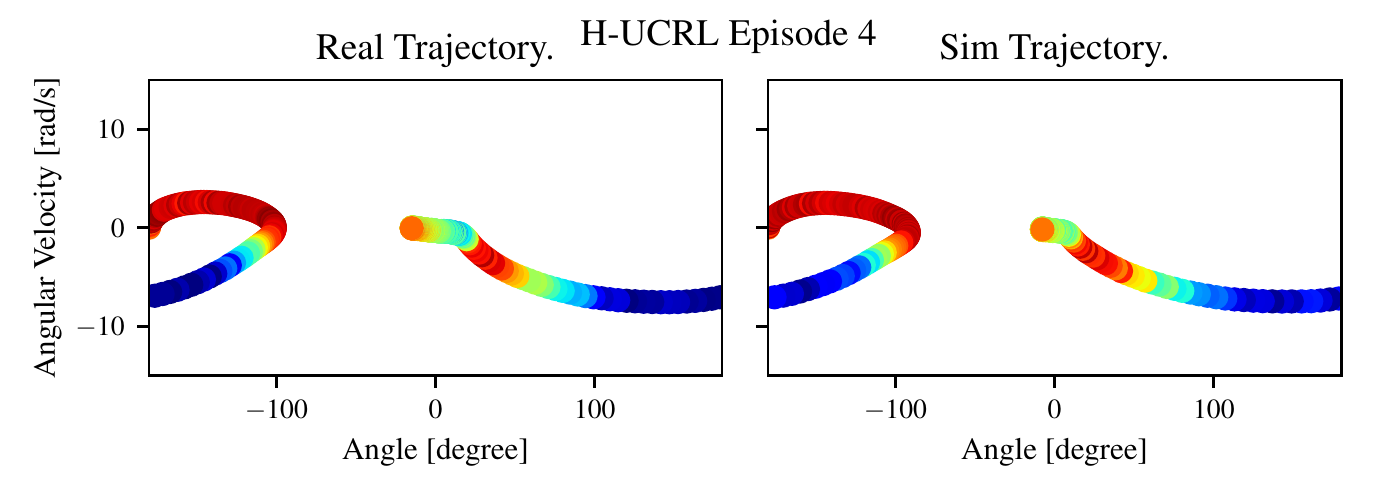}
\end{subfigure} \\
\begin{subfigure}{0.5\textwidth}
\centering
    \includegraphics[width=\linewidth]{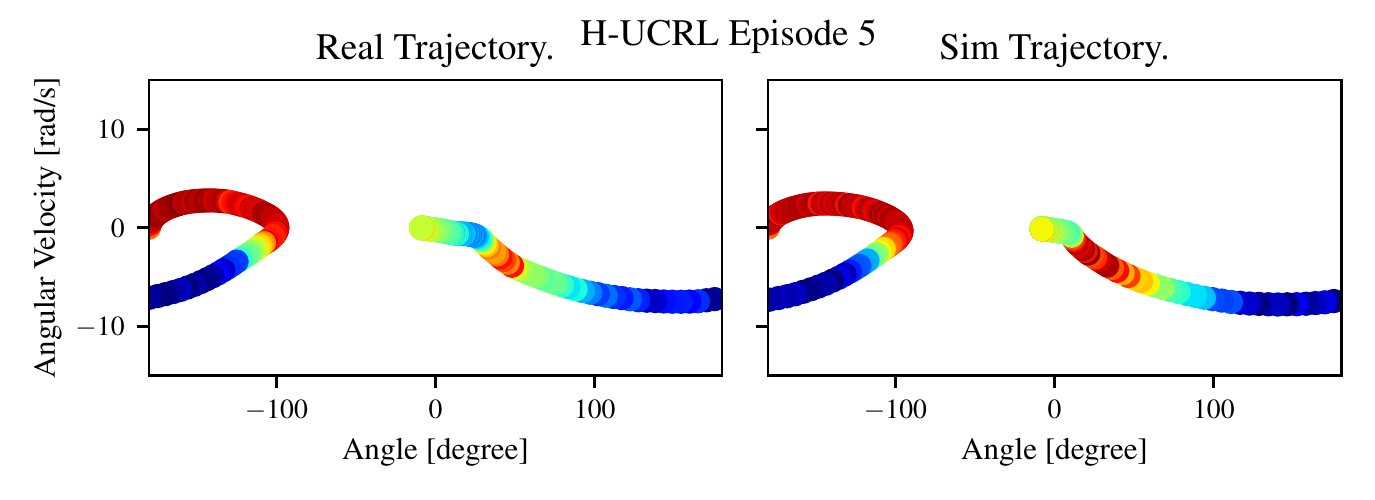}
\end{subfigure}%
\begin{subfigure}{0.5\textwidth}
\centering
    \includegraphics[width=\linewidth]{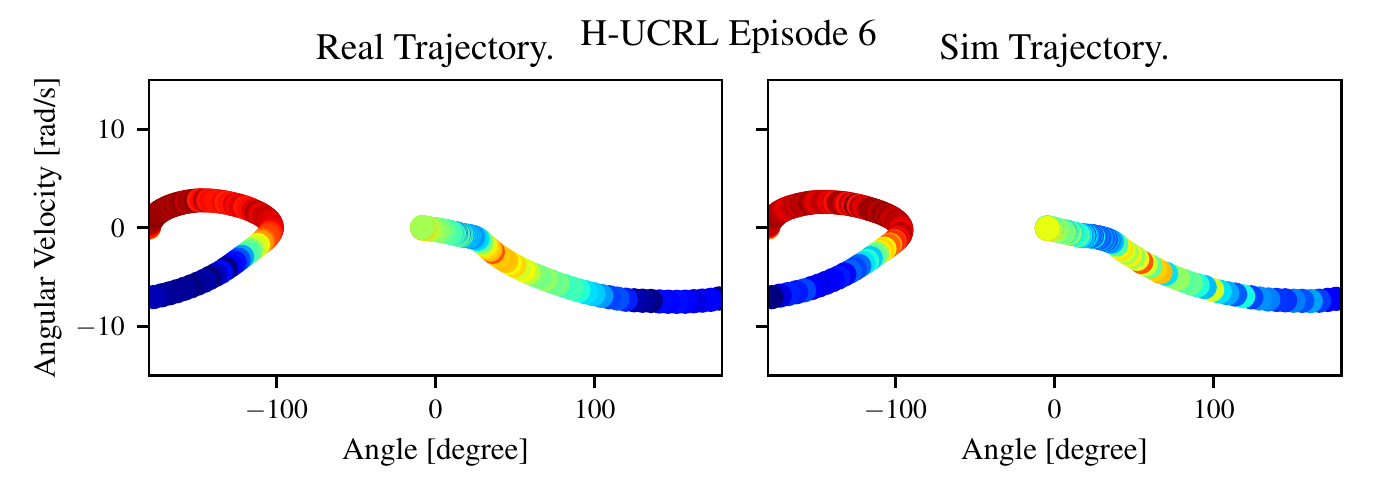}
\end{subfigure}
\caption{Real and simulated trajectories for first 6 episodes with \alg (0.1 action penalty).
We plot the trajectory in phase space, and use color coding to denote the action magnitude.}
\label{fig:optimistic_0.1}
\end{figure}

\begin{figure}[H]
\centering
\begin{subfigure}{0.5\textwidth}
\centering
    \includegraphics[width=\linewidth]{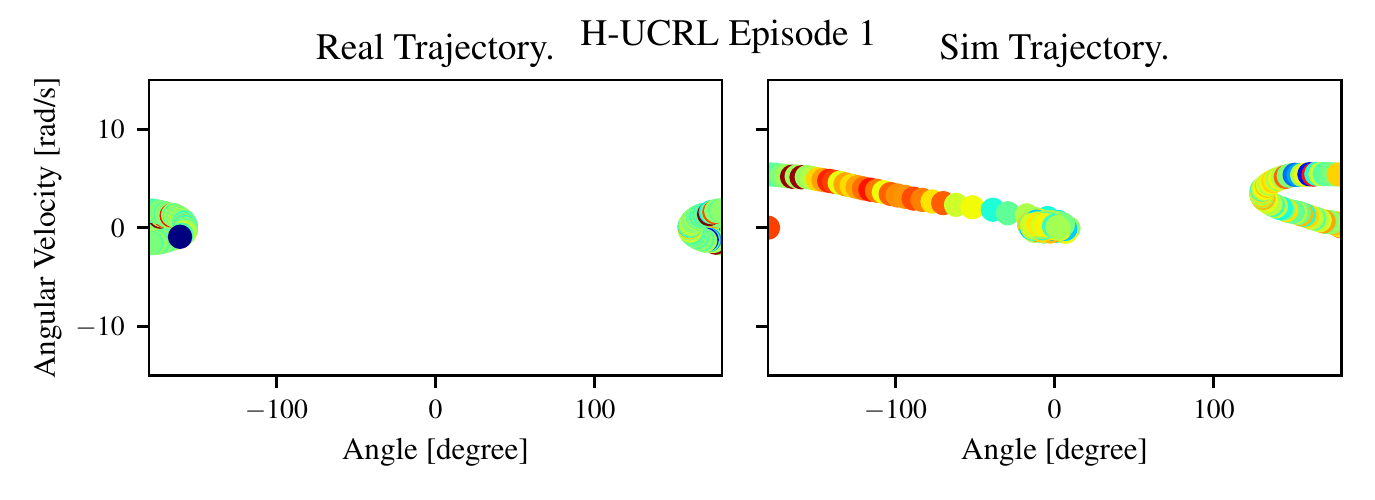}
\end{subfigure}%
\begin{subfigure}{0.5\textwidth}
\centering
    \includegraphics[width=\linewidth]{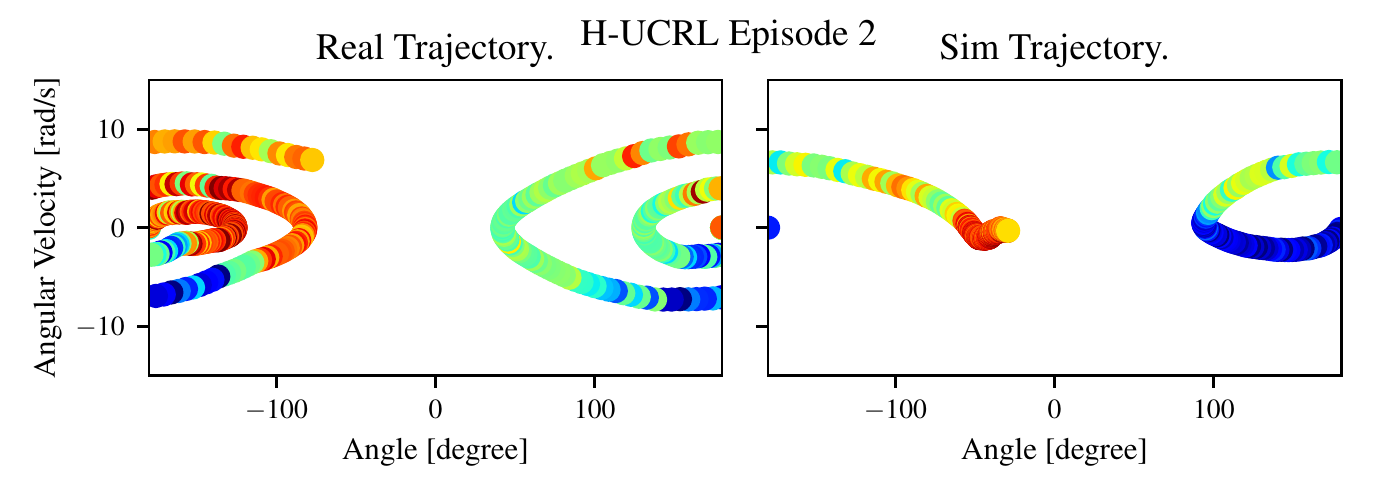}
\end{subfigure} \\ 
\begin{subfigure}{0.5\textwidth}
\centering
    \includegraphics[width=\linewidth]{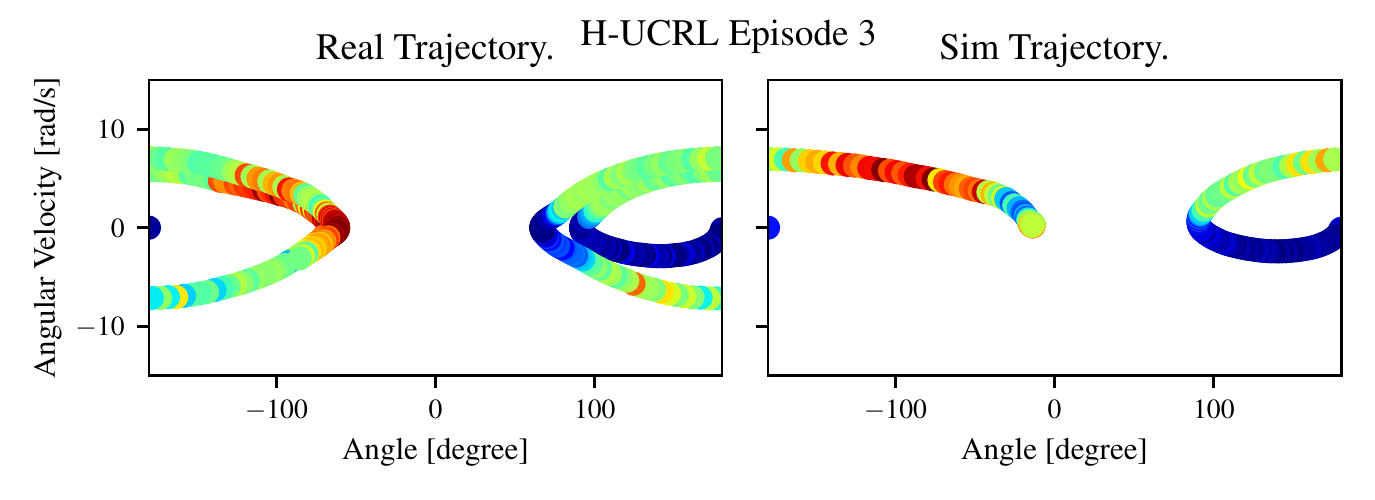}
\end{subfigure}%
\begin{subfigure}{0.5\textwidth}
\centering
    \includegraphics[width=\linewidth]{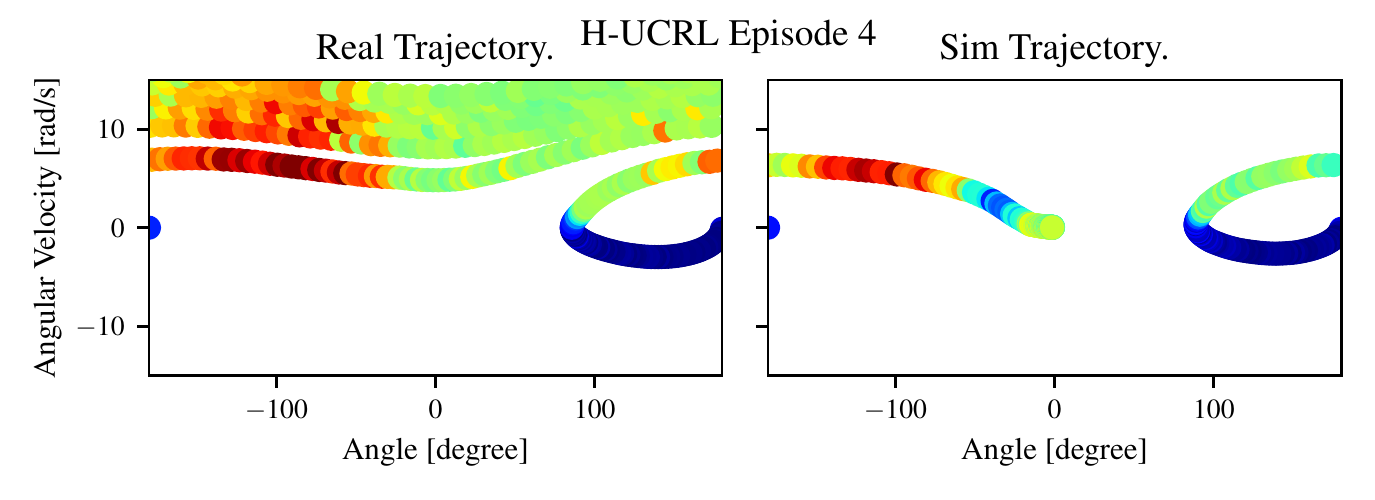}
\end{subfigure} \\
\begin{subfigure}{0.5\textwidth}
\centering
    \includegraphics[width=\linewidth]{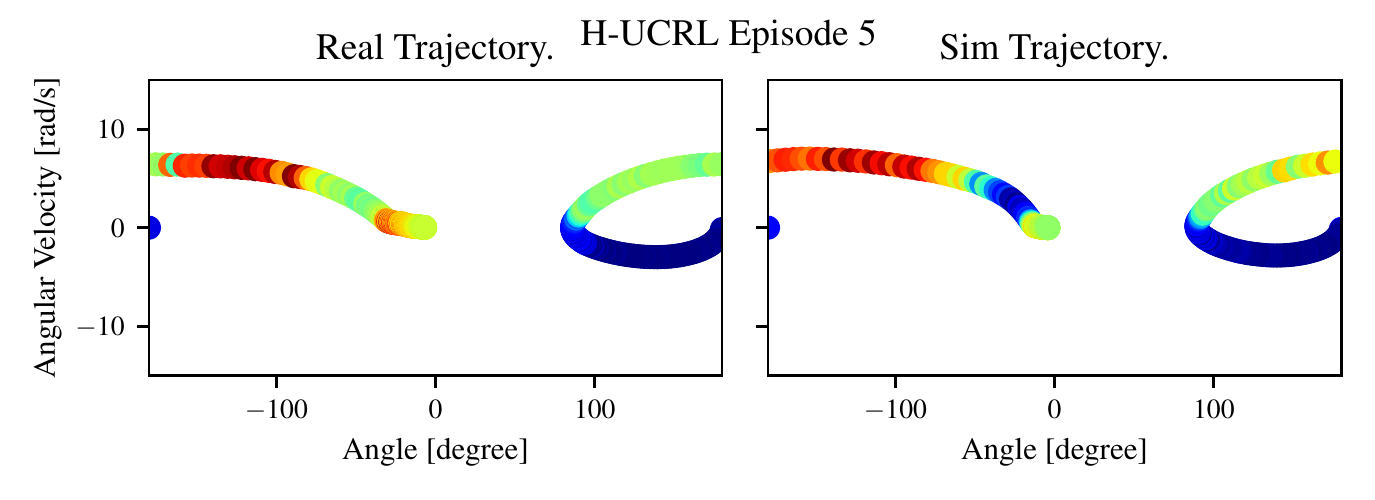}
\end{subfigure}%
\begin{subfigure}{0.5\textwidth}
\centering
    \includegraphics[width=\linewidth]{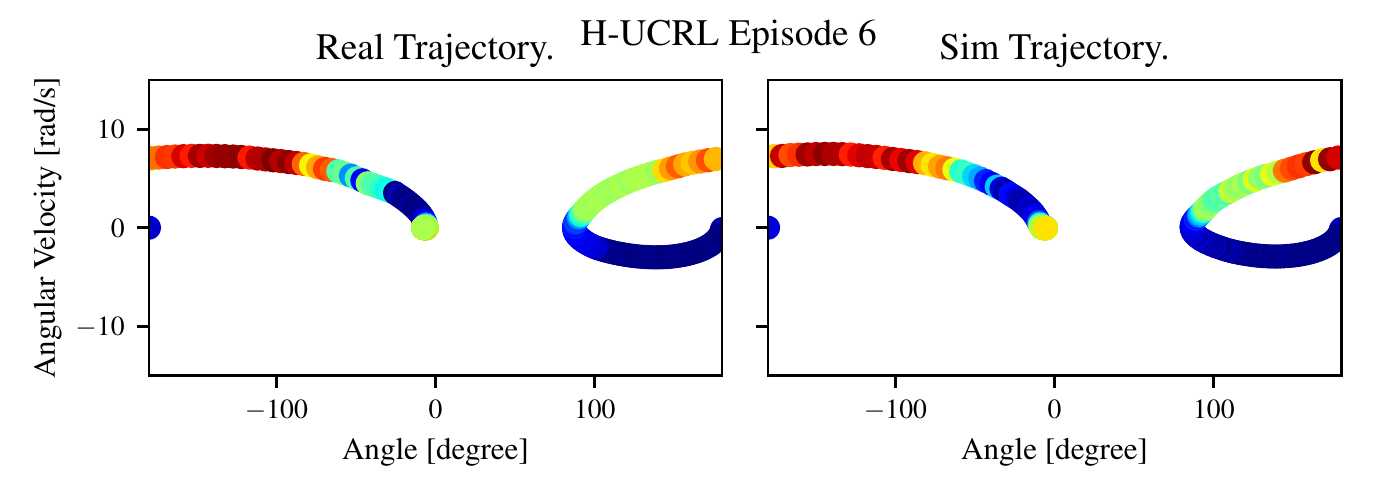}
\end{subfigure}
\caption{Real and simulated trajectories for first 6 episodes with \alg (0.2 action penalty).
We plot the trajectory in phase space, and use color coding to denote the action magnitude.}
\label{fig:optimistic_0.2}
\end{figure}

\subsection{Further Experiments on Thompson Sampling} \label{ap:ThompsonSampling}
We found surprising that Thompson Sampling under-performs compared to optimistic exploration. 
To understand better why this happens, we perform different experiments in this section.

\subsubsection{Can the sampled models solve the task?}
One possibility is that, when doing posterior sampling, the agent learns a model for the sampled model, which might be biased.
If this was the case, we would expect to see the {\em simulated} returns, i.e., the returns of the optimal policy in the sampled system $\tilde{f}_i$ large.

In \cref{fig:simulated_returns} we show the returns of the last simulated trajectory starting from the bottom position of each episode. This figure indicates that there is no model bias, i.e., the simulated returns for Thompson sampling are also low. 
We conclude that it is not over-fitting to the sampled model, but rather the algorithm cannot solve the task with the sampled model.

\FloatBarrier

\begin{figure}[ht]
  \includegraphics[width=\linewidth]{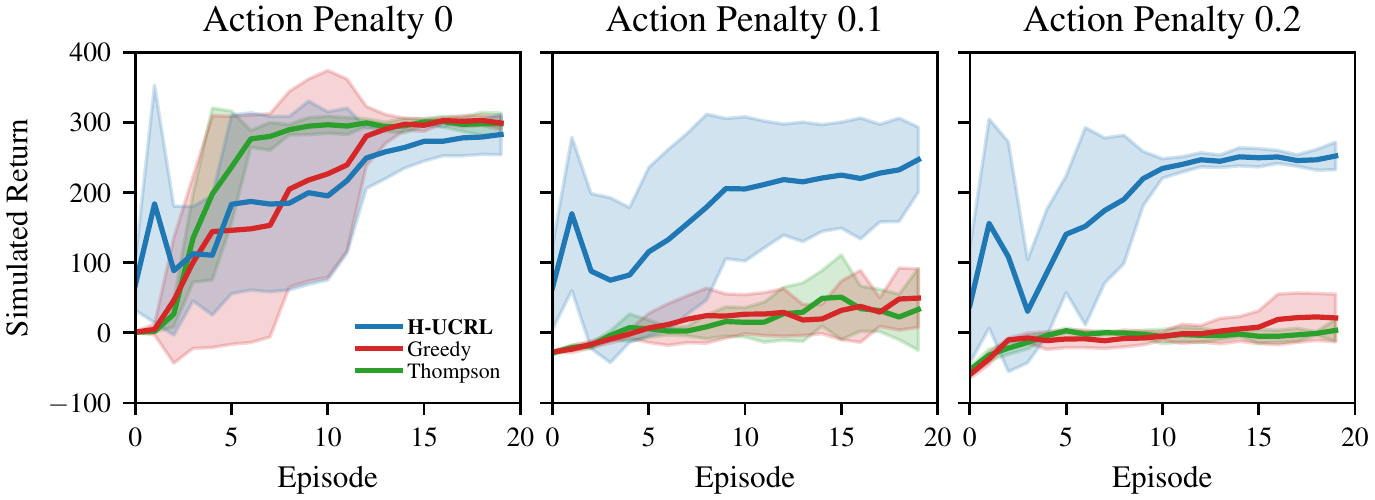}
  \caption{Total return from last simulated trajectory with the same initial state as the environment initial state. \alg has higher simulated returns than Greedy and Thompson as the action penalty increases.}
  \label{fig:simulated_returns}
\end{figure}

\FloatBarrier

\subsubsection{Is it variance starvation?}
Another possibility is Thompson Sampling suffers variance starvation, i.e., all ensemble members' predictions are identical. 
Variance starvation means that the approximate posterior variance is smaller than the true posterior variance. 
When this happens, (approximate) Thompson Sampling fails because of lack of exploration \citep{Wang2018BatchedBO}. 
In contrast to UCRL-stye algorithms where the optimism is implemented {\em deterministically}, Thompson sampling implements optimism {\em stochastically}. Thus, it is crucial that the variance is not underestimated.

If there was variance starvation, we would expect to see the epistemic variance along simulated trajectories  shrink. 
In \cref{fig:simulated_scale} we show the average simulated uncertainty during training, considered as the predictive variance of the ensemble.
To summarize the predictive uncertainty into a scalar, we consider the trace of the Cholesky factorization of the covariance matrix. 
From the figure, we see that \alg starts with the same predictive uncertainty as greedy and Thompson sampling. 
Furthermore, the variance of Thompson sampling does not shrink. 
We conclude that there is no variance starvation in the one-step ahead predictions.

\begin{figure}[ht]
  \includegraphics[width=\linewidth]{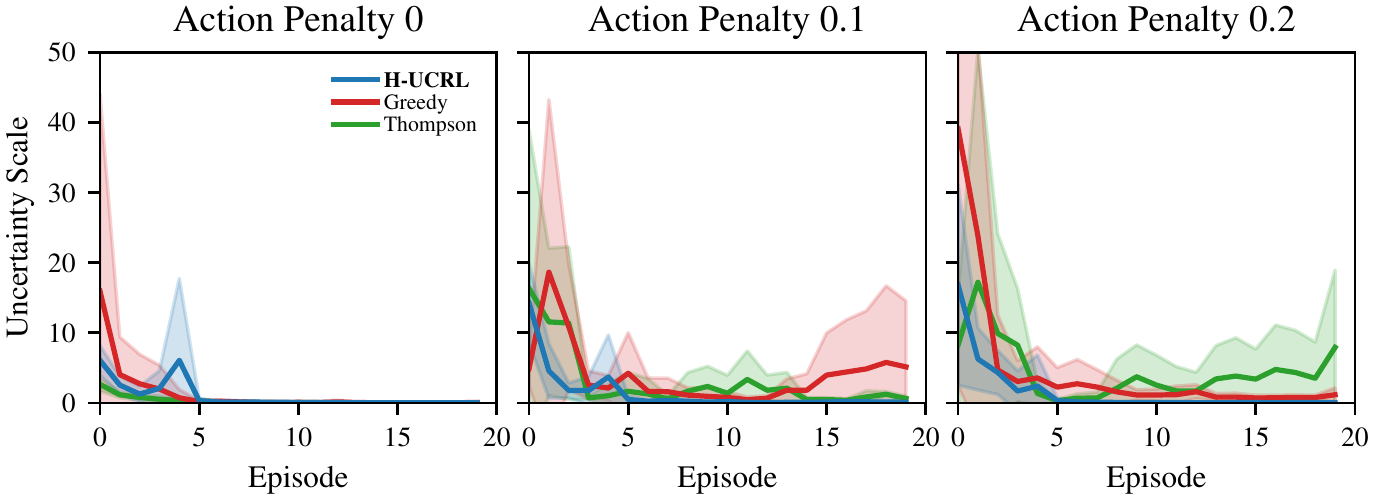}
  \caption{Epistemic model uncertainty along simulated trajectories. Thompson and Greedy have the same or more uncertainty than \alg.}
  \label{fig:simulated_scale}
\end{figure}

\FloatBarrier

\subsubsection{Is the number of ensemble members enough?}
In order to verify this hypothesis, we ran the same experiments with 5, 10, 20, 50, and 100 ensemble members. 
All models swing-up the pendulum with 0 action penalty. With 0.1 action penalty, the 20, 50, and 100 ensembles find a swing up in only one run out of five. 
With 0.2 action penalty, no model finds a swing-up strategy. 
This suggests that having larger ensembles could help, but it is not convincing. 
Furthermore, the model training computational complexity increases linearly with the number of ensemble members, which limits the practicality of larger ensembles. 

\begin{figure}[ht]
  \includegraphics[width=\linewidth]{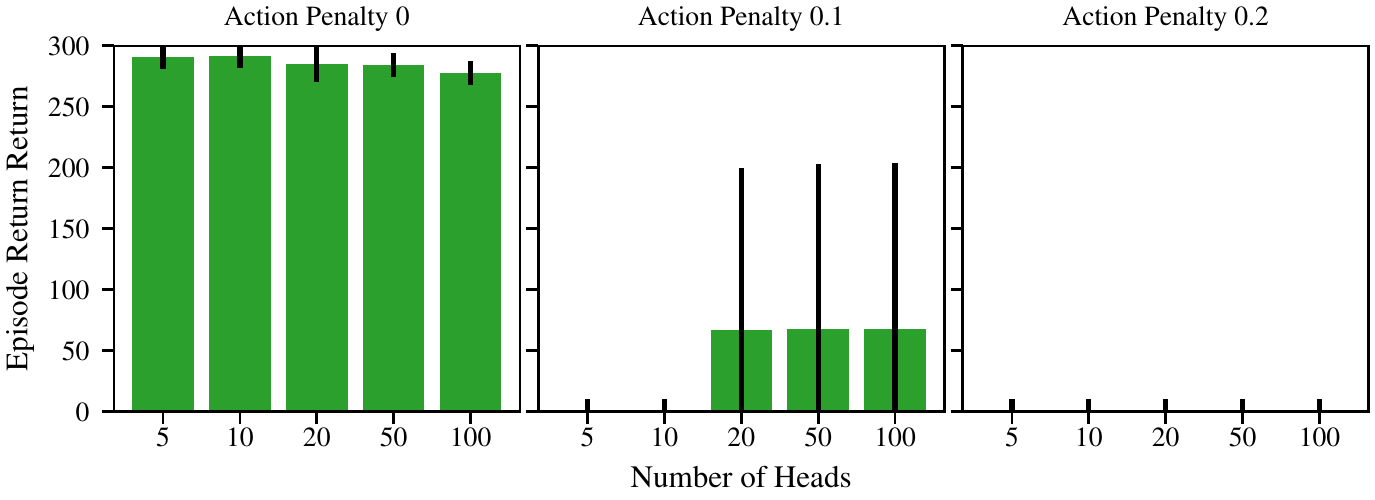}
  \caption{Episodic returns using Thompson Sampling for different number of ensemble members}
  \label{fig:num_heads}
\end{figure}
\FloatBarrier

\subsubsection{Is it the bootstrapping procedure during Training?}
Yet another possibility is that the bootstrap procedure yields inconsistent models for Thompson sampling.
To simulate bootstrapping, for each transition and ensemble member, we sample a mask from a Poisson distribution \citep{Osband2016Deep}. 
Then, we train using the loss of each transition multiplied by this mask. 
This yields correct one-step ahead confidence intervals. 
However, the model is used for multi-step ahead predictions. 
To test if this is the reason of the failure we repeat the experiment {\em without} bootstrapping the transitions. 
The only source of discrepancy between the models comes from the initialization of the model. 
This is how \citet{Chua2018Deep} train their probabilistic models and the models learn from {\em consistent} trajectories.

In \cref{fig:no_bootstrapping} we show the results when training without bootstrapping. 
The learning curves closely follow those with bootstrapping in \cref{fig:inverted_evolution}.
We conclude that the bootstrapping procedure is likely not the cause of the failure of Thompson Sampling. 

\begin{figure}[ht]
  \includegraphics[width=\linewidth]{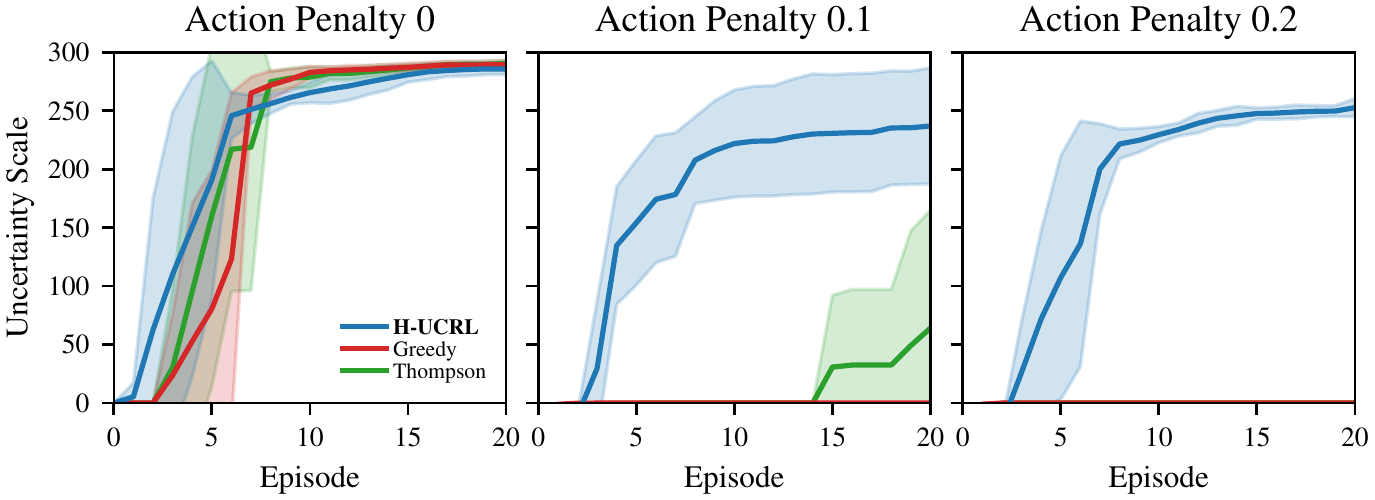}
  \caption{Episodic Returns in inverted pendulum {\em without} bootstrapping data while learning the model.}
  \label{fig:no_bootstrapping}
\end{figure}
\FloatBarrier

\subsubsection{Are probabilistic ensembles not a good approximation to the posterior in Thompson sampling?}  \label{ap:TS_approxGP}
We next investigate the possibility that Probabilistic Ensembles are not a good approximation for \modelposterior. To this end, we consider the Random Fourier Features (RFF) proposed by \citet{Rahimi2008RFF} for GP Models. To sample a posterior, we sample a set of random features and use the same features throughout the episodes as required by theoretical results for Thompson sampling and suggested by \citet{Hewing2019Simulation} to simulate trajectories. 
RFFs, however, are known to suffer from variance starvation. 
We also consider Quadrature Fourier Features (QFF) proposed by \citet{Mutny2018Efficient}.
QFFs have provable no-regret guarantees in the Bandit setting as well as a uniform approximation bound. 

In \cref{fig:rff}, we show the results for both RFF (1296 features), and QFFs (625 features). 
Neither QFFs nor RFFs find a swing-up maneuver for action penalties larger than zero, whereas optimistic exploration with both QFFs and RFFs do. 
For 0 action penalty, optimistic exploration with RFFs underperforms compared to greedy exploitation and Thompson sampling. This might be due to variance starvation of RFFs because we do not see the same effect on QFFs. 
We conclude that PE are as good as other approximate posterior methods such as random feature models. 

\begin{figure}[ht]
  \includegraphics[width=\linewidth]{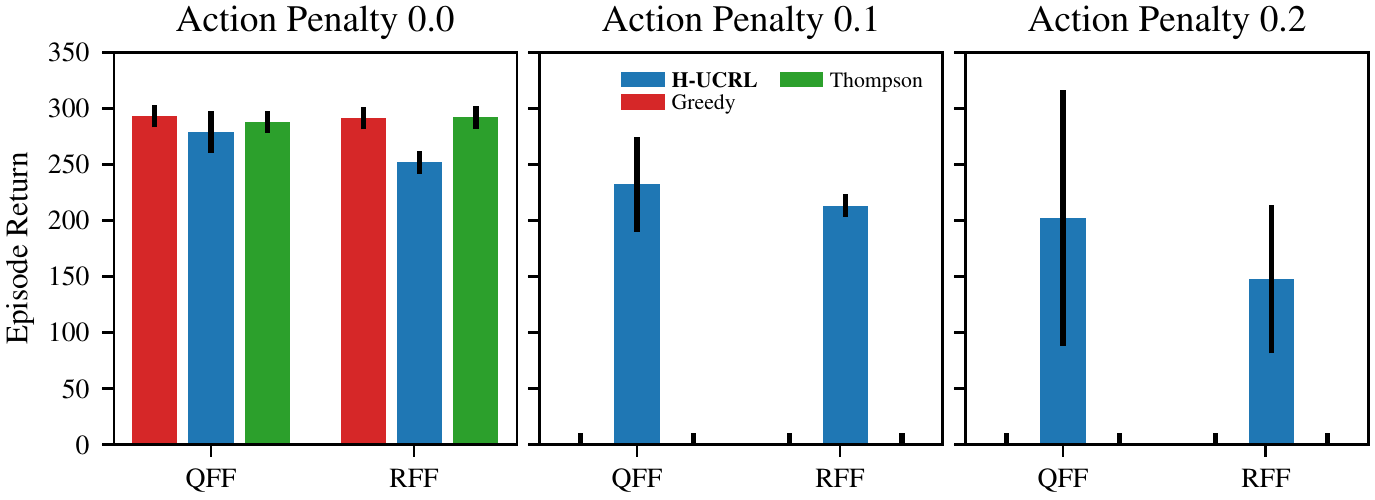}
  \caption{Episodic Returns in inverted pendulum using Random Fourier Features (RFF) and Quadrature Fourier Features (QFF).}
  \label{fig:rff}
\end{figure}
\FloatBarrier

\subsubsection{Is it the optimization procedure?}
The final and perhaps most enlightening experiment is the following. 
We run optimistic exploration with five ensemble heads and save snapshots of the models after the first, fifth and tenth episode. 
Then, we optimize a different policy for each of the models separately. In \cref{fig:best_head} we compare the simulated returns using optimistic exploration on the ensemble at each episode against the {\em maximum} return obtained by the best head. 

After the first episode, the simulated returns using optimistic exploration always find an optimistic swing-up trajectory, whereas the best-head always returns zero. This indicates that, when the uncertainty is large, optimistic exploration finds a better policy than approximate Thompson sampling. 
Without action penalty, the best head return quickly catches up to the simulated ones with optimistic exploration. 
For an action penalty of 0.1, after five episodes the best head is not able to find a swing-up trajectory. 
However, after ten episodes it does. 
This shows that the optimization algorithm is able to find the policy that swings-up a single model. 
However, when Thompson sampling is used to collect data, the optimization does not find such a policy. 
This indicates that the models learned using \alg better reduce the uncertainty around the high-reward region and each member of the ensemble has {\em sharper} predictions.
For 0.2 action penalty, the best head never finds a swing-up policy in ten episodes.

\begin{figure}[ht]
  \includegraphics[width=\linewidth]{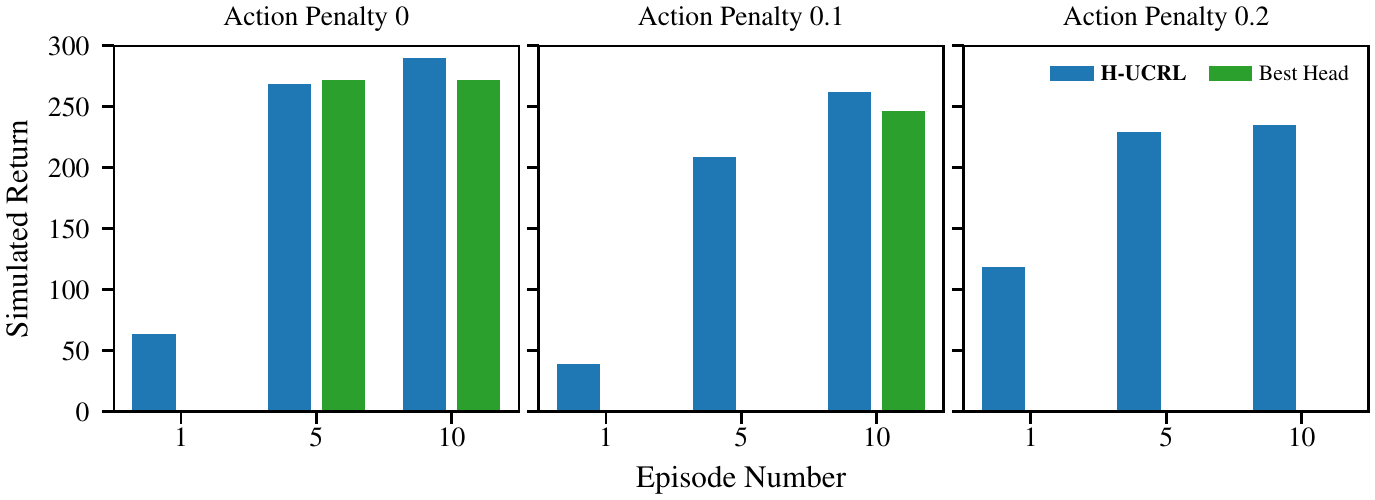}
  \caption{Simulated Returns using \alg vs. Maximum simulated return over all ensemble members using the same model as \alg.}
  \label{fig:best_head}
\end{figure}
\FloatBarrier

\subsubsection{Conclusions}
We believe that the poor performance of Thompson sampling relative to \alg suggests that a probabilistic ensemble with five members is sufficient to construct reasonable confidence intervals (hence \alg finds good policies), but does not comprise a rich enough posterior distribution for Thompson Sampling.
We suspect that this effect is inherent to the multi-step RL setting.
It seems to be the case that an approximate posterior model whose variance is rich enough for one-step predictions does not sufficiently represent/cover the diversity of plausible trajectories in the multi-step setting.
Thompson sampling implements optimism {\em stochastically}: for it to work, we must be able to sample a model that solves the task using multi-step predictions. 
Designing {\em tractable} approximate posteriors with {\em sufficient} variance for multi-step prediction is still a challenging problem. 
For instance, an ensemble model with $B$ members that has sufficient variance for 1-step predictions, requires $B^N$ members for N-step predictions, this quickly becomes intractable.

Compared to Thompson sampling, UCRL algorithms in general, and \alg in particular, only require one-step ahead calibrated predictive uncertainties in order to successfully implement optimism. 
This is because the optimism is implemented {\em deterministically} and it can be used recursively in a computationally efficient way. 
Furthermore, we know how to train (and calibrate) models to capture the uncertainty. 
This hints that optimism might be better suited than approximate Thompson sampling in model-based reinforcement learning. 

\newpage




%% file: appendix/dyna_mpc.tex

\section{Solving the Augmented Greedy Exploitation Program}
\label{ap:dyna-mpc}

In this section, we discuss how to practically solve the greedy exploitation problem with the augmented hallucination variables. 
In \Cref{sec:practical} we showed that the optimization program is a stochastic optimal-control problem for the hallucinated model $\tilde{f}$. There are two common ways to solve this stochastic optimal-control problem: off-line policy search and on-line planning.
In \Cref{ssec:ops}, we describe offline policy search algorithms, in \Cref{ssec:op} we present online planning algorithms, and in \Cref{ssec:combining} we show how to combine these algorithms. 

\subsection{Offline Policy Search} \label{ssec:ops}

Off-line policy search usually parameterize a policy $\pi(\cdot; \theta)$ using a function approximation method (e.g., neural networks), and then uses the policy $\pi(\cdot; \theta)$ to interact with the environment. 
We parameterize both the true and hallucinated policies with neural network $\pi(\cdot; \theta), \eta(\cdot; \theta)$. 
Next, we describe how to augment common policy-search algorithms with hallucinated policies. Any of such algorithms can be used as the \texttt{PolicySearch} method in \Cref{alg:hucrl}.

\textbf{Imagined Data Augmentation} consists of using the model to simulate data and then use these data to learn a policy using a model-free RL method. 
For example, the celebrated Dyna algorithm from \citet{Sutton1990Integrated}, DAD from \citet{venkatraman2016improved}, IB from \citet{kalweit2017uncertainty}, and I2A \citet{racaniere2017imagination} generate data by sampling from expected models.
In \Cref{alg:hall-data-augm}, we show HDA (for Hallucinated Data Augmentation). In HDA, we generate data using the optimistic dynamics in ~\eqref{eq:expected_performance_exploration} and then call any model-free RL algorithm such as SAC \citep{Haarnoja2018SAC}, MPO \citep{Abdolmaleki2018MPPO}, TD3 \citep{Fujimoto2018TD3}, TRPO \citep{Schulman2015TRPO}, or PPO \citep{Schulman2017Proximal}. 
Furthermore, the initial state distribution where hallucinated trajectories start from might be any exploratory distribution. 
This greatly simplifies the task of the \texttt{ModelFree} algorithm.
Usually these strategies combine true with hallucinated data buffers. To match dimensions between these, we augment the action space of the true data buffer with samples of a standard normal. 
This strategy usually suffers from model-bias as model errors compound throughout a trajectory, yielding highly biased estimates that hinder the policy optimization \citep{VanHasselt2019MBRLbias}.

\begin{algorithm}[htpb]
  \caption{Hallucinated Data Augmentation}
  \label{alg:hall-data-augm}
  \begin{algorithmic}[1]
    \INPUT{Calibrated dynamical model $(\bmu, \bSigma)$,
          reward function $r(\x, \u)$,
          horizon $\Ti$, 
          initial state distribution $d(\x_0)$, 
          number of iterations $\Ti_{\text{iter}}$,
          number of data points $\Ti_{\text{data}}$,
          initial parameters $\theta_{\ni-1}, \vartheta_{\ni-1}$, 
          model-free algorithm \texttt{ModelFree}.
          }
    \State Initialize $\theta_{\ni,0} \gets \theta_{\ni-1}, \vartheta_{\ni,0} \gets \vartheta_{\ni-1}$
    \For{$i = 1, \ldots,  \Ti_{\text{iter}}$}
    \State{\color{blue} /* Simulate Data */}
    \State Initialize hallucinated data buffer $\mathcal{D}_{\mathrm{h}} = \{\emptyset\}$. 
    \For{$i = 1, \ldots, \Ti_{\text{data}}$} 
    \State Start from initial state distribution $\hat{\x}_0 \sim d(\x_0)$.
    \For{$\ti = 0, \dots, \Ti-1$}
        \State Compute action $\hat{\u}_\ti \sim  \pi(\hat{\x}_\ti;\theta_{\ni,i})$, $\hat{\u}'_\ti \sim  \eta(\hat{\x}_\ti;\theta_{\ni,i})$
        \State Sample next state $\hat{\x}_{\ti + 1} \sim \bmu_\ni(\hat{\x}_\ti, \hat{\u}_\ti) + \beta_{\ni} \bSigma_\ni(\hat{\x}_\ti, \hat{\u}_\ti) \hat{\u}'_\ti + \noise_{\ti}$ \label{alg:hda:next-state}.
        \State Append transition to buffer $\mathcal{D}_{\mathrm{h}} \gets \mathcal{D}_{\mathrm{h}} \cup \{(\hat{\x}_\ti, \hat{\x}_{\ti+1}, \hat{\u}_{\ti},  \hat{\u}'_\ti, r(\hat{\x}_\ti, \hat{\u}_{\ti})) \}$.
    \EndFor
    \EndFor
    \State{\color{blue} /* Optimize Policy */}
    \State $\theta_{\ni,i+1}, \vartheta_{\ni,i+1} \gets \texttt{ModelFree}(\mathcal{D}_{\mathrm{h}}, \theta_{\ni,i}, \vartheta_{\ni,i})$
    \EndFor
    \OUTPUT Final policy and critic $\theta_{\ni} = \theta_{\ni, \Ti_{\text{iter}}}$, $\vartheta_{\ni} = \vartheta_{\ni, \Ti_{\text{iter}}}$
  \end{algorithmic}
\end{algorithm}

\textbf{Back-Propagation Through Time} is an algorithm that updates the policy parameters by computing the derivatives of the performance w.r.t. the parameters directly. 
For instance, PILCO from \citet{Deisenroth2011PILCO} and MBAC from \citet{Clavera2020MAAC} are different examples of practical algorithms that use a greedy policy~\eqref{eq:expected_performance_exploration} using GPs and ensembles of neural networks, respectively.
In \Cref{alg:hbptt}, we show how to adapt BPTT to hallucinated control.
Like in BPTT it samples the trajectories in a differentiable way, i.e., using the reparameterization trick \citep{Kingma2013AutoEncoding}. 
Under some assumptions (such as moment matching), the sampling step in \Cref{alg:hbptt:next-state} of \Cref{alg:hbptt} can be replaced by exact integration as in PILCO \citep{Deisenroth2011PILCO}.
While performing the rollout, it computes the performance and at the end it bootstrapped with a critic. 
This critic is learned using a policy evaluation \texttt{PolEval} algorithm such as Fitted Value Iteration \citep{antos2008fitted}. 
This strategy usually suffers from high variance due to the stochasticity of the sampled trajectories and the compounding of gradients \citep{McHutchon2014Modelling}.
Interestingly, \citet{Parmas2018Pipps} propose a method to combine the model-free gradients given by any HDA strategy together with the model-based gradients given by HBPTT, but we leave this for future work.
We found that limiting the KL-divergence between the policies in different episodes as suggested by \citet{Schulman2015TRPO} helps to control this variance by regularization.

\begin{algorithm}[htpb]
  \caption{Hallucinated Back-Propagation Through Time}
  \label{alg:hbptt}
  \begin{algorithmic}[1]
    \INPUT{Calibrated dynamical model $(\bmu, \bSigma)$,
          reward function $r(\x, \u)$,
          horizon $\Ti$, 
          initial state distribution $d(\x_0)$, 
          number of iterations $\Ti_{\text{iter}}$,
          initial parameters $\theta_{\ni-1}, \vartheta_{\ni-1}$, 
          learning rate $eta$,
          policy evaluation algorithm \texttt{PolEval},
          regularization $\lambda$. 
          }
    \State Initialize $\theta_{\ni,0} \gets \theta_{\ni-1}, \vartheta_{\ni,0} \gets \vartheta_{\ni-1}$
    \For{$i = 1, \ldots, \Ti_{\text{iter}}$} 
    \State{\color{blue} /* Simulate Data */}
    \State Start from initial state distribution $\hat{\x}_0 \sim d(\x_0)$.
    \State Restart $J \gets 0$
    \For{$\ti = 0, \dots, \Ti-1$}
        \State Compute action $\hat{\u}_\ti \sim  \pi(\hat{\x}_\ti;\theta_{\ni,i})$, $\hat{\u}'_\ti \sim  \eta(\hat{\x}_\ti;\theta_{\ni,i})$
        \State Sample next state $\hat{\x}_{\ti + 1} \sim \bmu_\ni(\hat{\x}_\ti, \hat{\u}_\ti) + \beta_{\ni} \bSigma_\ni(\hat{\x}_\ti, \hat{\u}_\ti) \hat{\u}'_\ti + \noise_{\ti}$ \label{alg:hbptt:next-state}.
        \State Accumulate $J \gets J + \gamma^{\ti} r(\hat{\x}_\ti, \hat{\u}_\ti) - \lambda \mathrm{KL}(\pi(\hat{\x}_\ti;\theta_{\ni,i}) || \pi(\hat{\x}_\ti;\theta_{\ni-1}))$. 
    \EndFor
    \State Bootstrap $J \gets J + \gamma^{\Ti} Q(\hat{\x}_\Ti, \pi(\hat{\x}_\Ti; \theta_{\ni,i}), \eta(\hat{\x}_\Ti; \theta_{\ni,i}) ;\vartheta_{\ni, i})$
    \State{\color{blue} /* Optimize Policy */}
    \State Compute gradient $\partial J / \partial \theta_{\ni}$ with back-propagation through time.
    \State Do gradient step $\theta_{\ni,i + 1} \gets \theta_{\ni,i} + \eta \partial J / \partial \theta_{\ni}$
    \State Update Critic $\vartheta_{\ni,i + 1} \gets \texttt{PolEval}(\theta_{\ni,i + 1})$
    \EndFor
    \OUTPUT Final policy and critic $\theta_{\ni} = \theta_{\ni, \Ti_{\text{iter}}}$, $\vartheta_{\ni} = \vartheta_{\ni, \Ti_{\text{iter}}}$
  \end{algorithmic}
\end{algorithm}

\textbf{Model-Based Value Expansion} is an Actor-Critic approach that uses the model to compute the next-states for the Bellman target when learning the action-value function. 
It then uses pathwise derivatives \citep{mohamed2019monte} through the learned action-value function. 
For example MVE from \citep{feinberg2018model} and STEVE from \citet{Buckman2018Steve} use such strategy. 
In \Cref{alg:hall-value-exp}, we show H-MVE (Hallucinated-Model Based Value Expansion). 
Here we use optimistic trajectories only to learn the Bellman target. 
In turn, the learned action-values functions are optimistic and so are the pathwise gradients computed through them. 
This strategy is usually less data efficient than BPTT or IDA as it uses the model only to compute targets, but suffers less from model bias. 
To address data efficiency, one can combine HVE and HDA to compute optimistic value functions as well as simulating optimistic data. 

\begin{algorithm}[htpb]
  \caption{Hallucinated Value Expansion}
  \label{alg:hall-value-exp}
  \begin{algorithmic}[1]
    \INPUT{Calibrated dynamical model $(\bmu, \bSigma)$,
          reward function $r(\x, \u)$,
          number of steps $\Ti$, 
          number of iterations $\Ti_{\text{iter}}$,
          initial parameters $\theta_{\ni-1}, \vartheta_{\ni-1}$, 
          true data buffer $\mathcal{D}_{\mathrm{r}}$,
          learning rate $\eta$,
          polyak parameter $\tau$.
          }
    \State Initialize $\theta_{\ni,0} \gets \theta_{\ni-1}, \vartheta_{\ni,0} \gets \vartheta_{\ni-1}, \bar{\vartheta}_{\ni,0} \gets \vartheta_{\ni-1} $
    \For{$i = 1, \ldots,  \Ti_{\text{iter}}$}
    \State{\color{blue} /* Simulate Data */}
    \State Start from buffer $\hat{\x}_0 \sim \mathcal{D}_{\mathrm{r}}$.
    \State Initialize target $Q_{\textrm{target}} \gets 0$. 
    \State Compute prediction $Q_{\textrm{pred}} = Q(\hat{\x}_0; \vartheta_{\ni, i})$. 
    \For{$\ti = 0, \dots, \Ti-1$}
        \State Compute action $\hat{\u}_\ti \sim  \pi(\hat{\x}_\ti;\theta_{\ni,i})$, $\hat{\u}'_\ti \sim  \eta(\hat{\x}_\ti;\theta_{\ni,i})$
        \State Sample next state $\hat{\x}_{\ti + 1} \sim \bmu_\ni(\hat{\x}_\ti, \hat{\u}_\ti) + \beta_{\ni} \bSigma_\ni(\hat{\x}_\ti, \hat{\u}_\ti) \hat{\u}'_\ti + \noise_{\ti}$ \label{alg:hve:next-state}.
        \State Accumulate target $Q_{\textrm{target}} \gets \gamma^{\ti} r(\hat{\x}_\ti, \hat{\u}_\ti)$.
    \EndFor
    \State Bootstrap $Q_{\textrm{target}} \gets Q_{\textrm{target}} + \gamma^{\Ti} Q(\hat{\x}_\Ti, \pi(\hat{\x}_\Ti; \theta_{\ni,i}), \eta(\hat{\x}_\Ti; \theta_{\ni,i}) ;\bar{\vartheta}_{\ni, i})$
    \State{\color{blue} /* Optimize Critic */}
    \State $\vartheta_{\ni,i+1} \gets \vartheta_{\ni,i} - \eta \nabla_\vartheta (Q_{\textrm{pred}} - Q_{\textrm{target}} )^2$
    \State Update target parameters $\bar{\vartheta}_{\ni,i+1} \gets \tau \bar{\vartheta}_{\ni,i} + (1-\tau) \vartheta_{\ni,i+1}$
    
    \State{\color{blue} /* Optimize Policy */}
    \State $\theta_{\ni,i+1} \gets \theta_{\ni,i} + \eta \nabla_{\theta_{\ni, i}} Q(\hat{\x}_0; \vartheta_{\ni, i})$
    \EndFor
    \OUTPUT Final policy $\theta_{\ni} = \theta_{\ni, \theta_{\ni}}$.
  \end{algorithmic}
\end{algorithm}

\subsection{Online Planning} \label{ssec:op}
An alternative approach is to consider non-parametric policies and directly optimize the true and hallucinated actions as $\u_{\ti,\ni} \in [-1,1]^\ninp,\u'_{\ti,\ni} \in [-1,1]^\nstate$. 
This is usually called Model-Predictive Control (MPC) and it is implemented in a receding horizon fashion \citep{Morari1999}.
That means that for each new state encounter online the HUCRL planning problem \eqref{eq:optimistic_exploration} is solved using the actions as decission variables. 
This addresses model errors compounding as the trajectories are evaluated through the real trajectories, but it comes at high online computational costs, which limit the applicability of such algorithms to simulations. 

GP-MPC \citet{Kamthe2018DataEfficient} and PETS \citet{Chua2018Deep} are MPC-based methods that use the greedy policy~\eqref{eq:expected_performance_exploration} using GP and neural networks ensembles, respectively. 
Other MPC solvers such as POPLIN \citet{Wang2019Exploring} or POLO \citep{Lowrey2019POLO} are also compatible with such dynamical models.
In H-MPC (Hallucinated-MPC), we directly optimize both the control and hallucinated inputs jointly and any of the previous methods can be used as the MPC solver.  
\citet{Moldovan2015Optimismdriven} also use MPC to solve an optimistic exploration scheme but only on linear models and, like other on-line planning methods, are extremely slow for real-time deployment. 

To solve the optimization problem, approximate local solvers are usually used that rely either on sampling or on linearization. We discuss how to use both of them with hallucinated inputs. These algorithms can be used as the \texttt{Plan} method in \Cref{alg:hucrl}.

\paragraph{Random Sampling Methods}
An approximate way of solving MPC problems is to exhaustively sample the decision variables.
Shooting methods sample the actions and then propagate the trajectory through the model whereas collocation methods sample both the states and the actions. 
For simplicity, we only consider shooting methods. 
This method initializes particles at the current state. 
For each particle, it samples a sequence of actions from a proposal distribution and rollouts each particle independently, computing the returns of such sequence. 
This process is repeated updating the proposal distribution. 
Random Shooting \citep{Richards2006Robust}, the Cross-Entropy Method \citep{Botev2013CEM}, and Model-Predictive Path Integral Control \citep{Williams2016MPPI} differ in the ways to select the elite actions between iterations and how to update the sampling distributions. 
All these methods maintain a distribution over the actions. 
POPLIN from \citet{Wang2019Exploring} instead maintains a distribution over the weights of a policy network and samples different policies.
The main advantage of this method is that it correlates the random samples through the dynamics, possibly scalling to higher dimensions.
Any of these methods can be used with hallucination. 
We show in \Cref{alg:hall-shooting} the pseudo-code for a meta-Hallucinated shooting algorithm. 

\begin{algorithm}[t]
  \caption{Hallucinated Shooting Method}
  \label{alg:hall-shooting}
  \begin{algorithmic}[1]
    \INPUT{Calibrated dynamical model $(\bmu, \bSigma)$,
          terminal reward $V$,
          reward function $r(\x, \u)$,
          horizon $\Ti$, 
          current state $\x_\ti$, 
          number of particles $n_{\text{particle}}$,
          number of iterations $n_{\text{iter}}$,
          number of elite particles $n_{\text{elite}}$.
          initial sampling distribution $d(\cdot)$,
          algorithm to evaluate actions $\texttt{EliteActions}$,
          algorithm to update distribution $\texttt{UpdateDistribution}$.
          }
    \For{$i = 1, \dots, n_{\text{iter}}$}
        \State{\color{blue} /* Simulate Data */}
        \State{Initialize $n_{\text{particle}}$ at the current state $\hat{\x}_{0}^{(i)} = \x_\ti$}
        \State{Initialize $J^{(i)} \gets 0$}
    \For{$\ti = 0, \dots, \Ti-1$}
        \State Sample action $\hat{\u}^{(i)}_\ti, \hat{\u}'^{(i)}_\ti \sim d(\cdot)$
        \State Sample next state $\hat{\x}^{(i)}_{\ti + 1} \sim \bmu_\ti(\hat{\x}_{\ti}^{(i)}, \hat{\u}_{\ti}^{(i)}) + \beta_{\ni} \bSigma_\ti(\hat{\x}^{(i)}_{\ti}, \hat{\u}^{(i)}_{\ti}) \hat{\u}'^{(i)}_\ti + \noise_{\ti}$.
        \State Accumulate $J^{(i)} \gets J^{(i)} + \gamma^\ti r( \hat{\x}^{(i)}_{\ti}, \hat{\u}^{(i)}_{\ti})$
    \EndFor
    \State Bootstrap $J^{(i)} \gets J^{(i)} + \gamma^{\Ti} V(\hat{\x}^{(i)}_\Ti)$.
    \State $a, a' \gets \texttt{EliteActions} (J^{(i)}, \hat{\u}^{(i)}_{0:\Ti-1}, \hat{\u}'^{(i)}_{0:\Ti-1}, n_{\textrm{elite}})$
    \State{\color{blue} /* Optimize Policy */}
    \State{Update proposal distribution $d(\cdot) \gets \texttt{UpdateDistribution}(a, a')$.}
    \EndFor
    \OUTPUT{Return best action $a, a' \gets \texttt{EliteActions}(J^{(i)}, \hat{\u}^{(i)}_{0:\Ti-1}, \hat{\u}'^{(i)}_{0:\Ti-1}, 1) $.} 
  \end{algorithmic}
\end{algorithm}

\paragraph{Differential Dynamic Programming (DDP)}
DDP can be interpreted as a second-order shooting method \citet{Jacobson1968DDP} for dynamical systems.
For linear dynamical models with quadratic costs, problem \cref{eq:expected_performance_exploration} is a quadratic program (QP) that enjoys a closed form solution \citep{Morari1999}.
To address non-linear systems and other cost functions, a common strategy is to use a variant of iLQR \citet{Li2004iLQR,Todorov2005iLQG,Tassa2012Synthesis} which linearizes the system and uses a second order approximation to the cost function to solve sequential QPs (SQP) that approximate the original problem. 
When the rewards and dynamical model are differentiable, this method is faster to sampling methods as it uses the problem structure to update the sampling distribution.

\subsection{Combining Offline Policy Search with Online Planning} \label{ssec:combining}
MPC methods suffer less from model bias, but typically require substantial computation. 
Furthermore, they are limited to the planning horizon unless a \emph{learned} terminal reward is used to approximate the reward-to-go \citep{Lowrey2019POLO}. 
On the other hand, off-policy search approaches yield policies and value function estimates (critics) that are fast to evaluate, but suffer from bias \citep{VanHasselt2019MBRLbias}.
We propose to combine these methods to get the best of both worlds: 
First, we learn parametric policies $\pi$ and $\eta$ using a policy search algorithm. 
Then, we use such policies as a warm-start for the sampling distributions of the planning algorithm. 
We name this planning algorithm Dyna-MPC, as it resembles the Dyna architecture proposed by \citet{Sutton1990Integrated} and we show the pseudo-code for hallucinated models in \Cref{alg:dyna-mpc}. 

Closely related to Dyna-MPC is POPLIN \citep{Wang2019Exploring}. 
We also use a policy to initialize actions and and then refine them with a shooting method. 
Nevertheless, we use a policy search algorithm to optimize the policy parameters instead of the cross-entropy method. 
\citet{Hong2019MBLRL} also uses MPC to refine an off-line learned policy. 
However, they use a model-free algorithm directly form real data instead of model-based policy search.

\begin{algorithm}[t]
  \caption{Dyna-MPC with Hallucinated Models}
  \label{alg:dyna-mpc}
  \begin{algorithmic}[1]
    \INPUT{Calibrated dynamical model $(\bmu, \bSigma)$,
          learned policies $\pi(\cdot; \theta)$, $\eta(\cdot; \theta)$
          learned critic $Q(\cdot; \vartheta)$,
          reward function $r(\x, \u)$,
          horizon $\Ti$, 
          current state $\x_\ti$, 
          number of particles $n_{\text{particle}}$,
          number of iterations $n_{\text{iter}}$,
          number of elite particles $n_{\text{elite}}$.
          initial sampling distribution $d(\cdot)$,
          algorithm to evaluate actions $\texttt{EliteActions}$,
          algorithm to update distribution $\texttt{UpdateDistribution}$.
          }
    \For{$i = 1, \dots, n_{\text{iter}}$}
        \State{\color{blue} /* Simulate Data */}
        \State{Initialize $n_{\text{particle}}$ at the current state $\hat{\x}_{0}^{(i)} = \x_\ti$}
        \State{Initialize $J^{(i)} \gets 0$}
    \For{$\ti = 0, \dots, \Ti-1$}
        \State Sample action $\hat{\u}^{(i)}_\ti, \hat{\u}'^{(i)}_\ti \sim (\pi(\hat{\x}_{\ti}^{(i)};\theta), \eta(\hat{\x}_{\ti}^{(i)};\theta)) + d(\cdot)$
        \State Sample next state $\hat{\x}^{(i)}_{\ti + 1} \sim \bmu_\ti(\hat{\x}_{\ti}^{(i)}, \hat{\u}_{\ti}^{(i)}) + \beta_{\ni} \bSigma_\ti(\hat{\x}^{(i)}_{\ti}, \hat{\u}^{(i)}_{\ti}) \hat{\u}'^{(i)}_\ti + \noise_{\ti}$.
        \State Accumulate $J^{(i)} \gets J^{(i)} + \gamma^\ti r( \hat{\x}^{(i)}_{\ti}, \hat{\u}^{(i)}_{\ti})$
    \EndFor
    \State Bootstrap $J^{(i)} \gets J^{(i)} + \gamma^{\Ti} Q(\hat{\x}^{(i)}_\Ti, \hat{\u}^{(i)}_\Ti, \hat{\u}'^{(i)}_\Ti; \vartheta)$.
    \State $a, a' \gets \texttt{EliteActions} (J^{(i)}, \hat{\u}^{(i)}_{0:\Ti-1}, \hat{\u}'^{(i)}_{0:\Ti-1}, n_{\textrm{elite}})$
    \State{\color{blue} /* Optimize Policy */}
    \State{Update proposal distribution $d(\cdot) \gets \texttt{UpdateDistribution}(a, a')$.}
    \EndFor
    \OUTPUT{Return best action $a, a' \gets \texttt{EliteActions}(J^{(i)}, \hat{\u}^{(i)}_{0:\Ti-1}, \hat{\u}'^{(i)}_{0:\Ti-1}, 1) $.} 

  \end{algorithmic}
\end{algorithm}

%% file: appendix/exploration_proofs.tex

\section{Proofs for Exploration Regret Bound}
\label{ap:exploration_proof}

In this section, we prove the main theorem.

\subsection{Notation}

In the following, we implicitly denote with $\x_{\ti, \ni}$ the states visited under the true dynamics $f$ in \cref{eq:stochastic_dynamic_system_additive} and with $\xo_\ti$ the states visited under $\pi_\ni$ but the optimistic dynamics $\tilde{f}_\ni(\x, \u) = \bmu_{\ni-1}(\x, \u) + \bSigma_{\ni-1}(\x, \u) \eta_\ni(\x, \u)$,
\begin{subequations}
\begin{align}
    \x_{\ti+1, \ni} &= f(\x_{\ti,\ni}, \u_{\ti, \ni} ) + \noise_{\ti,\ni} \\
    \u_{\ti, \ni} &= \pi_\ni(\x_{\ti,\ni}  ) \\
    \intertext{and} 
    \xo_{\ti+1, \ni} &= \tilde{f}_\ni(\x_{\ti,\ni}, \uo_{\ti, \ni} ) + \noise_{\ti,\ni} \\
    &= \bmu_{\ni-1} (\x_{\ti,\ni}, \uo_{\ti, \ni} )
    + \bSigma_{\ni-1} (\x_{\ti,\ni}, \uo_{\ti, \ni} ) \eta_\ni(\x_{\ti,\ni}, \uo_{\ti, \ni})
    + \noise_{\ti,\ni} \\
    \uo_{\ti, \ni} &= \pi_\ni(\xo_{\ti,\ni} ) .
\end{align}
\label{eq:xo_x_definitions_for_proofs}
\end{subequations}
Since the control actions $\u_{\ti,\ni} = \pi_\ni(\x_{\ti,\ni})$ and $\uo_{\ti,\ni} = \pi_\ni(\xo_{\ti,\ni})$ are fixed given $\pi_\ni$, we generally drop the dependence on $u$ and write $f(\x) = f(\x, \pi_\ni(\x))$, $\bmu(\x, \pi_\ni(\x))$, etc. We also drop the subscript $\ni$ from $\x_{\ti,\ni}$ whenever it is clear that we refer to the $\ni$th episode. Lastly, when no norm is specified, $\| \cdot \| = \| \cdot \|_2$ refers to the two-norm.

We start by clarifying that as a consequence of \cref{as:dynamics_f_lipschitz,as:reward_lipschitz} the closed-loop dynamics are Lipschitz continuous too.
\begin{corollary}
    \label{cor:f_lipschitz}
     As in \cref{as:boundedness_of_dynamics}, let the open-loop dynamics $f$ in \cref{eq:stochastic_dynamic_system_additive} be $L_f$-Lipschitz continuous and the policy $\pi \in \Pi$ be $L_\pi$-Lipschitz continuous w.r.t. to the 2-norm. Then the closed-loop system is $L_\mathrm{fc}$-Lipschitz continuous with $L_\mathrm{fc} = L_f \sqrt{1 + L_\pi}$.
\end{corollary}
\begin{proof}
  \begin{align}
    \|f(\x, \pi(\x)) - f(\x', \pi(\x')) \|_2
    &\leq L_f \| (\x - \x', \pi(\x) - \pi(\x')) \|_2 \\
    &= L_f \sqrt{ \| (\x - \x'\|_2^2 + \| \pi(\x) - \pi(\x')) \|_2^2 } \\
    &\leq L_f \sqrt{ \| (\x - \x'\|_2^2 + L_\pi \| \x - \x')) \|_2^2 } \\
    &= \underset{\eqdef L_\mathrm{fc}}{\underbrace{ L_f \sqrt{1 + L_\pi}}} \| \x - \x'\|_2 
  \end{align}
\end{proof}

\subsection{Bounding the Regret}
We start by bounding the cumulative regret in terms of the predictive variance of the states/actions on the true trajectory (the one that we will later collect data one). 

\begin{lemma}
  \label{lem:eta_fun_existance}
  Under \cref{as:well_calibrated_model}, for any sequence $\x_{\ti,\ni}$ generated by the true system \cref{eq:stochastic_dynamic_system_additive}, there exists a function $\eta \colon \R^\nstate \to [-1, 1]^\nstate$ such that $\x_{\ti,\ni} = \xo_{\ti,\ni}$ if $\noise = \tilde{\noise}$.
\end{lemma}
\begin{proof}
    By \cref{as:well_calibrated_model} we have $| f(\x) - \bmu(\x) | \leq \beta \bsigma(\x)$ elementwise. Thus for each $\x, \u$ there exists a vector $\bm{\eta}$ with values in $[-1, 1]^\nstate$ such that $f(\x, \u) = \mu(\x, \u) + \bSigma(\x, \u) \bm{\eta}$. Let the function $\eta(\cdot)$ return this vector for each state and action, then the result follows.
\end{proof}

\begin{lemma}
  Under \cref{as:well_calibrated_model}, with probability at least $(1 - \delta)$ we have for all $\ni \geq 0$ that the regret $r_\ni$ is bounded by
  \begin{equation}
    r_\ni = J(f, \pi^*) - J(f, \pi_\ni) \leq J(\tilde{f}_\ni, \pi_\ni) - J(f, \pi_\ni)
  \end{equation}
\end{lemma}
\begin{proof}
  By \cref{as:well_calibrated_model}, we know from \cref{lem:eta_fun_existance} that the true dynamics are contained within the feasible region of \cref{eq:optimistic_exploration}; that is, there exists an $\eta(\cdot) \colon \R^\nstate \times \R^\ninp \to [-1, 1]^\nstate$ such that with $\tilde{f}(\x) = \bmu(\x) + \bSigma(\x) \eta(\x)$ we have $J(f, \pi^*) = \tilde{J}(\tilde{f}, \pi^*)$. As a consequence, we have $J(f, \pi^*) \leq J(\tilde{f}_\ni, \pi_\ni)$ and the result follows.
\end{proof}

Thus, to bound the instantaneous regret $r_\ni$, we must bound the difference between the optimistic value estimate $J(\tilde{f}_\ni, \pi_\ni)$ and the true value $J(f, \pi_\ni)$. We can use the Lipschitz continuity properties to obtain

\begin{lemma}
  \label{lem:exploration:regret:J_difference_bound_by_state_divergence}
  Based on \cref{as:reward_lipschitz,as:pi_lipschitz} we have
\begin{equation}
  | J(\tilde{f}_\ni, \pi_\ni) - J(f, \pi_\ni)| \leq L_r \sqrt{1 + L_\pi}  \sum_{\ti=0}^\Ti \E[\noise = \tilde{\noise}]{ \|\x_{\ti,\ni} - \xo_{\ti,\ni} \|_2 }
\end{equation}
\end{lemma}
\begin{proof}
\begin{align}
  | J(\tilde{f}_\ni, \pi_\ni) - J(f, \pi_\ni)| 
  &= \left| \E[\tilde{\noise}] { \sum_{\ti=0}^\Ti r(\xo_\ti, \pi_\ni(\xo_\ti)) } 
  - \E[\noise]{ \sum_{\ti=0}^\Ti r(\x_\ti, \pi_\ni(\x_\ti)) } \right|  \\
  &= \left| 
  \E[\noise = \tilde{\noise}]{ \sum_{\ti=0}^\Ti r(\xo_\ti, \pi_\ni(\xo_\ti))  -  r(\x_\ti, \pi_\ni(\x_\ti))  } 
  \right|    \\
  &\leq L_r \sqrt{1 + L_\pi}  \sum_{\ti=0}^\Ti \E[\noise = \tilde{\noise}]{ \|\xo_\ti - \x_\ti \|_2 },
\end{align}
where $\E[\noise = \tilde{\noise}]{\cdot}$ means in expectation over $\noise$ and with $\tilde{\noise} = \noise$; that is, $\tilde{\noise}$ and $\noise$ are the same random variable. 
\end{proof}

\begin{figure}[ht]
  \includegraphics[width=\linewidth]{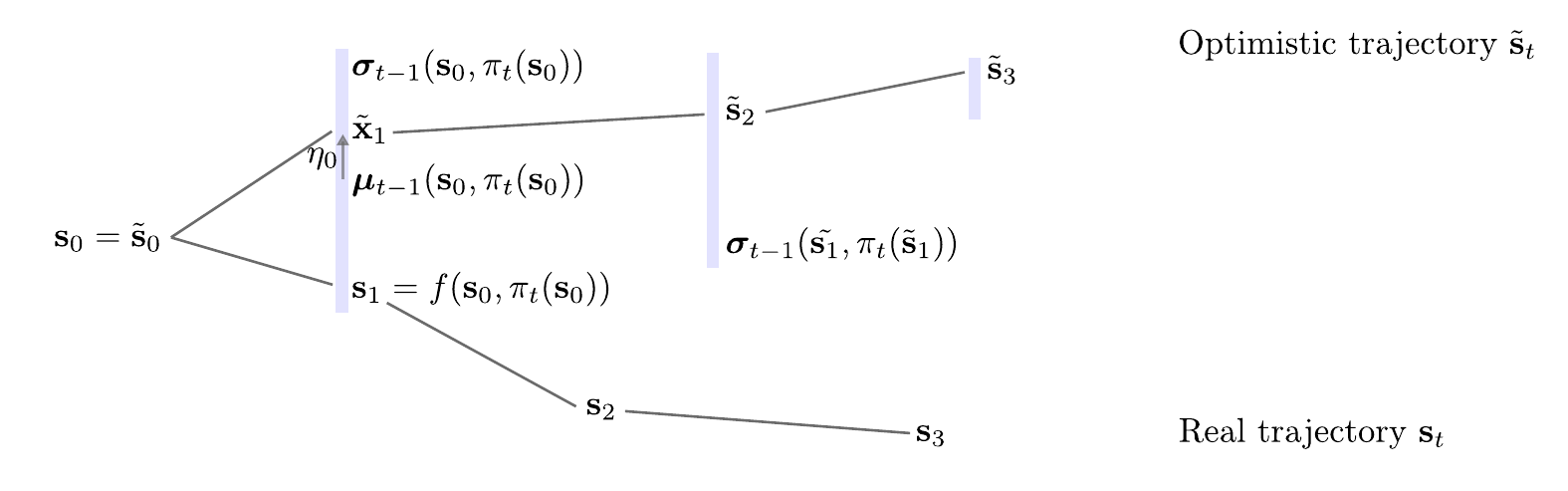}
  \caption{Illustrative comparison of the true state trajectory $\x_\ti$ under the policy $\pi_{\btheta}$ and the optimistic trajectory $\xo_\ti$ from \cref{eq:optimistic_exploration}. After one step, $\x_1$ is contained within the confidence intervals (grey bars). The optimistic dynamics are chosen within this confidence interval to maximize performance. Since the optimistic dynamics are constructed iteratively based on the previous state $\xo_\ti$, beyond one step the true dynamics are not contained in the confidence intervals.}
  \label{fig:optimistic_trajectory}
\end{figure}

What remains is to bound the deviation of the optimistic and the true trajectory. We show a different perspective of \cref{fig:optimistic_trajectory_new} in \cref{fig:optimistic_trajectory}, where we explicitly show the ``real'' state trajectory under a policy and for a given noise realisation the the optimistic trajectory with its one-step uncertainty estimates as in \cref{eq:optimistic_exploration}. We exploit the Lipschitz continuity of $\bsigma$ from \cref{as:model_predictions_lipschitz} in order to bound the deviation in terms of $\bsigma_{\ni-1}$ at states of the ``real'' trajectory.

\begin{lemma}
  \label{lem:exploration:regret:bounded_state_difference_with_same_noise_noconstraint}
  Under \cref{as:dynamics_f_lipschitz,as:pi_lipschitz,as:reward_lipschitz,as:well_calibrated_model,as:model_predictions_lipschitz}, let $\bar{L}_f = 1 + L_\mathrm{fc} + 2 \beta_{\ni-1} L_\sigma \sqrt{1 + L_\pi}$. Then, 
  for all iterations $\ni > 0$, any function $\eta \colon \R^\nstate \times \R^\ninp \to [-1, 1]^\nstate$ and any sequence of~$\noise_\ti$ with $\tilde{\noise}_\ti = \noise_\ti$, $\pi \in \Pi$ with $1 \leq \ti \leq \Ti$ we have that
  \begin{equation}
    \| \x_{\ti,\ni} - \xo_{\ti,\ni} \| \leq 2 \beta_{\ni-1} \bar{L}_f^{\Ti-1} \sum_{i=0}^{\ti-1} \| \bsigma_{\ni-1}(\x_{i,\ni}) \| 
  \end{equation}
\end{lemma}
\begin{proof}
  We start by showing that, for any $\ti \geq 1$ we have
  \begin{equation}
    \| \x_{\ti,\ni} - \xo_{\ti,\ni} \| \leq 2 \beta_{\ni-1} \sum_{i=0}^{\ti-1} (L_\mathrm{fc} + 2 \beta_{\ni-1} L_\sigma \sqrt{1 + L_\pi})^{\ti-1-i} \| \bsigma_{\ni-1}(\x_{i,\ni}) \| 
    \label{eq:stoch_sample_confidence_noconstraint}
  \end{equation}
  by induction. For the base case we have~$\xo_0 = \x_0$. Consequently, at iteration $\ni$ we have
  \begin{align}
    \| \x_{1,\ni} - \xo_{1,\ni}\|
    &= \| f(\x_0) + \noise_0 - \bmu_{\ni-1}(\x_0) - \beta_{\ni-1} \bSigma_{\ni-1}(\x_0) \eta(\x_0) - \tilde{\noise}_0 \| \\
    &\leq \|  f(\x_0) - \bmu_{\ni-1}(\x_0) \| + \beta_{\ni-1}  \| \bSigma_{\ni-1}(\x_0) \eta(\x_0) \| \\
    &\leq \beta_{\ni-1} \| \bsigma_{\ni-1}(\x_0) \| + \beta_{\ni-1} \| \bsigma_{\ni-1}(\x_0) \| \\
    &= 2 \beta_{\ni-1} \| \bsigma_{\ni-1}(\x_0) \|
  \end{align}

  For the induction step assume that \cref{eq:stoch_sample_confidence_noconstraint} holds at time step~$\ti$. Subsequently we have at iteration $\ni$ that
  \begin{align*}
    &\| \x_{\ti+1, \ni} - \xo_{\ti+1, \ni} \|  \\
    &= \| f(\x_\ti) + \noise_\ti - \bmu_{\ni-1}(\xo_\ti) - \beta_{\ni-1} \bSigma_{\ni-1}(\xo_\ti) \eta(\xo_\ti) - \tilde{\noise}_{\ti} \| \\
    &= \| f(\x_\ti) - \bmu_{\ni-1}(\xo_\ti) - \beta_{\ni-1} \bSigma_{\ni-1}(\xo_\ti) \eta(\xo_\ti) + f(\xo_\ti) - f(\xo_\ti)\| \\
    &= \| f(\xo_\ti) - \bmu_{\ni-1}(\xo_\ti) - \beta_{\ni-1} \bSigma_{\ni-1}(\xo_\ti) \eta(\xo_\ti) + f(\x_\ti) - f(\xo_\ti)\| \\
    &= \| f(\xo_\ti) - \bmu_{\ni-1}(\xo_\ti) \| + \| \beta_{\ni-1} \bSigma_{\ni-1}(\xo_\ti) \eta(\xo_\ti)\| + \| f(\x_\ti)  - f(\xo_\ti)\| \\
    &\leq \beta_{\ni-1} \| \bsigma_{\ni-1}(\xo_\ti) \| +  \beta_{\ni-1} \| \bsigma_{\ni-1}(\xo_\ti) \| +  L_\mathrm{fc} \|\x_\ti - \xo_\ti \| \\
    &= 2 \beta_{\ni-1} \| \bsigma_{\ni-1}(\xo_\ti) \| +  L_\mathrm{fc} \|\x_\ti - \xo_\ti \| \\
    &= 2 \beta_{\ni-1} \| \bsigma_{\ni-1}(\x_\ti) + \bsigma_{\ni-1}(\xo_\ti) - \bsigma_{\ni-1}(\x_\ti) \| +  L_\mathrm{fc} \|\x_\ti - \xo_\ti \| \\
    &\leq 2 \beta_{\ni-1} \left( \| \bsigma_{\ni-1}(\x_\ti) \| + L_\sigma \sqrt{1 + L_\pi} \| \x_\ti - \xo_\ti \| \right) +  L_\mathrm{fc} \|\x_\ti - \xo_\ti \| \\
    &= 2 \beta_{\ni-1} \| \bsigma_{\ni-1}(\x_\ti) \| + (L_\mathrm{fc} + 2 \beta_{\ni-1} L_\sigma \sqrt{1 + L_\pi}) \| \x_\ti - \xo_\ti \| \\
    &\leq 2 \beta_{\ni-1} \| \bsigma_{\ni-1}(\xo_\ti) \| + (L_\mathrm{fc} + 2 \beta_{\ni-1} L_\sigma \sqrt{1 + L_\pi}) 2 \beta_{\ni-1} \sum_{i=0}^{\ti-1} (L_\mathrm{fc} + 2 \beta_{\ni-1} L_\sigma \sqrt{1 + L_\pi})^{\ti-1-i} \| \bsigma_{\ni-1}(\x_i) \| \\
    &= 2 \beta_{\ni-1} \sum_{i=0}^{(\ti+1) - 1 } (L_\mathrm{fc} + 2 \beta_{\ni-1} L_\sigma \sqrt{1 + L_\pi})^{(\ti+1) -1 -i} \| \bsigma_{\ni-1}(\x_i) \|
  \end{align*}

  Thus \cref{eq:stoch_sample_confidence_noconstraint} holds. Now since $\ti \leq \Ti$ we have
  \begin{align}
    \| \x_{\ti,\ni} - \xo_{\ti,\ni} \| 
    &\leq 2 \beta_{\ni-1} \sum_{i=0}^{\ti-1} (L_\mathrm{fc} + 2 \beta_{\ni-1} L_\sigma \sqrt{1 + L_\pi})^{\ti-1-i} \| \bsigma_{\ni-1}(\x_{i,\ni}) \| \\
    &\leq 2 \beta_{\ni-1} \sum_{i=0}^{\ti-1} (1 + L_\mathrm{fc} + 2 \beta_{\ni-1} L_\sigma \sqrt{1 + L_\pi})^{\ti-1-i} \| \bsigma_{\ni-1}(\x_{i,\ni}) \| \\
    &\leq 2 \beta_{\ni-1} \underset{\eqdef \bar{L}_f}{\underbrace{(1 + L_\mathrm{fc} + 2 \beta_{\ni-1} L_\sigma \sqrt{1 + L_\pi})}}^{\Ti-1} \sum_{i=0}^{\ti-1} \| \bsigma_{\ni-1}(\x_{i,\ni}) \| \\
  \end{align} 
\end{proof}

\begin{corollary}
  \label{lem:exploration:regret:bound_expected_state_divergence_by_sigma}
  Under the assumptions of \cref{lem:exploration:regret:bounded_state_difference_with_same_noise_noconstraint}, for any sequence of~$\eta_\ti \in [-1, 1]$, $\btheta \in \mathcal{D}$, and $\ti \geq 1$, $\ni \geq 1$ we have that
  \begin{equation}
    \E[\noise = \tilde{\noise}]{ \| \x_{\ti,\ni} - \xo_{\ti,\ni} \| }
    \leq 2 \beta_{\ni-1} \bar{L}_f^{\Ti-1} \E[\noise]{\sum_{i=0}^{\ti-1} \| \bsigma_{\ni-1}(\x_{i, \ni}) \| } 
  \end{equation} 
\end{corollary}
\begin{proof}
  This is a direct consequence of \cref{lem:exploration:regret:bounded_state_difference_with_same_noise_noconstraint}.
\end{proof}

As a direct consequence of these lemmas, we can bound the regret in terms of the predictive uncertainty of our statistical model in expectation over the states visited under the true dynamics.

\begin{lemma}
  \label{lem:exploration:regret:bound_squared_regret_by_expected_variance}
  Under \cref{as:well_calibrated_model,as:reward_lipschitz,as:pi_lipschitz}, let $L_J = 2 L_r \sqrt{1 + L_\pi} \beta_{\ni-1} \bar{L}_f^{\Ti-1}$. Then, with probability at least $(1 - \delta)$ it holds for all $\ni \geq 0$ that
  \begin{equation}
    r_\ni^2 \leq L_J^2 \Ti^3 \E[\noise]{ \sum_{\ti=0}^{\Ti-1}
    \| \bsigma_{\ni-1}(\x_{\ti,\ni}) \|_2^2} 
    %
  \end{equation}
\end{lemma}
\begin{proof}
\begin{align}
  r_\ni &\leq J(\tilde{f}_\ni, \pi_\ni) - J(f, \pi_\ni)  \\
  &\leq L_r \sqrt{1 + L_\pi}  \sum_{\ti=0}^\Ti \E[\noise = \tilde{\noise}]{ \|\x_{\ti,\ni} - \xo_{\ti,\ni} \|_2 } \\
  &\leq 2 L_r \sqrt{1 + L_\pi} \beta_{\ni-1} \bar{L}_f^{\Ti-1} \sum_{\ti=0}^\Ti \E[\noise]{ 
    \sum_{i=0}^{\ti-1}  \| \bsigma_{\ni-1}(\x_{i,\ni}) \|_2
  }  \\
  &\leq 2 L_r \sqrt{1 + L_\pi} \beta_{\ni-1} \bar{L}_f^{\Ti-1} \Ti \E[\noise]{ 
    \sum_{\ti=0}^{\Ti-1} \| \bsigma_{\ni-1}(\x_{\ti,\ni}) \|_2 
  } 
\end{align}
where the third inequality follows from \cref{lem:exploration:regret:bound_expected_state_divergence_by_sigma}.
Now, let $L_J = 2 L_r \sqrt{1 + L_\pi} \beta_{\ni-1} \bar{L}_f^{\Ti-1}$, so that
\begin{align}
  r_\ni &\leq L_J \Ti \E[\noise]{ 
    \sum_{\ti=0}^{\Ti-1}  \| \bsigma_{\ni-1}(\x_{\ti,\ni}) \|_2 
  } \\
  r_\ni^2 &\leq L_J^2 \Ti^2 \left( \E[\noise]{ 
    \sum_{\ti=0}^{\Ti-1}  \| \bsigma_{\ni-1}(\x_{\ti,\ni}) \|_2 } \right)^2 \\
  &\leq L_J^2 \Ti^2 \E[\noise]{ \left( 
    \sum_{\ti=0}^{\Ti-1}  \| \bsigma_{\ni-1}(\x_{\ti,\ni}) \|_2 \right)^2 }  \\
  &\leq L_J^2 \Ti^3 \E[\noise]{  
    \sum_{\ti=0}^{\Ti-1}  \| \bsigma_{\ni-1}(\x_{\ti,\ni}) \|_2^2 }  
\end{align}
\end{proof}

\begin{lemma}
  \label{lem:exploration:regret:bound_cumulative_regret_by_expected_uncertainty}
  Under the assumption of \cref{as:well_calibrated_model,as:dynamics_f_lipschitz,as:pi_lipschitz,as:reward_lipschitz,as:model_predictions_lipschitz},  with probability at least $(1 - \delta)$ it holds for all $\ni \geq 0$ that
  \begin{equation}
    R_\Ni^2 \leq \Ni  L_J^2 \Ti^3 \sum_{\ni=1}^\Ni\E[\noise]{ \sum_{\ti=0}^{\Ti-1} \| \sigma_{\ni-1}(\x_{\ti,\ni}, \u_{\ti,\ni})^2 \|_2^2 } 
  \end{equation}
\end{lemma}
\begin{proof}
  \begin{align}
    R_\Ni^2 &= \left( \sum_{\ni=1}^\Ni r_\ni \right)^2 \\
    &\leq \Ni \sum_{\ni=1}^\Ni r_\ni^2 && \text{Jensen's} \\ 
    &\leq \Ni  L_J^2 \Ti^3 \sum_{\ni=1}^\Ni\E[\noise]{ \sum_{\ti=0}^{\Ti-1} \| \bsigma_{\ni-1}(\x_{\ti,\ni}, \u_{\ti,\ni})^2 \|_2^2 } 
    && \text{
      \cref{lem:exploration:regret:bound_squared_regret_by_expected_variance}
    } 
  \end{align}
\end{proof}

That is, at every iteration $\ni$ the regret bound increases by the sum of predictive uncertainties in expectation over the true states that we may visit. This is an instance-dependent bound, since it depends on specific data collected up to iteration $\ni$ within $\sigma_{\ni-1}$. We will replace this with a worst-case bound in the following.

\begin{lemma}
  \label{lem:exploration:regret:bound_cumulative_regret_by_worst_case_uncertainty}
  Under the assumption of \cref{as:well_calibrated_model,as:dynamics_f_lipschitz,as:pi_lipschitz,as:reward_lipschitz,as:model_predictions_lipschitz}, let $\x_{\ti,\ni} \in \X_\ni$, $\X_{\ni-1} \subseteq \X_\ni$, and $\u_{\ti,\ni} \in \U$ for all $\ti,\ni>0$ with compact sets $\X_\ni$ and $\U$. Then, with probability at least $(1 - \delta)$ it holds for all $\ni \geq 0$ that
  \begin{equation}
    R_\Ni^2 \leq \Ni  L_J^2 \Ti^3 I_\Ni(\X_\ni, \U) 
  \end{equation}
  where 
  \begin{equation}
      I_\Ni(\X, \U) = \max_{\dataset_1, \dots, \dataset_\Ni \subset \X \times \X \times \U, \, |\dataset_i| = \Ti} \sum_{\ni=1}^\Ni \sum_{\x, \u \in \dataset_\ni} \| \bsigma_{\ni-1}(\x, \u) \|_2^2
  \end{equation}
\end{lemma}
\begin{proof}
    As a consequence of $\x_{\ti,\ni} \in \X_\ni$ we have 
    \begin{equation}
        \sum_{\ni=1}^\Ni\E[\noise]{ \sum_{\ti=0}^{\Ti-1} \| \bsigma_{\ni-1}(\x_{\ti,\ni}, \u_{\ti,\ni})^2 \|_2^2 } 
        \leq I_\Ni(\X_\ni, \U)
    \end{equation}
    and thus
    \begin{equation}
        R_\Ni^2 \leq \Ni  L_J^2 \Ti^3 I_\Ni(\X_\ni, \U) .
    \end{equation}
\end{proof}

\begin{theorem}
  \label{thm:exploration:regret:general_regret_bound_increasing_sets}
   Under \cref{as:dynamics_f_lipschitz,as:pi_lipschitz,as:reward_lipschitz,as:well_calibrated_model,as:model_predictions_lipschitz} let $\x_{\ti,\ni} \in \X_\ni$, $\X_{\ni-1} \subseteq \X_\ni$, and $\u_{\ti,\ni} \in \U$ for all $\ti,\ni>0$. 
   Then, for all $\Ni \geq 1$, with probability at least $(1-\delta)$, the regret of \alg in \cref{eq:optimistic_exploration} is at most
  $
    R_\Ni \leq \Or{\beta_{\Ni-1}^\Ti L_\sigma^\Ti  \sqrt{ \Ni \Ti^3  \, I_\Ni(\X_\Ni, \U) }  }
  $.
\end{theorem}
\begin{proof}
  From \cref{lem:exploration:regret:bound_cumulative_regret_by_worst_case_uncertainty} we have
  \begin{align}
    R_\Ni^2 \leq \Ni  L_J^2 \Ti^3 I_\Ni(\X_\ni, \U)  \\
    R_\Ni \leq L_J \sqrt{ \Ti^3 I_\Ni(\X_\ni, \U) }
  \end{align}
  where $L_J = 2 L_r \sqrt{1 + L_\pi} \beta_{\ni-1} \bar{L}_f^{\Ti-1}$ from \cref{lem:exploration:regret:bound_squared_regret_by_expected_variance} and $\bar{L}_f = 1 + L_f + 2 \beta_{\ni - 1} L_\sigma \sqrt{1 + L_\pi}$ from \cref{lem:exploration:regret:bounded_state_difference_with_same_noise_noconstraint}. Plugging in we get $L_J = 2 L_r \sqrt{1 + L_\pi} \beta_{\ni-1} (1 + L_f + 2 \beta_{\ni - 1} L_\sigma  \sqrt{1 + L_\pi})^{\Ti-1} = \Or{\beta_{\ni - 1}^\Ti L_\sigma^\Ti}$ so that
  \begin{align}
    R_\Ni &\leq \Or{ \beta_{\ni - 1}^\Ti L_\sigma^\Ti \sqrt{\Ti^3 I_\Ni(\X_\ni, \U) } } 
  \end{align} 
\end{proof}

\generalregretbound*
\begin{proof}
  A direct consequence of \cref{thm:exploration:regret:general_regret_bound_increasing_sets}.
\end{proof}

\ActivateWarningFilters[pdftoc]
\section{Properties of the Functions $\boldsymbol{\eta(\cdot)}$}
\DeactivateWarningFilters[pdftoc]
\label{ap:exploration:practical_implementation}
\label{ap:properties_of_eta_functions}

So far, we have considered general functions $\eta \colon \R^\nstate \times \R^\ninp \to [-1, 1]^\nstate$, which can potentially be discontinuous. However, as long as \cref{lem:eta_fun_existance} holds and the true dynamics are feasible in \cref{eq:optimistic_exploration}, we can use any more restrictive function class. In this section, we investigate properties of $\eta$.

It is clear, that it is sufficient to consider functions $\eta$ such that $\bSigma_\ni(\x) \eta(\x)$  is Lipschitz continuous, since it aims to approximate a Lipschitz continuous function $f$:
\begin{lemma}
    \label{lem:eta_times_sigma_lipschitz}
    With \cref{as:well_calibrated_model,as:dynamics_f_lipschitz,as:pi_lipschitz,as:model_predictions_lipschitz} let $\eta(\cdot)$ be a function such that $f(\x) - \bmu_\ni(\x) = \beta_\ni \bSigma_\ni(\x) \eta(\x)$ as in \cref{lem:eta_fun_existance}. Then $\bSigma_\ni(\x) \eta(\x)$ is Lipschitz continuous.
\end{lemma}
\begin{proof}
  \begin{align}
      \| \bSigma_\ni(\x) \eta(\x) - \bSigma_\ni(\x') \eta(\x')\|
      &\leq \| f(\x) - \bmu_\ni(\x) - ( f(\x') - \bmu_\ni(\x')) \| \\
      &\leq (L_f + L_\mu) \| \x -\x' \|
  \end{align}
\end{proof}

Unfortunately, the same is not true for $\eta$ on its own in general. However, if the predictive standard deviation $\bsigma$ does not decay to zero, this holds.

\begin{lemma}
  Under the assumptions of \cref{lem:eta_times_sigma_lipschitz} let $0 < \sigma_\mathrm{min} \leq \bsigma(\x, \u) \leq \sigma_\mathrm{max}$ elementwise for all $\x,\u \in \X \times \U$. Then, with probability at least $(1-\delta)$, there exists a Lipschitz-continuous function $\eta(\cdot)$ with $\| \eta(\cdot) \|_\infty = 1$ such that $f(\x) - \bmu_\ni(\x) = \beta_\ni \bSigma_\ni(\x) \eta(\x)$ for all $\x \in \R^\nstate$.
\end{lemma}
\begin{proof}
  By contradiction. Let $\eta(\cdot)$ be a function that is not Lipschitz continuous such that $f(\x) - \bmu(\x) = \beta \bSigma(\x) \eta(\x)$. By assumption we know that $\bsigma_\ni(\x)$ is strictly larger than zero and bounded element-wise from above by some constant. As a consequence, $\bSigma^{-1}(\x)$ exists and is $L_\sigma / \sigma_\mathrm{min}^2$-Lipschitz continuous w.r.t. the Frobenius norm. Thus, we have
  \begin{align*}
    &\| \eta(\x) - \eta(\x') \|_2 \\
    &= \| \frac{1}{\beta} \bSigma^{-1}(\x) (f(\x) - \bmu(\x)) - \frac{1}{\beta} \bSigma^{-1}(\x') (f(\x') - \bmu(\x')) \|_2 \\
    &\leq |\frac{1}{\beta}| \| \bSigma^{-1}(\x) ( (f(\x) - \bmu(\x)) - (f(\x') - \bmu(\x')) ) \|_2 
    + |\frac{1}{\beta}| \| \left( \bSigma^{-1}(\x) - \bSigma^{-1}(\x') \right) (f(\x') - \bmu(\x')) \|_2
    \\
    &\leq |\frac{1}{\beta}| \| \bSigma^{-1}(\x) \|_\mathrm{F} \| (f(\x) - \bmu(\x)) - (f(\x') - \bmu(\x')) \|_2 
    + |\frac{1}{\beta}| \| f(\x') - \bmu(\x') \|_2  \| \bSigma^{-1}(\x) - \bSigma^{-1}(\x') \|_\mathrm{F}
    \\
    &\leq |\frac{1}{\beta}| \| \bSigma^{-1}(\x) \|_\mathrm{F} (L_\mathrm{fc} + L_\mu  \sqrt{1 + L_\pi}) \| \x - \x' \|_2 
    + |\frac{1}{\beta}| \| \beta \bsigma(\x') \|_2  \| \bSigma^{-1}(\x) - \bSigma^{-1}(\x') \|_\mathrm{F}
    \\
    &\leq \frac{\sqrt{\nstate}}{\beta \sigma_\mathrm{min}} (L_\mathrm{fc} + L_\mu  \sqrt{1 + L_\pi}) \| \x - \x' \|_2 
    + \frac{\sqrt{\nstate}\sigma_\mathrm{max}}{\sigma_\mathrm{min}^2} L_\sigma \sqrt{1 + L_\pi} \| \x - \x' \|_2
  \end{align*}
  Since $\beta_\ni > 0$ we have that $\eta(\x)$ is Lipschitz continuous, which is a contradiction.
\end{proof}

Thus, it is generally sufficient to optimize over Lipschitz continuous functions in order to obtain the same regret bounds as in the optimistic case. However, it is important to note that the complexity of the function (i.e., its Lipschitz constant) will generally increase as the predictive variance decreases. It is easy to construct cases where $\bsigma(\cdot) = 0$ implies that $\eta$ has to be discontinuous. However, at least in theory $\bsigma(\cdot) = 0$ is impossible with finite data when the system is noisy ($\sigma > 0$). Also note that as $\bsigma$ decreases, the effect of $\eta$ on the dynamics also decreases.

This might also motivate optimizing over a function that model $\bSigma_{\ni-1}(\x, \u) \eta(\x, \u)$ jointly, since that one is regular even for $\bsigma(\cdot) = 0$. However, this would require regularizing the resulting function to be bounded by $\beta_\ni \bsigma_\ni(\x, \u)$ and might lead to difficulties with policy optimization, since the resulting hallucinated actions are no longer normalized to $[-1, 1]^{\nstate}$. We leave it as an avenue for future research.

%% file: appendix/gp_background.tex
\section{Background on Gaussian Processes}
\label{sec:back:gaussian_process}

Gaussian processes are a nonparametric Bayesian model that has a tractable, closed-form posterior distribution \citep{Rasmussen2006Gaussian}. The goal of Gaussian process inference is to infer a posterior distribution over a nonlinear map~${f'(\ab): \adomain \to \R}$ from an input vector~${\ab \in \adomain }$ with $\adomain \subseteq \R^d$ to the function value~$f'(\ab)$. This is accomplished by assuming that the function values $f'(\ab)$, associated with different values of $\ab$, are random variables and that any finite number of these random variables have a \emph{joint} normal distribution~\citep{Rasmussen2006Gaussian}.

A Gaussian process distribution is parameterized by a prior mean function and a covariance function or kernel $k(\ab, \ab')$, which defines the covariance of any two function values~$f(\ab)$ and $f(\ab')$ for ${\ab, \ab' \in \adomain}$. In this work, the mean is assumed to be zero without loss of generality. The choice of kernel function is problem-dependent and encodes assumptions about the unknown function. A review of potential kernels can be found in \citep{Rasmussen2006Gaussian}.

We can condition a Gaussian process on the observations $\mb{y}_\ni$ at input locations $\adomain_\ni$. The Gaussian process model assumes that observations are noisy measurements of the true function value with Gaussian noise, ${\omega \sim \mathcal{N}(0,\sigma^2)}$. The posterior distribution is again a Gaussian process with mean $\mu_\ni$, covariance $k_\ni$, and variance $\sigma_\ni$, where
\begin{align}
\mu_\ni(\ab) &= \mb{k}_\ni(\ab)  (\mb{K}_\ni + \mb{I} \sigma^2)^{-1} \mb{y}_\ni ,
\label{eq:gp_prediction_mean} \\
k_\ni(\ab, \ab') &= k(\ab, \ab') - \mb{k}_\ni(\ab) (\mb{K}_\ni + \mb{I} \sigma^2)^{-1} \mb{k}_\ni^\T(\ab'),
\label{eq:gp_prediction_covariance} \\
\sigma^2_\ni(\ab) &= k_\ni(\ab, \ab).
\label{eq:gp_prediction_variance}
\end{align}
The covariance matrix~${\mb{K}_\ni \in \R^{|\adomain_\ni| \times |\adomain_\ni|}}$ has entries ${[\mb{K}_\ni]_{(i,j)} = k(\ab_i, \ab_j)}$ with $\ab_i, \ab_j \in \adomain_\ni$ and
the vector
${\mb{k}_\ni(\ab) =
\left[ \begin{matrix}
	k(\ab,\ab_1), \dots,k(\ab,\ab_{|\adomain_\ni |})
\end{matrix}  \right]}$
contains the covariances between the input~$\ab$ and the observed data points in~$\adomain_\ni$.
The identity matrix is denoted by $\mb{I}$. 

Given the Gaussian process assumptions, we obtain point-wise confidence estimates from the marginal Normal distribution specified by $\mu_\ni$ and $\sigma_\ni$. For finite sets, the Gaussian process belief induces a \emph{joint} normal distribution over function values that is correlated through \cref{eq:gp_prediction_covariance}. We can use this to fulfill \cref{as:well_calibrated_model} for continuous sets by using a union bound and exploiting that samples from a Gaussian process are Lipschitz continuous with high probability \citep[Theorem 2]{Srinivas2012Gaussian}.

\subsection{Information Capacity}

One important property of normal distributions is that the confidence intervals contract after we observe measurement data. How much data we require for this to happen generally depends on the variance of the observation noise, $\sigma^2$, and the size of the function class; i.e., the assumptions that we encode through the kernel. In the following, we use results by \citet{Srinivas2012Gaussian} and use the mutual information to construct such a capacity measure.

Formally, the mutual information between the Gaussian process prior on $f'$ at locations $\overline{\adomain}$ and the corresponding noisy observations $\mb{y}_{\overline{\adomain}}$ is given by 
\begin{equation}
	\Mi{\mb{y}_{\overline{\adomain}}}{f'} = 0.5 \log | \mb{I} + \sigma^{-2} \mb{K}_{\overline{\adomain}} | ,
  \label{eq:mutual_information}
\end{equation}
where~$\mb{K}_{\overline{\adomain}}$ is the kernel matrix~$[k(\ab, \ab')]_{\ab, \ab' \in \overline{\adomain}}$ and $|\cdot|$ is the determinant. Intriguingly, for Gaussian process models this quantity only depends on the inputs in $\overline{\adomain}$ and not the corresponding measurements $\mb{y}_{\overline{\adomain}}$.
%
Intuitively, the mutual information measures how informative the collected samples~$\mb{y}_\adomain$ are about the function~$f$. If the function values are independent of each other under the Gaussian process prior, they provide large amounts of new information. However, if measurements are taken close to each other as measured by the kernel, they are correlated under the Gaussian process prior and provide less information.

The mutual information in \cref{eq:mutual_information} depends on the locations $\adomain_\ni$ at which we obtain measurements. While it can be computed in closed-form, it can also be bounded by the largest mutual information that any algorithm could obtain from $\ni$ noisy observations,
\begin{equation}
  \gamma_\ni = \max_{\adomain \subset D, \, |\adomain| \leq \ni} \Mi{\mb{y}_\adomain}{f'}.
  \label{eq:gamma_t}
\end{equation}
We refer to~$\gamma_\ni$ as the \emph{information capacity}, since it can be interpreted as a measure of complexity of the function class associated with a Gaussian process prior. It was shown by~\citet{Srinivas2012Gaussian} that~$\gamma_\ni$ has a sublinear dependence on~$\ni$ for many commonly used kernels such as the Gaussian kernel. This sublinear dependence is generally exploited by exploration algorithms in order to show convergence.

\subsection{Functions in a Reproducing Kernel Hilbert Space}
\label{sec:back:rkhs_function}

Instead of the Bayesian Gaussian process framework, we can also consider frequentist confidence intervals. Unlike the Bayesian framework, which inherently models a belief over a random function, frequentists assume that there is an \textit{a priori} fixed underlying function $f'$ of which we observe noisy measurements.

The natural frequentist counterpart to Gaussian processes are functions inside the Reproducing Kernel Hilbert Space (RKHS) spanned by the same kernel $k(\ab, \ab')$ as used by the Gaussian process in \cref{sec:back:gaussian_process}.
An RKHS~$\mathcal{H}_k$ contains well-behaved functions of the form~$f(\ab) = \sum_{i \geq 0} \alpha_i \, k(\ab, \ab_i)$, for given representer points~$\ab_i \in \R^d$ and weights $\alpha_i \in \R$ that decay sufficiently quickly. For example, the Gaussian process mean function \cref{eq:gp_prediction_mean} lies in this RKHS. The kernel function $k(\cdot, \cdot)$ determines the roughness and size of the function space and the induced RKHS norm~$\|f'\|_{k}^2 = \langle f', \, f' \rangle_k = \sum_{i,j \geq 0} \alpha_i \alpha_j k(\ab_i, \ab_j)$ measures the complexity of a function~$f' \in \mathcal{H}_k$ with respect to the kernel. In particular, the function $f'$ is Lipschitz continuous with respect to the kernel metric 
\begin{equation}
	d(\ab, \ab') = \sqrt{ k(\ab, \ab) + k(\ab', \ab') - 2 k(\ab, \ab')},
	\label{eq:kernel_metric}
\end{equation}
so that $| f'(\ab) - f'(\ab') | \leq \| f' \|_k d(\ab, \ab')$, see the proof of Proposition 4.30 by \citet{Christmann2008Support}.

\subsubsection{Confidence Intervals}
\label{sec:rkhs_confidence_intervals}

We can construct an estimate together with reliable confidence intervals if the measurements are corrupted by $\sigma$-sub-Gaussian noise. This is a class of noise where the tail probability decays exponentially fast, such as in Gaussian random variables or any distribution with bounded support. Specifically, we have the following definition.

\begin{definition}[\citet{Vershynin2010Introduction}]
	A random variable $X$ is $\sigma$-sub-Gaussian if $\mathbb{P} \left\{ |X| > s \right\} \leq \exp(1 - s^2 / \sigma^2)$ for all $s > 0$.
	\label{def:sub_gaussian}
\end{definition}

While the Gaussian process framework makes different assumptions about the function and the noise, Gaussian processes and RKHS functions are closely related \citep{Kanagawa2018Gaussian} and it is possible to use the Gaussian process posterior marginal distributions to infer reliable confidence intervals on~$f'$.

\begin{restatable}[\citet{Abbasi-Yadkori2012Online,Chowdhury2017Kernelized}]{lemma}{confidencethm}
Assume that $f$ has bounded RKHS norm $\|f'\|_k \leq B$ and that measurements are corrupted by~$\sigma$-sub-Gaussian noise. If $\beta_\ni^{1/2} = B + 4 \sigma \sqrt{ \Mi{\mb{y}_{\ni}}{f} + 1 + \mathrm{ln}(1 / \delta)}$, then for all~${\ab \in \adomain}$ and~${\ni \geq 0}$ it holds jointly with probability at least~${1 - \delta}$ that
$
\left|\, f'(\ab) - \mu_{\ni}(\ab) \,\right| \leq \beta_{\ni}^{1/2} \sigma_{\ni}(\ab).
$
\label{thm:confidence_interval}
\end{restatable}
\cref{thm:confidence_interval} implies that, with high probability, the true function~$f'$ is contained in the confidence intervals induced by the posterior Gaussian process distribution that uses the kernel~$k$ from~\cref{thm:confidence_interval} as a covariance function, scaled by an appropriate factor~$\beta_\ni$. In contrast to \cref{sec:back:gaussian_process}, \cref{thm:confidence_interval} does not make probabilistic assumptions on $f'$. In fact, $f'$ could be chosen adversarially, as long as it has bounded norm in the RKHS.

Since the frequentist confidence intervals depend on the mutual information and the marginal confidence intervals of the Gaussian process model, they inherit the same contraction properties up to the factor $\beta_\ni$. However, note that the confidence intervals in \cref{thm:confidence_interval} hold jointly through the continuous domain $\adomain$. This is not generally possible for Gaussian process models without employing additional continuity arguments, since Gaussian process distributions are by definitions only defined via a multivariate Normal distribution over \emph{finite} sets. This stems from the difference between a Bayesian belief and the frequentist perspective, where the function is unknown but fixed \textit{a priori}.

\subsection{Extension to multiple dimensions}

It is straight forward to extend these models to functions with vector-values outputs by extending the input domain by an extra input argument that indexes the output dimension. While this requires special kernels, they have been analyzed by \citet{Krause2011Contextual}.

\begin{restatable}[based on \citet{Chowdhury2017Kernelized}]{lemma}{confidencethmmulti}
Assume that~${f'(\btheta, i) = [f'(\btheta)]_i}$ has RKHS norm bounded by~$B$ and that measurements are corrupted by~$\sigma$-sub-Gaussian noise. Let $\adomain_\ni = \mathcal{D}_\ni \times \mathcal{I}$ denote the measurements obtained up to iteration $\ni$. If $\beta_\ni = B + 4 \sigma \sqrt{ \Mi{\mb{y}_{\adomain_\ni}}{f'} + 1 + \mathrm{ln}(1 / \delta)}$, then the following holds for all parameters~${\btheta \in \mathcal{D}}$, function indices~${i \in \mathcal{I}}$, and iterations~${n \geq 0}$ jointly with probability at least~${1 - \delta}$:
\begin{equation}
\big|\, f'(\btheta, i) - \mu_{n}(\btheta, i) \,\big| \leq \beta_{n} \sigma_{n}(\btheta, i)
\end{equation}
\label{thm:confidence_interval_multi}
\end{restatable}

%% file: appendix/gp_variance_lipschitz.tex
\section{Lipschitz Continuity of Gaussian Process Predictions}
\label{ap:gp_variance_lipschitz}
\label{ap:gp_predictions_lipschitz}

Since the mean function is a linear combination of kernels evaluations (features), it is easy to show that it is Lipschitz continuous if the kernel function is Lipschitz continuous \citep{Lederer2019Uniform}. However, existing bounds for the Lipschitz constant for the posterior standard deviation $\sigma_\ni(\cdot)$ depend on the number of data points. Since our regret bounds depend on $L_\sigma^\Ti$, this would render our regret bound superlinear and thus meaningless. 

In the following, we show that the GP standard deviation is Lipschitz-continuous with respect to the kernel metric.

\begin{definition}[Kernel metric] $d_k(\ab, \ab') = \sqrt{k(\ab, \ab) + k(\ab', \ab') - 2 k(\ab, \ab')}$.
\end{definition}

We start with the standard deviation.

\begin{lemma} For all $\ab$ and $\ab'$ in $\adomain$ and all $\ni \geq 0$, we have
\begin{equation}
    |\sigma_\ni(\ab) - \sigma_\ni(\ab')| \leq d_k(\ab, \ab')
\end{equation}
  \label{lem:gp_std_lipschitz}
\end{lemma}
\begin{proof}
  From Mercer's theorem we know that each kernel can be equivalently written in terms of an infinite-dimensional inner product, so that $k(\ab, \ab') = \langle k(\ab, \cdot), k(\ab', \cdot) \rangle_k$, where $<\cdot,\cdot>_k$ is the inner product in the Reproducing Kernel Hilbert Space corresponding to the kernel $k$. We can think of Gaussian process regression as linear regression based on these infinite-dimensional feature vectors. In particular, it follows from \citep[Appendix D]{Kirschner2018Information} that we can write the Gaussian process posterior standard deviation $\sigma_\ni(\ab)$ as the weighted norm of the infinite-dimensional feature vectors $k(\ab, \cdot)$,
  \begin{equation}
    \sigma_\ni(\ab) = \| k(\ab, \cdot) \|_{\mb{V}^{-1}_\ni},
  \end{equation}
  where $\mb{V}_\ni = \sigma^2 \mb{M}^* \mb{M} + \mb{I}$ and $\mb{M}$ is a linear operator that corresponds to the infinite-dimensional feature vectors $k(\ab_i, \cdot)$ of the data points $\ab_i$ in $\adomain_\ni$ so that $[\mb{M} \mb{M}^*]_{(i,j)} = k(\ab_i, \ab_j)$, where $\ab_i$ and $\ab_j$ are the $i$th and $j$th data point in $\adomain_\ni$. Now we have that the minimum eigenvalue of $\mb{V}_\ni$ is larger or equal than one, which implies that the maximum eigenvalue of $\mb{V}_\ni^{-1}$ is less or equal than one. Thus,
  \begin{align}
    | \sigma_\ni(\ab) - \sigma_\ni(\ab') |
    &= \big| \| k(\ab, \cdot)\|_{\mb{V}^{-1}_\ni} - \| k(\ab', \cdot)\|_{\mb{V}^{-1}_\ni} \big| \label{eq:lipschits_s1} \\
    &\leq \| k(\ab, \cdot) - k(\ab', \cdot) \|_{\mb{V}^{-1}_\ni}, \label{eq:lispchitz_s2} \\
    &\leq \| k(\ab, \cdot) - k(\ab', \cdot) \|_k, \\
    &= \sqrt{\langle k(\ab, \cdot) - k(\ab', \cdot), k(\ab, \cdot) - k(\ab', \cdot)\rangle_k}, \\
    &= \sqrt{k(\ab, \ab) - k(\ab, \ab') - k(\ab', \ab) + k(\ab', \ab')}, \\
    &= \sqrt{k(\ab, \ab) + k(\ab', \ab') - 2 k(\ab, \ab')}, \\
    &= d_k(\ab, \ab'),
  \end{align}
  where $\cref{eq:lipschits_s1} \to \cref{eq:lispchitz_s2}$ follows from the reverse triangle inequality.
\end{proof}

To show that \cref{lem:gp_std_lipschitz} implies Lipschitz continuity of the variance, the key observation is that standard deviation $\sigma_ni(\ab)$ is bounded for all $\ni \geq 0$. In particular, 
\begin{equation}
    \sigma_\ni(\ab) \leq \sigma_0(\ab) = \sqrt{k(\ab, \ab)} \leq  \max_{\ab, \ab' \in \R^d} \sqrt{k(\ab, \ab')} \eqdef \sqrt{|k|_\infty}     
\end{equation}

Based on this, we have the following result.

\begin{lemma}
For all $\ab$ and $\ab'$ in $\adomain$ and all $\ni \geq 0$, we have
\begin{align}
    |\sigma_\ni(\ab) - \sigma_\ni(\ab')| \leq 2 \sqrt{|k|_\infty} \, d_k(\ab, \ab')
\end{align}
  \label{lem:gp_var_lipschitz}
  \label{ap:lem:gp_variance_lipschitz}
\end{lemma}
\begin{proof}
For any compact domain $\mathcal{D}$ the function $f(x) = x^2$ is Lipschitz continuous for $\x \in \mathcal{D}$ with Lipschitz constant $|df/dx|_\infty = \max_{x \in \mathcal{D}} 2 |x|$. Since $0 \leq \sigma_\ni(\ab) \leq \sqrt{|k|_\infty}$, we have
\begin{align}
    | \sigma_\ni^2(\ab)- \sigma_\ni^2(\ab') |
    &\leq 2 \sqrt{|k|_\infty} \, \big| \sigma_\ni(\ab) - \sigma_\ni(\ab') \big| \\
    &\leq 2 \sqrt{|k|_\infty} \, d_k(\ab, \ab')
\end{align}
\end{proof}

%% file: appendix/gp_exploration_proofs.tex
\section{Regret Bound for Gaussian Process model}
\label{ap:sec:gp_regret_bounds}

\subsection{Assumptions about the model}
\label{sec:model:epistemic_uncertainty}

\begin{assumption}
  \label{as:kernel_metric_lipschitz}
  \label{as:kernel_lipschitz}
  The both the kernel and the kernel metric \cref{eq:kernel_metric} are Lipschitz continuous.
\end{assumption}

Note that the kernel metric is not trivially Lipschitz if the kernel is Lipschitz, since  the square root function has unbounded derivatives at zero. However, for many commonly used kernels, e.g., the linear and squared exponential kernels, the kernel metric is in fact Lipschitz continuous. 

As a direct consequence of \cref{as:kernel_metric_lipschitz} together with \cref{ap:gp_predictions_lipschitz} we know that $\bsigma(\cdot)$ is $L_\sigma$-Lipschitz continuous. 

\begin{assumption}
  The model $f$ has RKHS norm bounded by $B_f$ with respect to a kernel that fulfilles \cref{as:kernel_lipschitz,as:kernel_metric_lipschitz} and $k((\x, \pi(\x), (\x, \pi(\x)) \leq 1$ for all $\pi \in \Pi$ and $\x \in \X$.
  \label{as:reliable_statistical_model}
  \label{as:f_bounded_rkhs_norm}
  \label{as:kernel_bounded}
\end{assumption}

This assumption allows us to learn a calibrated model of the function $g$. Note that the assumption of a bounded kernel over a compact domain $\X$ is mild, since any scaling can be absorbed into the constant $B_f$. We weaken this assumption in \cref{ap:unbounded_domains}, where we bound the domain $\X$ rather than assuming compactness.

Since RKHS functions are linear combinations of the kernel function evaluated at representer points, the continuity assumptions on the kernel directly transfer to continuity assumptions on the function $f$, so that we get the following result. 

\begin{corollary}
  Under \cref{as:f_bounded_rkhs_norm}, the dynamics function $f$ is $L_f$-Lipschitz continuous with respect to the 2-norm.
  \label{cor:g_lipschitz}
\end{corollary}
\begin{proof}
  For scalar functions, this is a direct consequence of \cref{as:reliable_statistical_model} and \citep[Cor. 4.36]{Christmann2008Support}. This directly generalizes to the vector case.
\end{proof}

Since the state $\x$ is observed directly, the \cref{as:reliable_statistical_model} allows us to learn a reliable statistical model of $f$ that conforms with the requirement of a well-calibrated model in \cref{as:well_calibrated_model}. In particular, for each transition from $(\x_\ti, \u_\ti)$ to $\x_{\ti+1}$, we add $\nstate$ observations, one for each output dimension, to $\dataset_\ni$ as in \cref{thm:confidence_interval_multi}.

\begin{corollary}
  \label{cor:gp_model_reliable}
  Under \cref{as:reliable_statistical_model,as:transition_noise_sub_gaussian} with $\beta_\ni$ as in \cref{thm:confidence_interval_multi} and a Gaussian process model trained on observations $\mb{x}_{\ti+1}$ based on an input $\mb{a} = (\x_\ti, \u_\ti)$, the following holds with probability $1 - \delta$ for all $\ni \geq 0$, $\x \in \R^\nstate$, and $\u \in \R^\ninp$:
  \begin{equation}
    | f(\x, \u, i) - \mu_\ni(\x, \u, i) | \leq \beta_\ni \sigma_\ni(\x, \u, i)
    \label{eq:f_confidence_index_wise}
  \end{equation}
  \label{lem:f_confidence_intervals}
  \label{lem:f_confidence_interval}
\end{corollary}

In the following, we write
\begin{align}
\bmu_\ni(\x, \u) &= (\mu_{\ni\Ti\nstate}(\x, \u, 1), \dots, \u_{\ni\Ti\nstate}(\x, \u, \nstate)), 
\label{eq:stacked_gp_mean}
\\
\bsigma_\ni(\x, \u) &= (\sigma_{\ni\Ti\nstate}(\x, \u, 1), \dots, \sigma_{\ni\Ti\nstate}(\x, \u, \nstate))
\label{eq:stacked_gp_sigma}
\end{align}
to represent the individual elements as vectors. Note that $\bmu_\ni$ is conditioned on the $\ni\Ti\nstate$ individual one-dimensional observations after $\ni$ episodes. \cref{lem:f_confidence_intervals} allows us to build confidence intervals on the model error $g$ based on the scaled Gaussian process posterior variance. A direct consequence of these point-wise error bounds is that we can also bound the norm of the error on the vector-output of $f$.

\begin{corollary}
  Under the assumption of \cref{lem:f_confidence_intervals}, with probability $1 - \delta$ we have for all $\ni \geq 0$, $\x \in \R^\nstate$, and $\u \in \R^\ninp$ that
  \begin{equation}
    \| f(\x, \u) - h(\x, \u) - \bmu_\ni(\x, \u) \|_2 \leq \beta_\ni \| \sigma_\ni(\x, \u) \|_2
  \end{equation}
  \label{lem:f_confidence_norm}
\end{corollary}
\begin{proof}
  \begin{align}
    \| f(\x, \u) - \bmu_\ni(\x, \u) \|_2
    &= \left( \sum_{i=1}^\nstate |f(\x, \u, i) - \mu_\ni(\x, \u, i)|^2 \right)^{1/2} \\
    &\leq \left( \sum_{i=1}^\nstate | \beta_\ni \sigma_\ni(\x, \u, i) |^2 \right)^{1/2} 
    = \beta_\ni \| \bsigma_\ni(\x, \u) \|_2
  \end{align}
\end{proof}

\ActivateWarningFilters[pdftoc]
\subsection{Bounding $I_\Ni$ for the GP model}
\DeactivateWarningFilters[pdftoc]

In this section, we bound $I_\Ni$ based on the GP assumptions. This allows us to use them together with \cref{thm:exploration:regret:general_regret_bound_increasing_sets} to obtain regret bounds. We start with some preliminary lemmas

\begin{lemma}[\citet{Srinivas2012Gaussian}]
  \label{lem:log_smax_trick}
  $s^2 \leq \frac{s_\mathrm{max}^2}{\log(1 + s_\mathrm{max}^2)} \log(1 + s^2)$ for all $s \in [0, s_\mathrm{max}^2]$
\end{lemma}

\begin{lemma}
  \label{lem:exploration:regret:bound_sigma_by_log_sigma}
  Let $ | \sigma_\ni(\cdot) | \leq \sigma_\mathrm{max}$ and $\sigma > 0$. Then
    \begin{equation}
      \sigma_\ni^2(\ab) \leq \frac{\sigma_\mathrm{max}}{\log(1 + \sigma^{-2} \sigma_\mathrm{max})} \log(1 + \sigma^{-2} \sigma_\ni^2(\ab))
    \end{equation}
\end{lemma}
\begin{proof}
  \begin{align}
    \sigma_\ni^2(\ab) &\leq \sigma^2 (\sigma^{-2} \sigma_\ni^2(\ab)) \\
    \intertext{Now $\sigma^{-2} \sigma_\ni^2(\ab) \leq \sigma^{-2} \sigma_\mathrm{max}$ by assumption. Thus, we can use \cref{lem:log_smax_trick} to obtain}
    \sigma_\ni^2(\ab) &\leq \sigma^2 \frac{\sigma^{-2} \sigma_\mathrm{max}}{\log(1 + \sigma^{-2} \sigma_\mathrm{max})} \log(1 + \sigma^{-2} \sigma_\ni^2(\ab)) \\
    &= \frac{\sigma_\mathrm{max}}{\log(1 + \sigma^{-2} \sigma_\mathrm{max})} \log(1 + \sigma^{-2} \sigma_\ni^2(\ab))
  \end{align}
\end{proof}

\begin{lemma}
  \label{lem:exploration:regret:bound_variance_by_mutual_information}
  Let $\dataset_{1:\Ni}$ denote the $\Ni \Ti$ $\nstate$-dimensional observations collected up to iteration $\ni$ and $\mb{y}_{\dataset_{1:\ni}}$ the corresponding observations of the following states. Then
  \begin{equation}
    \frac{1}{2} \sum_{\ni=1}^\Ni \sum_{\ti=0}^{\Ti-1} \sum_{j=1}^\nstate \log(1 + \sigma^{-2} \sigma_{(\ni-1)\Ti\nstate}^2(\mb{x}_{\ti, \ni}, j)) 
  \leq \Ti \nstate \Mi{\mb{y}_{\dataset_\Ni}}{ f_{\dataset_\Ni}}
  \end{equation}
\end{lemma}
\begin{proof}
  \begin{align}
    &\hphantom{\leq} \frac{1}{2} \sum_{\ni=1}^\Ni \sum_{\ti=0}^{\Ti-1} \sum_{j=1}^\nstate \log(1 + \sigma^{-2} \sigma_{(\ni-1)\Ti\nstate}^2(\mb{x}_{\ti, \ni}, j)) \\
    &= \sum_{\ti=0}^{\Ti-1} \sum_{j=1}^\nstate \frac{1}{2} \sum_{\ni=1}^\Ni \log(1 + \sigma^{-2} \sigma_{(\ni-1)\Ti\nstate}^2(\mb{x}_{\ti, \ni}, j)) \\
    &\leq \Ti \nstate \Mi{\mb{y}_{\dataset_{1:\Ni}}}{ f_{\dataset_{1:\Ni}}}
  \end{align}
  Where the second to last step follows from \cite[Lemma 2]{Srinivas2012Gaussian} together with $\log(1 + x) \geq 0$ for $x \geq 0$ and the properties of the mutual information. In particular, the inner sum conditions on~$(\ni-1) \Ti \nstate$ measurements, but sums only over the one element $(\mb{x}_{\ti,\ni}, j)$. The mutual information in \cite[Lemma 2]{Srinivas2012Gaussian} instead sums over every element that we condition on in the next step. By adding the missing non-negative terms together with the fact that the mutual information is independent of the order of the observations we obtain the result.
  Another way to interpret this bound is that, in the worst case, we could hypothetically visit $\Ti$ times the same state during a trajectory and obtain the corresponding $\nstate$-dimensional observation. This explains the $\Ti\nstate$ factor that multiplies the mutual information.
\end{proof}

We can use these two lemmas to obtain:

\begin{lemma}
  \label{lem:gp_model_complexity_bound}
  For a GP model let $ | \sigma_\ni(\cdot) | \leq \sigma_\mathrm{max}$ and $\sigma > 0$. Then
  \begin{equation}
      I_\Ni(\X, \U) \leq \frac{\sigma_\mathrm{max} \Ti \nstate }{\log(1 + \sigma^{-2} \sigma_\mathrm{max})} \gamma_{\Ni \Ti \nstate}(\X \times \U \times \mathcal{I}_\nstate)  
  \end{equation}
\end{lemma}
\begin{proof}
\begin{align}
  I_\Ni(\X, \U) &= \max_{\dataset_1, \dots, \dataset_\Ni \subset \X \times \X \times \U, \, |\dataset_\ni| = \Ti} 
  \sum_{\ni=1}^\Ni \sum_{\x, \u \in \dataset_\ni} \| \bsigma_{\ni-1}(\x, \u) \|_2^2 \\
  &= \max_{\dataset_1, \dots, \dataset_\Ni \subset \X \times \X \times \U, \, |\dataset_\ni| = \Ti} 
  \sum_{\ni=1}^\Ni \sum_{\x, \u \in \dataset_\ni} \sum_{j=1}^\nstate \sigma_{(\ni-1)\Ti\nstate}^2(\x, \u, j) \\
  &\leq \frac{\sigma_\mathrm{max}}{\log(1 + \sigma^{-2} \sigma_\mathrm{max})} \max_{\dataset_1, \dots, \dataset_\Ni \subset \X \times \X \times \U, \, |\dataset_\ni| = \Ti} 
  \sum_{\ni=1}^\Ni \sum_{\x, \u \in \dataset_\ni} \sum_{j=1}^\nstate \log(1 + \sigma^{-2} \sigma_{(\ni-1)\Ti\nstate}^2(\x, \u, j) ) \\
  &\leq \frac{\sigma_\mathrm{max}}{\log(1 + \sigma^{-2} \sigma_\mathrm{max})} \max_{\dataset_1, \dots, \dataset_\Ni \subset \X \times \X \times \U, \, |\dataset_\ni| = \Ti} \Ti \nstate \Mi{\mb{y}_{\dataset_{1:\Ni}}}{ f_{\dataset_{1:\Ni}}} \\
  &\leq \frac{\sigma_\mathrm{max} \Ti \nstate }{\log(1 + \sigma^{-2} \sigma_\mathrm{max})} \gamma_{\Ni \Ti \nstate}(\X \times \U \times \mathcal{I}_\nstate) 
\end{align}
\end{proof}

To obtain an instance-independent bound, we must bound the mutual information by the worst-case mutual information as in \citep{Srinivas2012Gaussian}.

\begin{theorem}
    \label{thm:gp_regret_bound}
  Under \cref{as:dynamics_f_lipschitz,as:pi_lipschitz,as:reward_lipschitz,as:well_calibrated_model,as:model_predictions_lipschitz} let $\x_{\ti,\ni} \in \mathcal{X}_\ni$, $\X_{\ni-1} \subseteq \X_\ni$, and $\u_{\ti,\ni} \in \mathcal{U}$ for all $\ti,\ni>0$ with compact sets $\X_\ni$ and $\U$. Let $\|\bsigma(\cdot) \|_\infty \leq \sigma_\mathrm{max}$. At each iteration, select parameters according to \cref{eq:optimistic_exploration}. Then the following holds with probability at least $(1-\delta)$ for all $\ni \geq 1$
  \begin{equation}
    R_\Ni \leq \Or{\beta_{\Ni-1}^\Ti L_\sigma^\Ti \Ti^2  \sqrt{ \Ni \, \nstate \, \gamma_{\nstate \Ni \Ti}(\X_\Ni \times \U \times \mathcal{I}_\nstate) }  },
  \end{equation}
  where $\gamma_{\nstate \Ni \Ti}(\X \times \U \times \mathcal{I}_\nstate)$ is the information capacity after $(\nstate \ni \Ti)$ observations within the extended domain $\X \times \U \times \mathcal{I}_\nstate$.
\end{theorem}
\begin{proof}

  From \cref{thm:exploration:regret:general_regret_bound_increasing_sets} we have $R_\Ni^2 \leq \Ni  L_J^2 \Ti^3 I_\Ni(\X_\ni, \U) $. Together with \cref{lem:gp_model_complexity_bound} we obtain
  \begin{align} 
    R_\Ni &\leq L_J \sqrt{ \Ti^3 I_\Ni(\X_\ni, \U) }
    \\
    &\leq L_J \left( \frac{\sigma_\mathrm{max} \Ti^4 \nstate }{\log(1 + \sigma^{-2} \sigma_\mathrm{max})} \gamma_{\Ni \Ti \nstate}(\X \times \U \times \mathcal{I}_\nstate) \right)^{1/2}
  \end{align}
  where $L_J = 2 L_r (1 + L_\pi) \beta_{\ni-1} \bar{L}_f^{\Ti-1}$ from \cref{lem:exploration:regret:bound_squared_regret_by_expected_variance} and $\bar{L}_f = 1 + L_f + 2 \beta_{\ni - 1} L_\sigma \sqrt{1 + L_\pi}$ from \cref{lem:exploration:regret:bounded_state_difference_with_same_noise_noconstraint}. Plugging in we get $L_J = 2 L_r (1 + L_\pi) \beta_{\ni-1} (1 + L_f + 2 \beta_{\ni - 1} L_\sigma  \sqrt{1 + L_\pi})^{\Ti-1} = \Or{\beta_{\ni - 1}^\Ti L_\sigma^\Ti}$ so that
  \begin{align}
    R_\Ni \leq \Or{L_\sigma^\Ti \beta_{\Ni-1}^\Ti \Ti^2  \sqrt{ \Ni\nstate \gamma_{\nstate \Ni \Ti}(\X_\ni \times \U \times \mathcal{I}_\nstate) }  }
  \end{align}
\end{proof}

Notably, unlike in \cref{thm:exploration:regret:general_regret_bound} we can actually bound the information capacity $\gamma$ in \cref{thm:gp_regret_bound}. For a GP model that uses a squared exponential kernel with independent outputs, we have $\gamma_{\nstate \Ni \Ti} \leq \Or{ \nstate (\nstate + \ninp) \log(\nstate \Ni \Ti) }$ by \citep{Srinivas2012Gaussian,Krause2011Contextual}, which renders the overall regret bound sublinear. Note that for the Matern kernel the best known bound on $\gamma_{\nstate \Ni \Ti}$ is $\Or{ \nstate (\nstate \Ni \Ti)^c \log(\nstate \Ni \Ti)}$ with $0 < c < 1$. This means the regret bound is not sublinear for long trajectories due to the $\beta_\ni^\Ti$ term in the regret bound. However, the bound is expected to be loose \citep{Scarlett2017LowerBound}. Tighter bounds can be computed numerically, see \cite[Fig. 3]{Srinivas2012Gaussian}.

Note that the requirement $\| \bsigma(\cdot)\|_\infty$ if fulfilled according to

\begin{lemma}
    \label{lem:gp_variance_bounded_if_kernel_bounded}
    Under \cref{as:kernel_bounded} we have $\bsigma(\x) \leq 1$ for all $\x \in \X$.
\end{lemma}
\begin{proof}
    This is a direct consequence of \cref{eq:gp_prediction_variance}.
\end{proof}

\ActivateWarningFilters[pdftoc]
\subsection{Comparison to \citet{Chowdhury2019Online}}
\DeactivateWarningFilters[pdftoc]
\label{ap:sec:comparison_to_chowdhury2019}

In this section, we compare our bound to the one by \citet{Chowdhury2019Online}. This is a difficult endeavour, because they make fundamentally different assumptions. In particular, they assume that the value function $v(x)$ is $L_M$-Lipschitz continuous, which hides all the complexity of thinking about different trajectories, as deviations between the two trajectories can be bounded after one step by $L_M \| \x_1 - \xo_1 \|$. In contrast, we do not make this high-level assumption and specifically reason about the entire trajectories based on system properties. Note, that the constant $L_M$ is at least $\Om{\Ti}$ without additional assumptions about the system and generally will depend on the statistical model (GP).

Secondly, they restrict the optimization over dynamics that are Lipschitz continuous, which means their algorithm depends on system properties that are difficult to estimate in general. However, this assumption avoids the dependency $\beta^\Ti$ in our regret bound, since it limits optimization to trajectories that are at most as smooth as the dynamics of the true system. The cost of this is that their algorithm is not tractable to implement or compute.

For completeness, in the following we modify our proof to use their assumption and show a regret bound that is comparable to the one by \citet{Chowdhury2019Online}.

\ActivateWarningFilters[pdftoc]
\subsubsection{Our bound under the assumptions of \citep{Chowdhury2019Online}}
\DeactivateWarningFilters[pdftoc]

Now, we show that if we assume that the optimistic dynamics are Lipschitz, which together with a Lipschitz-continuous policy implies the Lipschitz continuity of the value function that is assumed by \citet{Chowdhury2019Online}, we obtain the same regret bounds. 

Let 
\begin{equation*}
\begin{aligned}
  \widetilde{\modelclass}_\ni = \big\{ f' \mid | {}&\bmu(\x, \u) - f'(\x, \u) | \leq \beta  \bsigma(\x, \u)   \, \forall \x, \u \in \R^\nstate \times \R^\ninp,  \\
  {}&\|f'(\x, \u) - f'(\x', \u') \| \leq L_f \| (\x, \u) - (\x', \u') \|  \, \forall (\x, \u), (\x', \u') \in \R^\nstate \times \R^\ninp,
  \big\}
\end{aligned}
\end{equation*}
be the set of all Lipschitz continuous dynamics that are compatible with the uncertainty representation in \cref{as:well_calibrated_model}. We now consider a variant of \cref{eq:optimistic_exploration} that optimizes over dynamics in this set,
\begin{align}
   \label{eq:optimistic_exploration_intractable_lipschitz}
    \pi_\ni = \argmax_{\pi \in \Pi, \, \tilde{f}_\ni \in \widetilde{\modelclass}_\ni} J(\tilde{f}_\ni, \pi)
\end{align}
and we implicitly define $\xo$ and $\uo$ based on $\tilde{f}_\ni$ in \cref{eq:optimistic_exploration_intractable_lipschitz} for the remainder of this section, instead of the global definition from \cref{eq:xo_x_definitions_for_proofs}.
Note that this optimization is not tractable in the noisy case. 



For the exploration scheme in \cref{eq:optimistic_exploration_intractable_lipschitz} we have the following results that lead to improved regret bounds that match those in \citep{Chowdhury2019Online} up to constant factors.

\begin{lemma}
  \label{lem:exploration:regret:bounded_state_difference_with_same_noise}
  Under the assumptions of \cref{lem:f_confidence_intervals}, let $\bar{L}_f = L_f$. Then, for any sequence of~$\bm{\eta}_\ti \in [-1, 1]^\nstate$, any sequence of~$\noise_\ti$ with $\tilde{\noise}_\ti = \noise_\ti$, $\btheta \in \mathcal{D}$, and $\ti \geq 1$ we have that
  \begin{equation}
    \| \x_{\ti,\ni} - \xo_{\ti,\ni} \| \leq 2 \beta_{\ni-1} \bar{L}_f^{\Ti-1} \sum_{i=0}^{\ti-1} \| \bsigma_{\ni-1}(\x_{i,\ni}) \| 
  \end{equation}
\end{lemma}
\begin{proof}
  Let 
  \begin{equation}
    \tilde{f}(\xo_{\ti, \ni}) = \bmu_{\ni-1}(\xo_\ti) + \beta_{\ni-1} \bSigma_{\ni-1}(\xo_\ti) \bm{\eta}_{\ti}.
  \end{equation}
  Then by design we have $\|\tilde{f}(\x) - \tilde{f}(\x') \| \leq L_f \| \x - \x' \|$. 

  We start by showing that, for any $\ti \geq 1$, we have 
  \begin{equation}
    \| \x_{\ti,\ni} - \xo_{\ti,\ni} \| \leq 2 \beta_{\ni-1} \sum_{i=0}^{\ti-1} L_f^{\ti-1-i} \| \bsigma_{\ni-1}(\x_{i,\ni}) \| 
    \label{eq:stoch_sample_confidence}
  \end{equation}
  by induction.
  
  For the base case we have~$\xo_0 = \x_0$. Consequently, at $\ni$ we have
  \begin{align}
    \| \x_{1,\ni} - \xo_{1,\ni}\|
    &= \| f(\x_0)  + \noise_0 - \tilde{f}(\x_0) - \tilde{\noise}_0 \| \\
    &= \| f(\x_0) - \tilde{f}(\x_0)  \| \\
    &= \| f(\x_0) - \bmu_{\ni-1}(\x_0) - \beta_{\ni-1} \bSigma_{\ni-1}(\x_0) \bm{\eta}_0  \| \\
    &\leq \|  f(\x_0) - \bmu_{\ni-1}(\x_0) \| + \beta_{\ni-1}  \| \bsigma_{\ni-1}(\x_0) \bm{\eta}_0 \| \\
    &\leq \beta_{\ni-1} \| \bsigma_{\ni-1}(\x_0) \| + \beta_{\ni-1} \| \bsigma_{\ni-1}(\x_0) \| \\
    &= 2 \beta_{\ni-1} \| \bsigma_{\ni-1}(\x_0) \|
  \end{align}

  For the induction step assume that \cref{eq:stoch_sample_confidence} holds at time step~$\ti$. Subsequently we have at iteration $\ni$ that
  \begin{align*}
    \| \x_{\ti+1, \ni} - \xo_{\ti+1, \ni} \| &= \| f(\x_\ti) - \tilde{f}(\xo_\ti) \| \\
    &= \| f(\x_\ti) - \tilde{f}(\x_\ti) + \tilde{f}(\x_\ti) - \tilde{f}(\xo_\ti) \| \\
    &= \| f(\x_\ti) - \tilde{f}(\x_\ti) \|+ \| \tilde{f}(\x_\ti) - \tilde{f}(\xo_\ti) \| \\
    &\leq 2 \beta_{\ni-1} \| \bsigma_{\ni-1}(\x_\ti) \| + L_f \| \x_\ti - \xo_\ti \| \\
    &\leq 2 \beta_{\ni-1} \| \bsigma_{\ni-1}(\x_\ti) \| + L_f 2 \beta_{\ni-1} \sum_{i=0}^{\ti-1} L_f^{\ti-1-i} \| \bsigma_{\ni-1}(\x_{i,\ni}) \| \\
    &= 2 \beta_{\ni-1} \| \bsigma_{\ni-1}(\x_\ti) \| + 2 \beta_{\ni-1} \sum_{i=0}^{\ti-1} L_f^{\ti-1-i+1} \| \bsigma_{\ni-1}(\x_{i,\ni}) \| \\
    &= 2 \beta_{\ni-1} \sum_{i=0}^{(\ti + 1) - 1} L_f^{(\ti + 1)-1-i+1} \| \bsigma_{\ni-1}(\x_{i,\ni}) \|  \\
    &= 2 \beta_{\ni-1} \sum_{i=0}^{(\ti + 1) - 1} L_f^{(\ti + 1)-i} \| \bsigma_{\ni-1}(\x_{i,\ni}) \| 
  \end{align*}

  Thus \cref{eq:stoch_sample_confidence} holds. Now since $\ti \leq \Ti$ we have
  \begin{align}
    \| \x_{\ti+1, \ni} - \xo_{\ti+1, \ni} \|
    &\leq 2 \beta_{\ni-1} \sum_{i=0}^{\ti-1} L_f^{\ti-1-i} \| \bsigma_{\ni-1}(\x_{i,\ni}) \|
    &\leq 2 \beta_{\ni-1} L_f^{\Ti-1} \sum_{i=0}^{\ti-1} \| \bsigma_{\ni-1}(\x_{i,\ni}) \|
  \end{align} 
\end{proof}

\begin{theorem}
  \label{thm:general_regret_bound_intractable}
   Under \cref{as:well_calibrated_model,as:dynamics_f_lipschitz,as:pi_lipschitz,as:reward_lipschitz,as:model_predictions_lipschitz} let $\x_{\ti,\ni} \in \mathcal{X}_\ni$, $\X_{\ni-1} \subseteq \X_\ni$, and $\u_{\ti,\ni} \in \mathcal{U}$ for all $\ti,\ni>0$ with compact sets $\X_\ni$ and $\U$. Let $\|\bsigma(\cdot) \|_\infty \leq \sigma_\mathrm{max}$. At each iteration, select parameters according to \cref{eq:optimistic_exploration_intractable_lipschitz}. Then the following holds with probability at least $(1-\delta)$ for all $\ni \geq 1$
  \begin{equation}
    R_\Ni \leq \Or{ L_f^\Ti \Ti^2  \sqrt{ \Ni \, \nstate \, \gamma_{\nstate \Ni \Ti}(\X_\Ni \times \U) }  },
  \end{equation}
  where $\gamma_{\nstate \Ni \Ti}(\X \times \U)$ is the information capacity after $(\nstate \ni \Ti)$ observations within the domain $\X \times \U$.
\end{theorem}

Thus, our proof strategy also avoids the scaling $\beta^\Ti$ when we assume that optimizing over dynamics in $\mathcal{M}$ is tractable. Thus, the factor $\beta^\Ti$ is the cost that we pay for not being able to do so.

%% file: appendix/extension_unbounded_domain.tex

\section{Extension to Unbounded Domains}
\label{ap:unbounded_domains}

So far, we have assumed a compact domain $\X$. This is incompatible with the dynamic system in \cref{eq:stochastic_dynamic_system_additive}, since sub-Gaussian noise includes noise distributions with unbounded support. In this section, we show that we can bound the domain with high probability and that we can use continuity arguments to extend our previous theorem to this more general settings. This also avoids the implicit assumption that the dynamics function is bounded, which is not even true for linear systems.

\subsection{Bound on Aleatoric Uncertainty (Noise Bound)}
\label{ap:model:noise_bound}

We start by bounding the norm of the noise vector $\noise_\ti$ over all time steps $\ti$.

We know that the $\noise_\ti$ are i.i.d. sub-Gaussian vectors. We exploit the basic properties of sub-Gaussian random variables and refer to \citet[Chapter 5]{Vershynin2012Compressed} for a concise review.

\begin{lemma}{\citet[Corollary 5.17]{Vershynin2010Introduction}}
  Let $X_1, \dots, X_\nstate$ be independent centered sub-exponential random variables, and let $2\sigma = \max_i \| X_i \|_{\phi_1}$ be the largest, sub-exponential norm. Then, for every $\epsilon \geq 0$, we have 
  \begin{equation}
    \mathbb{P} \left\{ \left| \sum_{i=1}^\Ni X_i \right| \geq \epsilon \nstate \right\} \leq 2 \mathrm{exp} \left[ \frac{-\mathrm{e} \Ni}{2} \min \left( \frac{\epsilon^2}{4\sigma^2}, \frac{\epsilon}{2\sigma} \right) \right]
  \end{equation}
  \label{lem:sub_exponential_concentration}
\end{lemma}

This allows us to bound the 2-norm of the noise vectors in \cref{eq:stochastic_dynamic_system_additive}.

\begin{lemma}
  Let $\noise = (\snoise_1, \dots, \snoise_\nstate)$ be a vector with i.i.d. elements $[\noise]_i = \snoise_i$ that are $\sigma$-sub-Gaussian. Then, with probability at least $1 - \delta$, we have that
  \begin{equation}
    \| \noise \|_2^2 \leq 2\sigma \nstate + \frac{4 \sigma}{\mathrm{e}} \log \frac{2}{\delta}
  \end{equation}
  \label{lem:concentration_noise_2norm}
\end{lemma}
\begin{proof}
  Since the $\snoise_i$ are $\sigma$-sub-Gaussian, we have the $\snoise_i^2$ are $2\sigma$-sub-exponential \cite[Lemma 5.14]{Vershynin2010Introduction}. Thus we have
  \begin{equation*}
    \| \noise \|_2^2 = \sum_{i=1}^\nstate \snoise_i^2,
  \end{equation*}
  where the $\snoise_i^2$ are i.i.d. $2\sigma$-sub-exponential. Following \cref{lem:sub_exponential_concentration}, we have
  \begin{equation}
    \mathbb{P} \left\{ \| \noise \|_2^2 \geq \epsilon \nstate \right\} \leq 2 \mathrm{exp} \left[ \frac{-\mathrm{e} \nstate}{2} \min \left( \frac{\epsilon^2}{4 \sigma^2}, \frac{\epsilon}{2 \sigma} \right) \right]
  \end{equation}
  Now for $\epsilon \geq 2 \sigma$ we have $\epsilon^2/(4 \sigma^2) \geq \epsilon / (2 \sigma)$. Thus
  \begin{equation}
    \mathbb{P} \left\{ \| \noise \|_2^2 \geq (2\sigma + \epsilon) \nstate \right\} \leq 2 \mathrm{exp} \left[ \frac{-\mathrm{e} \nstate}{2} \frac{(2\sigma + \epsilon)}{2 \sigma}  \right]
    \leq 2 \mathrm{exp} \left[ \frac{-\mathrm{e} \nstate}{2} \frac{\epsilon}{2 \sigma}  \right] \label{eq:ap:lem_intermediary_noise_tail_bound}
  \end{equation}
  We want to upper bound the right hand side by $\delta$. so
  \begin{align}
    2 \exp\left[ \frac{-\mathrm{e}\nstate\epsilon}{4\sigma} \right] &\leq \delta, \\
    \frac{-\mathrm{e}\nstate\epsilon}{4\sigma} &\leq \log(\delta / 2), \\
    \frac{\mathrm{e}\nstate\epsilon}{4\sigma} &\geq \log(2 / \delta), \\
    \epsilon &\geq \frac{4\sigma}{\mathrm{e}\nstate} \log(2 / \delta).
  \end{align}
  the result follows by plugging the bound for $\epsilon$ into \cref{eq:ap:lem_intermediary_noise_tail_bound},
  \begin{align}
    (2\sigma + \epsilon)\nstate &= (2 \sigma + \frac{4\sigma}{\mathrm{e}\nstate} \log(2 / \delta)) \nstate \\
    &= 2\sigma \nstate + \frac{4 \sigma}{\mathrm{e}} \log \frac{2}{\delta}
  \end{align}
\end{proof}

As the last step, we apply the union bound to obtain confidence intervals over multiple steps.

\begin{restatable}{lemma}{jointtwonormnoisebound}
  Let $\noise_0, \noise_1, \dots$ be \iid random vectors with $\noise_\ti \in \R^\nstate$ such that each entry of the vector is \iid $\sigma$-sub-Gaussian. Then, with probability at least $(1-\delta)$,
  \begin{equation}
    \| \noise_\ti \|_2^2 \leq 2\sigma \nstate + \frac{4 \sigma}{\mathrm{e}} \log \frac{(\ti+1)^2 \pi^2}{3\delta}
    \label{eq:model:uniform_noise_bound}
  \end{equation}
  holds jointly for all $\ti \geq 0$. 
  \label{lem:time_step_uniform_noise_bound}
\end{restatable}
\begin{proof}
  At each time step $\ti$, we apply a probability budget of $\delta / \pi_\ti$ to the bound in \cref{lem:concentration_noise_2norm}, where $\pi_\ti \geq 0$ and $\sum_{\ti \geq 0} \pi_\ti^{-1} = 1$. In particular, we use $\pi_\ti = \frac{(\ti+1)^2 \pi^2}{6}$ as in \cite[Lemma 5.1]{Srinivas2012Gaussian}, so that we apply monotonically decreasing probability thresholds as $\ti$ increases. We obtain the result by applying a union bound over $\ti$, since $\sum_{\ti \geq 0} \delta / \pi_\ti = \delta$.
\end{proof}

This means that, for all time steps $\ti$, the noise is bounded within the hyper-sphere defined through \cref{eq:model:uniform_noise_bound} with high probability. In particular, the joint confidence intervals only come at the cost of a $\Or{\log \ti^2}$ increase in the confidence intervals over time.

\subsection{Bounding the Domain Under Aleatoric Uncertainty}

We exploit the $\sigma$-sub-Gaussian property of the transition noise and build on \cref{lem:sub_exponential_concentration,lem:concentration_noise_2norm} to obtain a bound over the domain. We start by applying a union bound on \cref{lem:concentration_noise_2norm} over the time horizon $\Ti$.

\begin{lemma}
  \label{lem:bound_noise_sum_per_iter}
  Let $\noise_0, \dots, \noise_{\Ti-1}$ be vectors with $\noise_i \in \R^\nstate$ such that each entry of the vector is i.i.d. $\sigma$-sub-Gaussian. Then, with probability at least $(1-\delta)$,
  \begin{equation}
    \sum_{\ti=0}^{\Ti-1} \| \noise_i \|_2 \leq \Ti \sqrt{ 2 \sigma \nstate + \frac{4 \sigma}{\mathrm{e}} \log \frac{2 T}{\delta} } 
  \end{equation}
\end{lemma}
\begin{proof}
  Now using \cref{lem:concentration_noise_2norm} with probability threshold $\delta / T$ and applying the union bound we, get that
  $\| \noise_i \|_2^2 \leq 2 \sigma \nstate + \frac{4 \sigma}{\mathrm{e}} \log \frac{2 T}{\delta}$ holds for all $0 \leq i \leq \Ti - 1$ with probability at least $1 - \delta$. 

  Now, first using Jensen's inequality and then plugging in the bound for $\| \noise_i \|_2^2$, we obtain
  \begin{align}
    \sum_{\ti=1}^\Ti \| \noise_\ti \|_2 &= \sum_{i=0}^{\Ti-1} \sqrt{ \| \noise_\ti \|_2^2} \\
    &\leq \sqrt{T} \sqrt{ \sum_{\ti=0}^{\Ti-1} \| \noise_\ti \|_2^2 } \\
    &\leq \sqrt{T} \sqrt{ \sum_{\ti=0}^{\Ti-1} \left( 2 \sigma\nstate+ \frac{4 \sigma}{\mathrm{e}} \log \frac{2 T}{\delta} \right) } \\
    &= \Ti \sqrt{ 2 \sigma\nstate+ \frac{4 \sigma}{\mathrm{e}} \log \frac{2 T}{\delta} } 
  \end{align}
\end{proof}

Lastly, we use a union bound over all iterations similar to \cite[Lemma 5.1]{Srinivas2012Gaussian}.

\begin{lemma}
  \label{lem:noise_2norm_bound_all_iter}
  Let $\noise_{t,n}$ be the random vectors as in \cref{lem:bound_noise_sum_per_iter} at iteration $n$. Then, with probability $(1 - \delta)$ we have for all $n \geq 1$ that
  \begin{equation}
    \sum_{t=1}^\Ti \| \noise_{\ti,\ni} \|_2 \leq \Ti \sqrt{ 2 \sigma\nstate+ \frac{4 \sigma}{\mathrm{e}} \log \frac{\Ti \pi^2 \ni^2}{3 \delta} } 
  \end{equation}
\end{lemma}
\begin{proof}
  At each iteration $n$, we apply a probability budget of $\delta / \rho_\ni$ to the bound in \cref{lem:bound_noise_sum_per_iter}, where $\rho_\ni \geq 0$ and $\sum_{\ni \geq 1} \rho_\ni^{-1} = 1$. In particular, we use $\rho_\ni = \frac{\ni^2 \pi^2}{6}$ as in \cite[Lemma 5.1]{Srinivas2012Gaussian}, so that we apply monotonically decreasing probability thresholds as $\ni$ increases. We obtain the result by applying a union bound over $\ni$, since $\sum_{\ni \geq 1} \delta / \rho_\ni = \delta$.
\end{proof}

Now that we can bound the noise over all iterations, we can bound the domain over which the system acts with a compact set.

\begin{lemma}
  \label{lem:states_compact_wrt_noise_norm}
  Let $f$ be $L_f$-Lipschitz continuous with respect to the norm $\| \cdot \|$. Then we have for all $\ti \geq 1$ that
  \begin{align}
    \| \x_\ti - \x_0 \| &\leq \sum_{i=0}^{\ti-1} L_\mathrm{fc}^i \| f(\x_0) - \x_0 \| + \sum_{i=0}^{\ti-1} L_\mathrm{fc}^{\ti-1-i} \| \noise_i \| 
    \label{eq:bound_state_divergence_long} \\
    &\leq (1+L_\mathrm{fc})^{\ti-1} \left( \ti \| f(\x_0) - \x_0 \| + \sum_{i=0}^{\ti-1} \| \noise_i \| \right)  
    \label{eq:bound_state_divergence}
  \end{align}
\end{lemma}
\begin{proof}
  We first proof \cref{eq:bound_state_divergence_long} by induction. For the base case we have
  \begin{align}
    \| \x_1 - \x_0 \| &= \| f(\x_0) + \noise_0 - \x_0 \| \\
    &\leq \| f(\x_0) - \x_0 \| + \| \noise_0 \|, \\
    &= L_\mathrm{fc}^0 \| f(\x_0) - \x_0 \| + L_\mathrm{fc}^0 \| \noise_0 \|.
  \end{align}
  For the induction step, assume that the assumption holds for some $\ti$. Then,
  \begin{align}
    \| \x_{t+1} - \x_0 \| ={}& \| f(\x_\ti) + \noise_\ti - \x_0 \| \\
    ={}& \| f(\x_\ti) - f(\x_0) + f(\x_0) - \x_0 + \noise_\ti \| \\
    \leq{}& \| f(\x_\ti) - f(\x_0) \| + \| f(\x_0) - \x_0 \| + \| \noise_\ti \| \\
    \leq{}& L_\mathrm{fc} \| \x_\ti - \x_0 \| + \| f(\x_0) - \x_0 \| + \| \noise_\ti \| \\
    \leq{}& L_\mathrm{fc} \left( \sum_{i=0}^{\ti-1} L_\mathrm{fc}^i \| f(\x_0) - \x_0 \| + \sum_{i=0}^{\ti-1} L_\mathrm{fc}^{\ti-1-i} \| \noise_i \| \right) \\ &+ \| f(\x_0) - \x_0 \| + \| \noise_\ti \| \\
    ={}& \sum_{i=1}^{(t-1)+1} L_\mathrm{fc}^i \| f(\x_0) - \x_0 \| + \| f(\x_0) - \x_0 \| \notag \\ 
    &+ \sum_{i=0}^{\ti-1} L_\mathrm{fc}^{(t+1)-1-i} \| \noise_i \| + \| \noise_\ti \| \\
    ={}& \sum_{i=0}^{(t-1)+1} L_\mathrm{fc}^i \| f(\x_0) - \x_0 \| + \sum_{i=0}^{(t+1)-1} L_\mathrm{fc}^{(t+1)-1-i} \| \noise_i \|
  \end{align}
  Which concludes the proof. For \cref{eq:bound_state_divergence}, note that $L_\mathrm{fc}^i \leq (1 + L_\mathrm{fc})^t$ for all $i \leq t$. Thus we have
  \begin{align}
    & \sum_{i=0}^{\ti-1} L_\mathrm{fc}^i \| f(\x_0) - \x_0 \| + \sum_{i=0}^{\ti-1} L_\mathrm{fc}^{\ti-1-i} \| \noise_i \| \\
    \leq{}& L_\mathrm{fc}^{\ti-1} \sum_{i=0}^{\ti-1} \bigg( \| f(\x_0) - \x_0 \| +  \| \noise_i \| \bigg) \\
    ={}& L_\mathrm{fc}^{\ti-1} \bigg( \ti \| f(\x_0) - \x_0 \| + \sum_{i=0}^{\ti-1} \| \noise_i \| \bigg)
  \end{align}
\end{proof}

\begin{restatable}{lemma}{boundeddomainlemma}
  \label{lem:exploration:regret:domain_bound}
    Let $b_\ni = L_\mathrm{fc}^{T-1} \Ti \left( B_0 + \sqrt{ 2 \sigma\nstate+ \frac{4 \sigma}{\mathrm{e}} \log \frac{\Ti \pi^2 n^2}{3 \delta}} \right)$ and $\| f(\x_0) - \x_0 \|_2 \leq B_0$. Then, with probability at least $(1 - \delta)$, we have for all iterations $n \geq 1$ and corresponding time steps $0 \leq \ti\leq \Ti$ that
    \begin{equation}
      \x_{\ti,\ni} \in \ball(\x_0, b_\ni),
    \end{equation}
    where $\ball(\x_0, b_\ni) = \{ \x \in \R^\nstate \mid \| \x - \x_0 \|_2 \leq b_\ni \}$ is a norm-ball centered around $\x_0$ with radius $b_\ni$.
\end{restatable}  
\begin{proof}
  From \cref{lem:states_compact_wrt_noise_norm}, we have for all $n \geq 1$, $0 \leq \ti\leq \Ti$ that
  \begin{equation}
    \| \x_{t,n} - \x_0 \|_2 \leq (1+L_\mathrm{fc})^{\ti-1} \left( \ti \| f(\x_0) - \x_0 \|_2 + \sum_{i=0}^{\ti-1} \| \noise_i \|_2 \right)  
  \end{equation}
  Now by \cref{as:dynamics_f_lipschitz,as:pi_lipschitz,as:reward_lipschitz} and 
  Combined with \cref{lem:noise_2norm_bound_all_iter}, we obtain
  \begin{align}
    \| \x_{t,n} - \x_0 \|_2 &\leq (1+L_\mathrm{fc})^{\ti-1} \left( \ti \| f(\x_0) - \x_0 \|_2 + \ti \sqrt{ 2 \sigma\nstate+ \frac{4 \sigma}{\mathrm{e}} \log \frac{t \pi^2 n^2}{3 \delta} }   \right) \\
    &\leq (1+L_\mathrm{fc})^{T-1} \Ti \left( \| f(\x_0) - \x_0 \|_2 + \sqrt{ 2 \sigma\nstate+ \frac{4 \sigma}{\mathrm{e}} \log \frac{\Ti \pi^2 n^2}{3 \delta} }   \right) \\
    &\eqdef b_\ni
  \end{align}
  Lastly, we have $\| f(\x_0) - \x_0 \|_2 \leq B_0$ by assumption, which concludes the proof.
\end{proof}

\subsection{Regret bounds over Unbounded Domains}

The probability for the noise bound is generally different from the one used for the well-calibrated model. We can derive a joint bound using a simple union bound.

\begin{lemma}
  \label{lem:exploration:regret:joint_f_confidence_and_domain_bound}
  Under \cref{as:well_calibrated_model,as:transition_noise_sub_gaussian,as:dynamics_f_lipschitz,as:pi_lipschitz,as:reward_lipschitz}, let $\| f(\x_0) - \x_0 \|_2 \leq B_0$ and define $b_\ni = L_\mathrm{fc}^{T-1} \Ti \left( B_0 + \sqrt{ 2 \sigma\nstate+ \frac{4 \sigma}{\mathrm{e}} \log \frac{\Ti \pi^2 n^2}{3 \delta}} \right)$. Then the following hold jointly with probability at least $(1-2\delta)$ for all $\ni \geq 1$ and $0 \leq \ti < \Ti$
  \begin{enumerate}[label=\roman*)]
    \item $| f(\x, \u) - \bmu_\ni(\x, \u) | \leq \beta_\ni \bsigma_\ni(\x, \u) $ ~elementwise for all $\x \in \R^\nstate$ and $\u \in \R^\ninp$
    \item $\x_{\ti,\ni} \in \ball(\x_0, b_\ni)$
  \end{enumerate}
\end{lemma}
\begin{proof}
  This follows directly from applying a union bound over \cref{lem:exploration:regret:domain_bound,lem:f_confidence_norm} with a probability budget of $\delta / 2$ for each.
\end{proof}
Note that the probability dropped from individual confidences of $1-\delta$ in \cref{as:well_calibrated_model} and \cref{lem:exploration:regret:domain_bound} to a joint confidence of $1 - 2\delta$.

Thus, we can used \cref{lem:exploration:regret:joint_f_confidence_and_domain_bound} together with \cref{lem:sigma_bounded_over_compact_set} to fulfill both the compact set and the boundedness requirements. The last assumption we need is boundedness of the predictions. For this, we introduce an additional weak assumptions

\begin{assumption}[Boundedness]
\label{as:boundedness_of_dynamics}
 The system dynamics at the first step are bounded, $\| f(\x_0) - \x_0 \|_2 \leq B_0$. Similarly we have $\bSigma(\x_0)$ and, if used, $k(\x_0, \x_0)$ bounded.
\end{assumption}

These assumptions are not restrictive, since any dynamical system that explodes to infinity after one step is generally not real-world relevant or controllable. Similarly, we cannot expect to do learning if our model's confidence intervals allow infinite predictions.

\begin{corollary}
    \label{lem:sigma_bounded_over_compact_set}
    Under \cref{as:boundedness_of_dynamics,as:model_predictions_lipschitz}, if the states live in a compact set $\X_\ni$, then $\bsigma(\x)$ is bounded.
\end{corollary}
\begin{proof}
  This follows trivially from \cref{as:model_predictions_lipschitz}, since $\x_0 \in \X$ and $\bsigma(\x_0)$ is bounded. Thus, by continuity, it must be bounded over a compact set.
\end{proof}

\begin{theorem}
  \label{thm:general_regret_unbounded_domain}
   Under \cref{as:well_calibrated_model,as:dynamics_f_lipschitz,as:pi_lipschitz,as:reward_lipschitz,as:model_predictions_lipschitz} let the noise distribution be $\sigma$-subGaussian as in \cref{as:transition_noise_sub_gaussian} and $\pi_\theta(\x) \in \U$ for all $\pi \in \Pi$ with $\U$ compact. At each iteration, select parameters according to \cref{eq:optimistic_exploration}. Then the following holds with probability at least $(1 - 2 \delta)$ for all $\Ni \geq 1$
  \begin{equation}
    R_\Ni \leq \Or{\beta_{\Ni-1}^\Ti L_\sigma^\Ti \Ti^2  \sqrt{ \Ni \, \nstate \, \gamma_{\nstate \Ni \Ti}( \ball(\x_0, b_\ni) \times \U \times \mathcal{I}_\nstate) }  },
  \end{equation}
  where $b_\ni = L_\mathrm{fc}^{T-1} \Ti \left( B_0 + \sqrt{ 2 \sigma\nstate+ \frac{4 \sigma}{\mathrm{e}} \log \frac{\Ti \pi^2 n^2}{3 \delta}} \right)$.
\end{theorem}
\begin{proof}
  By \cref{as:transition_noise_sub_gaussian} we know from \cref{lem:exploration:regret:joint_f_confidence_and_domain_bound} that with probability at least $(1 - 2\delta)$ the model is well-calibrated and $\x \in \X_\ni = \ball(\x_0, b_\ni)$. Boundedness of predictions follows from \cref{lem:sigma_bounded_over_compact_set}, so that all requirements of \cref{thm:exploration:regret:general_regret_bound} are satisfied and the result follows.
\end{proof}

\subsection{Bounding the Maximum Information Capacity for Gaussian Processes}
\label{ap:exploration:mutual_information_bound}

In \cref{thm:general_regret_unbounded_domain} the information capacity is a function of the domain size. Given the previous proofs, the radius of the domain increases at a logarithmic rate $b_\ni \in \Or{\log \ni^2}$, which also increases the information capacity. In the following two lemmas, we show how this affects the information capacity of the Gaussian process model.

\begin{lemma}[\citet{Srinivas2012Gaussian}]
  For the linear kernel $k(\x, \x') = \x^\mathrm{T} \x'$ with $\x \in \R^\nstate$ we have
  \begin{equation}
    \gamma_\ni(\ball(\x_0, b_\ni)) = \Or{\nstate \log(\ni)}
  \end{equation}
\end{lemma}

\begin{lemma}
  For the squared exponential kernel we have
  \begin{equation}
    \gamma_\ni(\ball(\x_0, b_\ni)) = \Or{b_\ni^\nstate (\log(\ni))^{\nstate+1}}
  \end{equation}
\end{lemma}
\begin{proof}
  The proof is the same as in \citep{Srinivas2012Gaussian}. In their notation, we have $n_T = \Or{b^d_\ni \log(b^d_\ni) }$ while analyzing the terms in the eigenvalue bound leads to $B_k(T^*) \sim b^d_\ni$. The remainder of the proof follows through as in the original paper, which leads to the result.
\end{proof}
Thus, the information capacity grows proportionally to the volume of the domain. Since $b_\ni$ in \cref{thm:general_regret_unbounded_domain} is $\Or{\log \ni^2}$ this means that this costs us only an additional logarithmic factor in the regret relative to a fixed domain $\X$.

Note that we are using a composite kernel to model the different output dimensions. Thus these bounds need to be combined with the methodology from \citet{Krause2011Contextual} in order to obtain bounds for the composite kernels. However, this does not affect the result.